\relax

\documentclass[10.5pt]{article}
\usepackage{fullpage}

\usepackage[round, sort]{natbib}
\usepackage[nottoc,notlot,notlof]{tocbibind}

\usepackage[colorlinks=true,linkcolor=blue, citecolor=blue]{hyperref}
\usepackage{url}            
\usepackage{booktabs}       
\usepackage{amsfonts}       
\usepackage{nicefrac}       
\usepackage{microtype}      

\usepackage{xcolor}
\usepackage{mathtools}

\usepackage{times}
\usepackage{epsfig}
\usepackage{graphicx}
\usepackage{amsmath}
\usepackage{amsthm}
\usepackage{amssymb}
\usepackage{bbm}
\usepackage[linewidth=1pt]{mdframed}
\usepackage{multirow}
\usepackage{booktabs}       
\usepackage{makecell}
\usepackage{placeins}
\usepackage{array}

\newtheorem{theorem}{Theorem}
\newtheorem{definition}[theorem]{Definition}
\newtheorem{proposition}[theorem]{Proposition}
\newtheorem{corollary}[theorem]{Corollary}
\newtheorem{lemma}[theorem]{Lemma}

\DeclareMathOperator*{\argmin}{arg\,min}
\DeclareMathOperator*{\argmax}{arg\,max}

\newcommand{\domx}{\mathcal{X}}
\newcommand{\domy}{\mathcal{Y}}

\newcommand{\vecq}{\boldsymbol{q}}

\newcommand{\veceta}{\boldsymbol{\eta}}
\newcommand{\vece}{\boldsymbol{e}}
\newcommand{\vecx}{\boldsymbol{x}}

\newcommand{\vecu}{\boldsymbol{u}}

\newcommand{\vecv}{\boldsymbol{v}}
\newcommand{\vecone}{\mathbbm{1}}

\newcommand{\zoloss}{\ell_{0\text{-}1}}
\newcommand{\focalloss}{\ell^\gamma_{\mathrm{FL}}}
\newcommand{\celoss}{\ell_{\mathrm{CE}}}
\newcommand{\Psibf}{\boldsymbol{\Psi}^\gamma}
\newcommand{\subunif}{\mathcal{S}^K}

\newcommand{\R}{\mathbb{R}}

\newcommand{\EJoint}{\mathop{\mathbb{E}}\limits_{(\vecx,y) \sim p(\vecx,y)}}

\newcommand{\resp}{\emph{resp.}}
\newcommand{\eg}{\emph{e.g.}}
\newcommand{\ie}{\emph{i.e.}}

\usepackage{graphicx}
\graphicspath{{./figures/}}
\newcommand\blfootnote[1]{%
  \begingroup
  \renewcommand\thefootnote{}\footnote{#1}%
  \addtocounter{footnote}{-1}%
  \endgroup
}

\begin{document}
\title{On Focal Loss for Class-Posterior Probability Estimation: \\A Theoretical Perspective}
\author{Nontawat Charoenphakdee$^{*1,2}$ \and Jayakorn Vongkulbhisal$^{*3}$ \and Nuttapong Chairatanakul$^{4,5}$ \and Masashi Sugiyama$^{2,1}$}
\date{
$^1$ The University of Tokyo 
$^2$ RIKEN AIP \\
$^3$ IBM Research 
$^4$ Tokyo Institute of Technology 
$^5$ RWBC-OIL, AIST \\
}
\maketitle
\blfootnote{*Nontawat and Jayakorn contributed equally.}
\begin{abstract}
The focal loss has demonstrated its effectiveness in many real-world
applications such as object detection and image classification, but
its theoretical understanding has been limited so far. In this paper,
we first prove that the focal loss is classification-calibrated, i.e.,
its minimizer surely yields the Bayes-optimal classifier and thus the
use of the focal loss in classification can be theoretically
justified. However, we also prove a negative fact that the focal loss
is not strictly proper, i.e., the confidence score of the classifier obtained
by focal loss minimization does not match the true class-posterior
probability and thus it is not reliable as a class-posterior
probability estimator. 
To mitigate this problem, we next prove that
a particular closed-form transformation of the confidence score allows us to
recover the true class-posterior probability. 
Through experiments on benchmark datasets, we demonstrate that our proposed transformation significantly improves the accuracy of class-posterior probability estimation.
\end{abstract}
\section{Introduction}
It is well-known that training classifiers with the same model architecture can have a huge performance difference if they are trained using different loss functions~\citep{bartlett2006,lin2017focal,ghosh2017robust}.
To choose an appropriate loss function, it is highly useful to know theoretical properties of loss functions.
For example, let us consider the hinge loss, which is related to the support vector machine~\citep{cortes1995support,bartlett2006}. 
This loss function is known to be suitable for classification since minimizing this loss can achieve the Bayes-optimal classifier.
However, it is also known that training with the hinge loss does not give the Bayes-optimal solutions for bipartite ranking~\citep{gao2015consistency,uematsu2017theoretically} and class-posterior probability estimation~\citep{reid2010}.
Such theoretical drawbacks of the hinge loss have been observed to be relevant in practice as well~\citep{platt1999probabilistic,uematsu2017theoretically}.
Not only the hinge loss, but many other loss functions have also been analyzed and their theoretical results have been used as a guideline to choose an appropriate loss function for many problems, \eg, classification from noisy labels~\citep{ghosh2017robust,charoenphakdee2019symmetric,liu2019peer}, classification with rejection~\citep{yuan2010,ni2019calibration}, and direct optimization of linear-fractional metrics~\citep{bao2020calibrated2, nordstrom2020calibrated}.

Recently, the focal loss has been proposed as an alternative to the popular cross-entropy loss~\citep{lin2017focal}.
This loss function has been shown to be preferable over the cross-entropy loss when facing the class imbalance problem.
Because of its effectiveness, it has been successfully applied in many applications, \eg, medical diagnosis~\citep{ulloa2020improving,xu2020focal,shu2019pathological,al2019dense}, speech processing~\citep{tripathi2019focal}, and natural language processing~\citep{shi2018normalized}.
Although the focal loss has been successfully applied in many real-world problems~\citep{ulloa2020improving,xu2020focal,shu2019pathological,chang2018brain,lotfy2019investigation,chen2019fr,romdhane2020electrocardiogram,tong2020pulmonary,sun2019drug,al2019dense}, considerably less attention has been paid to the theoretical understanding of this loss function.
For example, a fundamental question whether we can estimate a class-posterior probability from the classifier trained with the focal loss has remained unanswered. 
Knowing such a property is highly important when one wants to utilize the prediction confidence.
For example, one may defer the  decision to a human expert when a classifier has low prediction confidence~\citep{chow1970,yuan2010,ni2019calibration,mozannar2020consistent,charoenphakdee2020classification}, or one may use the prediction confidence to teach a new model, which has been studied in the literature of knowledge distillation~\citep{hinton2015distilling,vongkulbhisal2019unifying, menon2020distillation}.

Motivated by the usefulness of loss function analysis and the lack of theoretical understanding of the focal loss, the goal of this paper is to provide an extensive analysis of this loss function so that we can use it appropriately for the real-world applications. 
Our contributions can be summarized as follows:
\begin{itemize}
    \item In Sec.~\ref{sec:ClsCalibFocal}, we prove that the focal loss is classification-calibrated  (Thm.~\ref{thm:calib}), which theoretically confirms that the optimal classifier trained with the focal loss can achieve the Bayes-optimal classifier.
    \item In Sec.~\ref{sec:PropFocal}, we prove that learning with the focal loss can give both underconfident and overconfident classifiers (Thm.~\ref{thm:underover}). 
    Our result suggests that the simplex output of the classifier is not reliable as a class-posterior probability estimator~(Thm.~\ref{thm:not-proper}).
    \item In Sec.~\ref{sec:recover}, we prove that the true class-posterior probability can be theoretically recovered from the focal risk minimizer by our proposed novel transformation $\Psibf$ (Thm.~\ref{thm:recover}). 
    This allows us to calibrate the confidence score of the classifier, while maintaining the same decision rule
    ~(Prop.~\ref{prop:maintain}).
    \item We provide synthetic and benchmark experiments to empirically verify the behavior of the focal loss and the usefulness of our proposed transformation.
\end{itemize}
\section{Preliminaries}
In this section, we begin by describing the problem setting and notation we use in this paper. 
Then, we explain fundamental properties of loss functions used for classification, and end the section with a review of the focal loss. 

\subsection{Multiclass classification}
Let $\domx$ be an input space and $\domy=\{1, 2, \dots, K\}$ be a label space, where $K$ denotes the number of classes\footnote{Bold letters denote vectors, \eg, $\vecx$. Non-bold letters denote scalars, \eg, $x$. Subscripted letters denote vector elements, \eg, $x_i$ is element $i$ of $\vecx$. $\vecone_{[\cdot]}$ denotes the indicator function. $\vecx^\top$ denotes the transpose of $\vecx$.}. 
In multiclass classification, we are given labeled examples $\mathcal{D} = \{(\vecx_i, y_i) \}_{i=1}^n$ independently drawn from an unknown probability distribution (i.i.d.)~over $\domx \times \domy$ with density $p(\vecx, y)$. 
The goal of classification is to find a classifier $f\colon \domx \to \domy$ that minimizes the following classification risk:
\begin{align}
\label{eq:exp_risk}
    R^{\zoloss}(f) = \EJoint[\zoloss(f(\vecx), y)],
\end{align}
where $\zoloss$ is the zero-one loss $\zoloss(f(\vecx),y) = \vecone_{[f(\vecx) \neq y]}$.
Next, let us define the true class-posterior probability vector as $\veceta(\vecx) = [\eta_1(\vecx), \ldots, \eta_K(\vecx)]^\top$, where $\eta_y(\vecx) = p(y|\vecx)$ denotes the true class-posterior probability for a class $y$.
It is well-known that the Bayes-optimal classifier $f^{\zoloss,*}$, which minimizes the expected classification risk in Eq.~\eqref{eq:exp_risk}, can be defined as follows:.
\begin{definition}[Bayes-optimal classifier~\citep{zhang2004statisticalmulti}]
\rm{The Bayes-optimal solution of multiclass classification, $f^{\zoloss,*}=\argmin_f\,R^{\zoloss}(f)$, can be expressed as
\begin{equation}
\label{eq:bayes-cls}
    f^{\zoloss,*}(\vecx) = \argmax_y \eta_y(\vecx).
\end{equation}
}
\end{definition} 

As suggested in Eq.~\eqref{eq:bayes-cls}, knowing the true class-posterior probability $\veceta$ can give the Bayes-optimal classifier but the converse is not necessarily true~\citep{bartlett2006, vernet2011composite}. 
The support vector machine~\citep{cortes1995support} is a good example of a learning method that achieves the Bayes-optimal classifier but its confidence score is not guaranteed to obtain the true class-posterior probability~\citep{cortes1995support, platt1999probabilistic}. 

\subsection{Surrogate loss}
A common practice to learn a classifier using a neural network is to learn a mapping $\vecq: \mathcal{X} \to \Delta^K$, which maps an input to a $K$-dimensional simplex vector.
The simplex output of $\vecq$ is often interpreted as a probability distribution over predicted output classes. 
We denote $\vecq(\vecx) = [q_1(\vecx), \ldots, q_K(\vecx)]^\top$, where $q_y \colon \domx \to [0,1]$ is a score for class $y$ and $\sum_{y=1}^K q_y(\vecx)=1$. 
One typical choice of a mapping $\vecq$ would be a deep convolutional neural network with a softmax function as the output layer.
Given an example $\vecx$ and a trained mapping function $\vecq$, a decision rule $f^{\vecq}$ can be inferred by selecting a class with the largest score:
\begin{align}
    f^{\vecq}(\vecx) = \argmax_y q_y(\vecx). 
\end{align}

In classification, although the goal is to minimize the classification risk in Eq.~\eqref{eq:exp_risk}, it is not straightforward to minimize the classification risk in practice.
The first reason is we are given finite examples, not the full distribution.
Another reason is minimizing the risk w.r.t \emph{the zero-one loss} is known to be computationally infeasible~\citep{zhang2004statisticalmulti,bartlett2006}. 
As a result, it is common to minimize an \emph{empirical surrogate risk}~\citep{bartlett2006,vapnik1998statistical}.
Let $\ell\colon\Delta^K \times \Delta^K \to \R$ be a \emph{surrogate loss} and $\vece_{y} \in \{0,1\}^K$ be a one-hot vector with $1$ at the $y$-th index and $0$ otherwise.
By following the empirical risk minimization approach~\citep{vapnik1998statistical}, we minimize the following empirical surrogate risk:
\begin{align}
    \widehat{R}^\ell(\vecq) = \frac{1}{n} \sum_{i=1}^n \ell(\vecq(\vecx_i),\vece_{y_i}),
\end{align}
where regularization can also be added to avoid overfitting.

Note that the choice of a surrogate loss is not straightforward and can highly influence the performance of a trained classifier.
Necessarily, we should use a surrogate loss that is easier to minimize than the zero-one loss. 
Moreover, the surrogate risk minimizer should also minimize the expected classification risk in Eq.~\eqref{eq:exp_risk} as well.

\subsection{Focal loss}
In this paper, we focus on a surrogate loss $\ell\colon\Delta^K \times \Delta^K \to \R$ that receives two simplex vectors as arguments. 
Let $\vecu \in \Delta^K$, $\vecv \in \Delta^K$, and $\gamma \geq 0$ be a nonnegative scalar.
The focal loss $\focalloss:\Delta^K\times\Delta^K\rightarrow\R$ is defined as follows~\citep{lin2017focal}:
\begin{equation}
    \focalloss(\vecu, \vecv) = -\sum_{i=1}^K v_i(1-u_i)^\gamma \log(u_i).
\end{equation}
It can be observed that the focal loss with $\gamma = 0$ is equivalent to the well-known cross-entropy loss, \ie,~\citep{lin2017focal}:
\begin{equation}
    \label{eq:ce_loss}
    \ell_{\text{CE}}(\vecu, \vecv) = -\sum_{i=1}^K v_i \log(u_i).
\end{equation}
Unlike the cross-entropy loss that has been studied extensively~\citep{zhang2004statisticalmulti, buja2005loss, vernet2011composite}, we are not aware of any theoretical analysis on the fundamental properties of the focal loss.
Most analyses of the focal loss are based on an analysis of its gradient and empirical observation~\citep{lin2017focal,mukhoti2020calibrating}. 
In this paper, we will study the properties of \emph{classification-calibration}~\citep{bartlett2006, tewari2007consistency} (Sec.~\ref{sec:ClsCalibFocal}) and \emph{strict properness}~\citep{shuford1966admissible,buja2005loss, gneiting2007strictly} (Sec.~\ref{sec:PropFocal}) to provide a theoretical foundation to the focal loss.

\section{Focal loss is classification-calibrated}\label{sec:ClsCalibFocal}
In this section, we theoretically prove that minimizing the focal risk $R^{\focalloss}$ can give the Bayes-optimal classifier, which guarantees to maximize the expected accuracy in classification~\citep{zhang2004statisticalmulti}. 
We show this fact by proving that the focal loss is \emph{classification-calibrated}~\citep{bartlett2006, tewari2007consistency}.

First, let us define the pointwise conditional risk $W^\ell$ of an input $\vecx$ with its class-posterior probability $\veceta(\vecx)$:
\begin{align}
    W^\ell \big(\vecq(\vecx); \veceta(\vecx) \big) = \sum_{y \in \domy} \eta_y(\vecx) \ell\big( \vecq(\vecx), \vece_y \big).
\end{align}
Intuitively, the pointwise conditional risk $W^\ell$ corresponds to the expected penalty for a data point $\vecx$ when using $\vecq(\vecx)$ as a score function. 
Next, we give the definition of a classification-calibrated loss.
\begin{definition}[Classification-calibrated loss~\citep{bartlett2006,tewari2007consistency}]
\label{def:classification-calibration}
Consider a surrogate loss $\ell$. 
Let $\vecq^{\ell,*}=\argmin_{\vecq}\,W^{\ell}\big(\vecq(\vecx); \veceta(\vecx) \big)$ be the minimizer of the pointwise conditional risk. 
If $R^{\zoloss}(f^{\vecq^{\ell,*}}) = R^{\zoloss}(f^{\zoloss,*})$, then $\ell$ is classification-calibrated. 
\end{definition}

Classification-calibration guarantees that the minimizer of the pointwise conditional risk of a surrogate loss will give the Bayes-optimal classifier.
Definition~\ref{def:classification-calibration} suggests that by minimizing a classification-calibrated loss, even if $\vecq^{\ell,*}(\vecx)$ is not equal to the true class-posterior probability $\veceta(\vecx)$, we can still achieve the Bayes-optimal classifier from $\vecq^{\ell,*}(\vecx)$ as long as their decision rule matches. 

For notational simplicity, we use $\vecq^{\gamma,*}$ to denote $\vecq^{\focalloss,*}$, \ie, the focal risk minimizer with the parameter $\gamma$.
The following theorem guarantees that the focal loss is classification-calibrated (its proof can be found in Appx.~\ref{app:proof-calib}).

\begin{theorem}
    \label{thm:calib}
    For any $\gamma \geq 0$, the focal loss $\focalloss$ is classification-calibrated.
\end{theorem}

Our proof is based on showing that the focal loss has the \emph{strictly order-preserving property}, which is sufficient for classification-calibration~\citep{zhang2004statisticalmulti}. 
The order-preserving property suggests that for $\veceta(\vecx)$, the pointwise conditional risk $W^{\focalloss}$ has the minimizer $\vecq^{\gamma,*}(\vecx)$ such that $q^{\gamma,*}_i(\vecx) < q^{\gamma,*}_j(\vecx) \Rightarrow \eta_i(\vecx) < \eta_j(\vecx)$. 
Since $\vecq^{\gamma,*}$ preserves the order of $\veceta$, it is straightforward to see that $\argmax_y q^{\gamma,*}_y(\vecx)$ and $\argmax_y \eta_y(\vecx)$ are identical and thus the Bayes-optimal classifier can be achieved by minimizing the focal risk minimizer, \ie, $R^{\zoloss}(f^{\vecq^{\gamma,*}}) = R^{\zoloss}(f^{\zoloss,*})$. 
Our result agrees with the empirical effectiveness observed in the previous work~\citep{lin2017focal}, where evaluation metrics are based on accuracy or ranking such as mean average precision. 

\section{On confidence score of classifier trained with focal loss}\label{sec:PropFocal}
In this section, we analyze the focal loss for the class-posterior probability estimation problem. 
We theoretically prove that the simplex output of the focal risk minimizer~$\vecq^{\gamma,*}$ does not give the true class-posterior probability.
Further, we reveal that the focal loss can yield both underestimation and overestimation of the true class-posterior probability.

\subsection{Focal loss is \emph{not} strictly proper}
To ensure that a surrogate loss is appropriate for class-posterior probability estimation, it is required that a surrogate loss is \emph{strictly proper}, which is defined as follows.
\begin{definition}[Strictly proper loss~\citep{shuford1966admissible,buja2005loss, gneiting2007strictly}]
We say that a loss $\ell: \Delta^K \times \Delta^K \to \R$ is strictly proper if $\ell(\vecu,\vecv)$ is minimized if and only if $\vecu=\vecv$.
\end{definition}

The notion of strict properness can be seen as a natural requirement of a loss when one wants to estimate the true class-posterior probability~\citep{williamson2016composite}.
When comparing between the ground truth probability $\vecv$ and its estimate $\vecu$, we want a loss function to be minimized if and only if $\vecu=\vecv$, meaning that the probability estimation is correct.
Note that strict properness is a stronger requirement of a loss than classification-calibration because all strictly proper losses are classification-calibrated but the converse is false~\citep{reid2010,williamson2016composite}.

Here, we prove that the focal loss is not strictly proper in general (its proof is given in Appx.~\ref{app:proof-not-proper}). 
In fact, it is strictly proper if and only if $\gamma=0$, \ie, when it coincides with the cross-entropy loss.
\begin{theorem}
     For any $\gamma > 0$, the focal loss $\focalloss$ is not strictly proper.
     \label{thm:not-proper}
\end{theorem}

Our Thm.~\ref{thm:not-proper} suggests that to minimize the focal loss, the simplex output of a classifier does not necessarily need to coincide with the true class-posterior probability.
Surprisingly, a recent work~\citep{mukhoti2020calibrating} suggested that training with the focal loss can give a classifier with reliable confidence.
Although their finding seems to contradict with the fact that the focal loss is not strictly proper, we will discuss in Sec.~\ref{sec:exp-dis} that this phenomenon could occur in practice due to the fact that deep neural networks (DNNs) can suffer from overconfident estimation of the true class-posterior probability~\citep{guo2017calibration}.

\subsection{Focal loss gives under/overconfident classifier}

Motivated by the fact that the focal loss is not strictly proper and the intriguing effect of the focal loss for calibration of deep models~\citep{mukhoti2020calibrating}, we take a closer look at the behavior of the simplex output of the focal risk minimizer~$\vecq^{\gamma,*}$ compared with the true class-posterior probability.

We begin by pointing out that there exists the case where $\vecq^{\gamma,*}(\vecx)$ coincides with $\veceta(\vecx)$ (its proof can be found in Appx.~\ref{app:proof-focalcorrect}).

\begin{proposition}
    \label{thm:focalcorrect}
    Define $\subunif = \{\vecv \in \Delta^K: v_i \in \{0, \max_j v_j\}\}$. 
    If $\vecq^{\gamma,*}(\vecx) \in \subunif$, then $\vecq^{\gamma,*}(\vecx) = \veceta(\vecx)$.
\end{proposition}
The set $\subunif$ is the set of probability vectors where a subset of classes has uniform probability and the rest has zero probability, \eg, the uniform vector and one-hot vectors.
Prop.~\ref{thm:focalcorrect} indicates that, although the focal loss is not strictly proper, the focal risk minimizer can give the true class-posterior probability if  $\vecq^{\gamma,*}(\vecx) \in \subunif$.

For the rest of this section, we assume that $\vecq^{\gamma,*}(\vecx) \notin \subunif$ for readability.
Next, to analyze the focal loss behavior in general, we propose the notion of $\veceta$-underconfidence and $\veceta$-overconfidence of the risk minimizer $\vecq^{\ell,*}$ as follows.
\begin{definition}[$\veceta$-under/overconfidence of risk minimizer] 
\label{def:underover}
We say that the risk minimizer $\vecq^{\ell,*}$ is $\veceta$-underconfident ($\veceta$UC) at $\vecx$
if
\begin{equation}
     \max_y q^{\ell,*}_{y}(\vecx) - \max_y \eta_y(\vecx) < 0.
\end{equation}
Similarly, $\vecq^{\ell,*}$ is said to be $\veceta$-overconfident~($\veceta$OC) at $\vecx$ 
if
\begin{equation}
     \max_y q^{\ell,*}_{y}(\vecx) - \max_y \eta_y(\vecx) > 0.
\end{equation}
\end{definition}

Def.~\ref{def:underover} can be interpreted as follows. 
If $\vecq^{\ell,*}$ is $\veceta$UC (\resp, $\veceta$OC) at $\vecx$, then the confidence score $\max_y q^{\ell,*}_{y}(\vecx)$ for the predicted class must be lower (\resp, higher) than that of the true class-posterior probability $\max_y \eta_y(\vecx)$.
It is straightforward to see that the risk minimizer of any strictly proper loss does not give an $\veceta$UC/$\veceta$OC classifier because $\vecq^{\ell,*}$ must be equal to the true class-posterior probability~$\veceta$.
Thus, Def.~\ref{def:underover} is not useful for characterizing strictly proper losses but it is highly useful for analyzing the behavior of the focal loss.

We emphasize that the notion of $\veceta$-under/overconfidence of the risk minimizer is significantly different from the notion of \emph{overconfidence} that has been used in the literature of confidence-calibration~\citep{guo2017calibration,kull2019beyond, mukhoti2020calibrating}.
In that literature, \emph{overconfidence} was used to describe the \emph{empirical performance} of modern neural networks~\citep{degroot1983comparison,niculescu2005predicting}, where a classifier outputs an average confidence score higher than its average accuracy for a set of data points.
In our case, $\veceta$OC and $\veceta$UC are based on the behavior of the risk minimizer of the loss function, which does not concern with the empirical validation.

\begin{figure}
\centering
\includegraphics[scale=1.8]{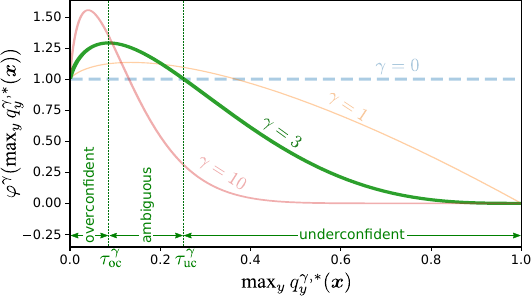}
\caption{\label{fig:psi} 
The function $\varphi^\gamma(v)$ for various $\gamma$. 
Visualization of the region where $\vecq^{\gamma,*}$ can be $\veceta$UC and $\veceta$OC are emphasized for $\gamma=3$. 
Whether $\vecq^{\gamma,*}$ is $\veceta$OC or $\veceta$UC can be largely determined by the relation between $\varphi^\gamma$ and the maximum predicted score $\max_y q^{\gamma,*}_y(\vecx)$.
See details in Thm.~\ref{thm:underover}.}
\end{figure}

To study the behavior of $\vecq^{\gamma,*}$, let us define a function $\varphi^\gamma:[0,1]\to\R$ as
\begin{align}
    \varphi^\gamma(v) =    (1-v)^\gamma-\gamma(1-v)^{\gamma-1}v\log v.\label{eq:varphi}
\end{align}
This function plays a key role in characterizing if $\vecq^{\gamma,*}$ is $\veceta$UC/$\veceta$OC.
See Appx.~\ref{app:proof} for more details on how $\varphi^\gamma$ was derived.
Next, we state our main theorem that characterizes the $\veceta$UC/$\veceta$OC behaviors of the risk minimizer of the focal loss $\vecq^{\gamma,*}$ (its proof is given in Appx.~\ref{app:proof-underover}).

\begin{theorem}
\label{thm:underover}
Consider the focal loss $\focalloss$ where $\gamma > 0$.
Define $\tau^{\gamma}_\mathrm{oc} = \argmax_v \varphi^\gamma(v)$ and $\tau^{\gamma}_\mathrm{uc} \in (0,1)$ such that $\varphi^\gamma(\tau^{\gamma}_\mathrm{uc})=1$. 
If $\vecq^{\gamma,*}(\vecx) \notin \subunif$, we have
\begin{enumerate}
    \item $0 < \tau^{\gamma}_\mathrm{oc} < \tau^{\gamma}_\mathrm{uc} < 0.5$.
    \item $\vecq^{\gamma,*}$ is $\veceta$OC for $  \max_y q^{\gamma,*}_y(\vecx) \in(0, \tau^{\gamma}_\mathrm{oc}]$.
    \item $\vecq^{\gamma,*}$ is $\veceta$UC for $  \max_y q^{\gamma,*}_y(\vecx) \in[\tau^{\gamma}_\mathrm{uc}, 1 )$ .
\end{enumerate}
\end{theorem}
Thm.~\ref{thm:underover} suggests that \textbf{training with focal loss can lead to both $\veceta$UC and $\veceta$OC classifiers}.
It also indicates that we can determine if $\vecq^{\gamma,*}$ is $\veceta$OC or $\veceta$UC at $\vecx$ if $\max_y q^{\gamma,*}_y(\vecx)$ is in $(0, \tau^{\gamma}_\mathrm{oc}]$ or $[\tau^{\gamma}_\mathrm{uc}, 1)$.
For $\max_y q^{\gamma,*}_y(\vecx)\in(\tau^{\gamma}_\mathrm{oc}, \tau^{\gamma}_\mathrm{uc})$, we may require the knowledge of $q_{y'}$ for all $y' \in \domy$ to determine if $\vecq^{\gamma,*}$ is $\veceta$UC or $\veceta$OC. 
Nevertheless, in Sec.~\ref{sec:recover}, we will show that given any $\vecq^{\gamma,*}(\vecx)$, $\veceta$UC and $\veceta$OC can be determined everywhere \emph{including the ambiguous region} $(\tau^{\gamma}_\mathrm{oc}, \tau^{\gamma}_\mathrm{uc})$ by using our novel transformation~$\Psibf$.    
Fig.~\ref{fig:psi} illustrates the overconfident, ambiguous, and underconfident regions of $\vecq^{\gamma,*}$.
Interestingly, the fact that $\vecq^{\gamma,*}$ can be $\veceta$OC cannot be explained by the previous analysis~\citep{mukhoti2020calibrating}, which only implicitly suggested that $\vecq^{\gamma,*}$ is $\veceta$UC by interpreting focal loss minimization as the minimization of an upper bound of the regularized Kullback-Leibler divergence.

Since calculating $\tau^\gamma_\mathrm{oc}$ and $\tau^\gamma_\mathrm{uc}$ is not straightforward because their simple close-form solutions may not exist for all~$\gamma$, 
we provide the following corollary to show that there exists a region where $\vecq^{\gamma,*}$ is always $\veceta$UC regardless of the choice of $\gamma$ (its proof is given in Appx.~\ref{app:proof-onehalfunder}).
\begin{corollary}
\label{cor:onehalfunder}
 For all $\gamma > 0$, $\vecq^{\gamma,*}$ is $\veceta$UC if $  \max_y q^{\gamma,*}_y(\vecx) \in(0.5, 1)$.
\end{corollary}
Cor.~\ref{cor:onehalfunder} suggests that $\vecq^{\gamma,*}$ is $\veceta$UC when the label is not too ambiguous.
In practice, a classifier is more likely to be $\veceta$UC but it still could be $\veceta$OC when the number of classes $K$ is large and $\gamma$ is small.
Fig.~\ref{fig:underover}b demonstrates that $\vecq^{\gamma,*}$ can be $\veceta$OC when having $1000$ classes for different $\gamma$.\footnote{We numerically found that $\vecq^{\gamma,*}$ can be $\veceta$OC with as minimum as $K=5$ classes with $\gamma \leq 0.03$ when $\max_y q_y^{\gamma,*}(\vecx)\rightarrow1/K$.}

We also provide the following corollary, which is an immediate implication from Cor.~\ref{cor:onehalfunder} for the binary classification scenario (its proof is given in Appx.~\ref{app:proof-binaryunder}).
\begin{corollary}
\label{cor:binaryunder}
For all $\gamma > 0$, $\vecq^{\gamma,*}$ is always $\veceta$UC in binary classification unless $\vecq^{\gamma,*}(\vecx)$ is uniform or a one-hot vector. 
\end{corollary}
Fig.~\ref{fig:underover}a demonstrates that $\vecq^{\gamma,*}$ is $\veceta$UC in binary classification, where a larger $\gamma$ causes a larger gap between $\max_y q_y^{\gamma,*}$ and the true class-posterior probability.
\begin{figure}
\centering
\includegraphics[scale=1.8]{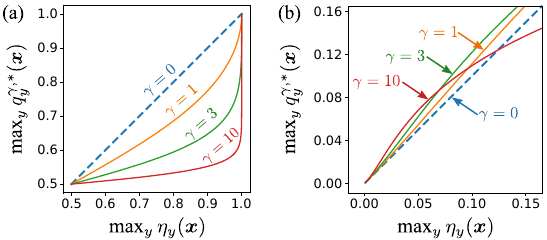}
\caption{\label{fig:underover} Relation between $\max_y \eta_y(\vecx)$ and $\max_y q^{\gamma,*}_y(\vecx)$ under different~$\gamma$. (a) shows the relation in the binary case, where $\vecq^{\gamma,*}$ can be $\veceta$UC and the effect of $\veceta$UC is stronger as $\gamma$ increases. 
On the other hand, (b) shows the relation of a 1000-way classification where the focal loss can also exhibit $\veceta$OC.}
\end{figure}
\section{Recovering class-posterior probability from classifiers trained with focal loss}
\label{sec:recover}
In this section, we propose a novel transformation $\Psibf$ to recover the true class-posterior probability from the focal risk minimizer with theoretical justification. 
Then, we provide a numerical example to demonstrate its effectiveness.
\subsection{Proposed transformation $\Psibf$}
Our following theorem reveals that there exists a transformation that can be computed \emph{in a closed form} to recover the true class-posterior probability from the focal risk minimizer (its proof is given in Appx.~\ref{app:proof-recover}).

\begin{theorem}
    \label{thm:recover}
    Let $\veceta(\vecx)$ be the true class-posterior probability of an input $\vecx$ and $\vecq^{\gamma,*}=\argmin_{\vecq}\,W^{\focalloss}\big(\vecq(\vecx), \veceta(\vecx) \big)$ be the focal risk minimizer, where $\gamma \geq 0$.
 Then, the true class-posterior probability $\veceta(\vecx)$ can be recovered from $\vecq^{\gamma,*}$ with the transformation $\Psibf:\Delta^K\to\Delta^K$, \ie,
    \begin{align}
     \veceta(\vecx)=
     \Psibf(\vecq^{\gamma,*}(\vecx)),
     \end{align}\label{eq:Psigamma_def}
     where 
    \begin{align}
     \small
     \Psibf(\vecv) &= [\Psi_1^\gamma(\vecv), \ldots, \Psi_K^\gamma(\vecv)]^\top
     ,\\
     \Psi_i^\gamma(\vecv) &= \frac{h^\gamma(v_i)}{\sum_{l=1}^K h^\gamma(v_l)}
     ,\\
        h^\gamma(v)&=\frac{v}{\varphi^\gamma(v)}\hspace{-.5px}=\hspace{-.5px}\frac{v}{(1-v)^\gamma-\gamma(1-v)^{\gamma-1}v\log v}.
    \end{align}
\end{theorem}
For completeness, we also define $\Psi^\gamma_i(\vecv) = v_i$ if $\vecv$ is a one-hot vector.
Note that if $\gamma=0$, then $\eta_i(\vecx)=\Psi_i^\gamma(\vecq^{\gamma,*}(\vecx)) = q^{\gamma,*}_i(\vecx)$. 
Hence, our result suggests that the cross-entropy risk minimizer $\vecq^{\celoss,*}$ does not need to apply any additional transformation to obtain the true class-posterior probability, which coincides with the known analysis of the cross-entropy loss~\citep{feuerverger1992some,gneiting2007strictly}.
On the other hand, an additional step of applying $\Psibf$ is required when $\gamma \neq 0$ for the focal loss.
We also want to emphasize that for any given $\max_y q_y^{\gamma,*}(\vecx)$ in the ambiguous region~(see Fig.~\ref{fig:psi}), one can easily determine if it is $\veceta$UC or $\veceta$OC by comparing $\max_y q_y^{\gamma,*}(\vecx)$ and $\max_y\Psi_y^\gamma( \vecq^{\gamma,*}(\vecx))$.

Next, we confirm that our proposed transformation $\Psibf$ does not degrade the classification performance of the classifier by proving that $\Psibf$ preserves the decision rule (its proof is given in Appx.~\ref{app:proof-maintain}). 

\begin{proposition}
\label{prop:maintain}
    Given $\vecv \in \Delta^K$ and $\gamma \geq 0$,
    we have 
    \begin{align}
         \argmax_i\ \Psi_i^\gamma(\vecv)= \argmax_i v_i.
    \end{align}
\end{proposition}

In summary, if one wants to recover the true class-posterior probability from the focal risk minimizer with $\gamma\neq0$, an additional step of applying $\Psibf$ is suggested by Thm.~\ref{thm:recover}.
However, if one only wants to know which class has the highest prediction probability, then applying~$\Psibf$ is unneeded since it does not change the prediction result.
We want to emphasize that that using the transformation $\Psibf$ to recover the true class-posterior probability is significantly different and orthogonal from using a heuristic technique such as Platt scaling~\citep{platt1999probabilistic}.
The differences are: (1) Using $\Psibf$ is theoretically guaranteed given the risk minimizer and (2) No additional training is involved since the transformation~$\Psibf$ does not contain any tuning parameter, whereas Platt scaling requires additional training, which can be computationally expensive when using a large training dataset. 
Note that a transformation such as $\Psibf$ that relates a risk minimizer to the true class-posterior probability is not guaranteed to exist for every loss, \eg, there is no such transformation for the hinge loss~\citep{cortes1995support,platt1999probabilistic,reid2010}.

\subsection{Numerical illustration}\label{sec:numerical_illus}
\begin{figure*}
\includegraphics[width=\textwidth]{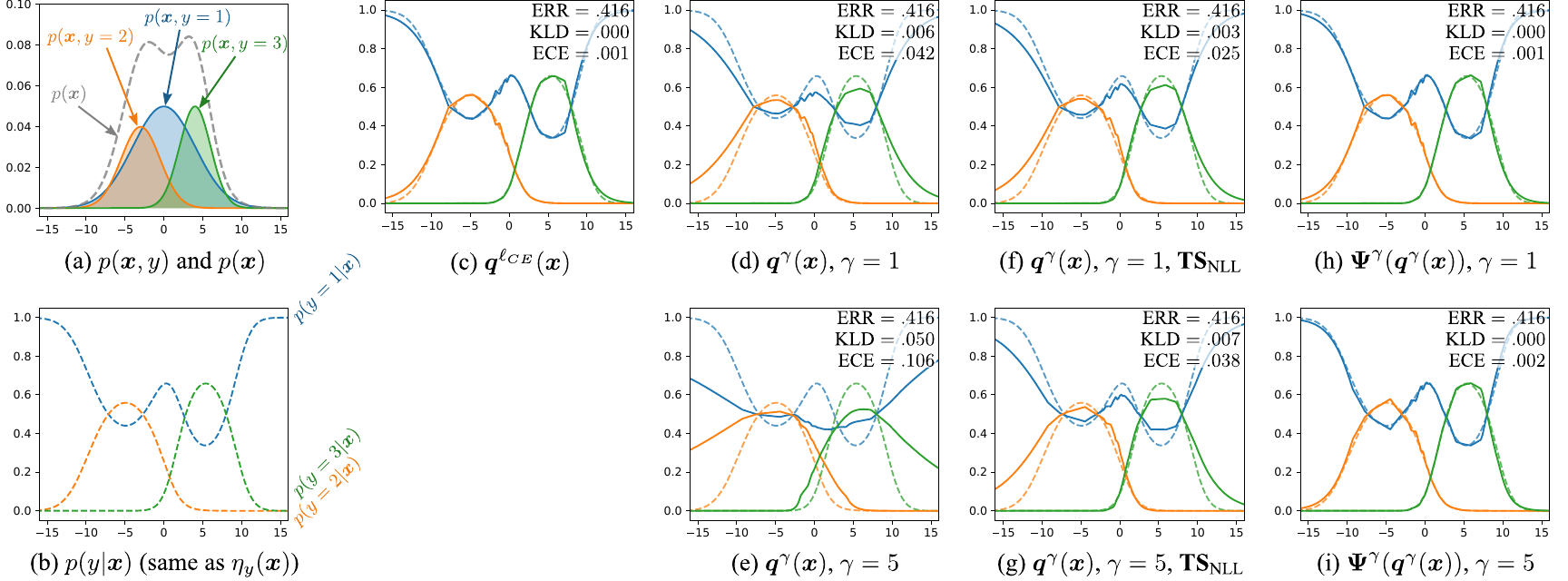}
\caption{\label{fig:fig_synth}Demonstration of the  underconfident ($\veceta$UC) property of the focal loss and the result of the transformation $\Psibf$. 
(a) and (b) show $p(\vecx)$, $p(\vecx,y)$, and $p(y|\vecx)$ used for training the MLPs. 
For (c-i), solid graphs are the raw or transformed predicted scores from the MLPs while dashed graphs are $p(y|\vecx)$ (same as (b)). 
The ERR, KLD, and ECE values are reported on the top-right of each subfigure. 
(c) shows $\vecq^\celoss(\vecx)$ of an MLP trained with $\celoss$ while
(d) and (e) show $\vecq^\gamma(\vecx)$ of MLPs trained with $\focalloss$ with $\gamma=1 \text{ and } 5$.
(f) and (g) show the scores after processing with \textbf{TS}$_\textrm{NLL}$. 
(h) and (i) show the scores after using the our proposed $\Psibf$ in Eq.~\eqref{eq:Psigamma_def}.
See Sec.~\ref{sec:numerical_illus} for details.
}
\end{figure*}
Here, we use synthetic data to demonstrate the $\veceta$UC property of the focal loss and show that applying $\Psibf$ can successfully recover the true class-posterior probability.
The purpose of using the synthetic data is because we know the true class-posterior probability $\veceta$ in this problem.
Unlike many real-world datasets where only hard labels are given, we can directly evaluate the quality of class-posterior probability estimation using the Kullback-Leibler divergence (KLD), which is defined as
    $\mathrm{KL}(\veceta(\vecx)||\vecq(\vecx)) = \sum_{i=1}^K \eta_i(\vecx)\log\frac{\eta_i(\vecx)}{q_i(\vecx)}.$

We simulate a $1$-dimensional $3$-class classification problem with the distribution given in Fig.~\ref{fig:fig_synth}a. 
We then trained three-layer multilayered perceptrons (MLPs) with $\ell_\mathrm{CE}$ and $\ell^\gamma_\mathrm{FL}$ ($\gamma=1 \text{ and } 5$) using data sampled from the distribution. 
The estimated confidence scores $q^\ell_y(\vecx)$ of all losses are shown in Fig.~\ref{fig:fig_synth}c,d,e. 
We can see that all MLPs can correctly identify the class having the highest class-posterior probability for the whole $\domx$
and achieve roughly the same classification error (ERR), which corresponds to the fact that both $\ell_\mathrm{CE}$ and $\focalloss$ are classification-calibrated. 
However, while $q^{\ell_\mathrm{CE}}_y(\vecx)$ in Fig.~\ref{fig:fig_synth}c could correctly estimate $\eta_y(\vecx)$, $q^{\gamma}_y(\vecx)$ in Fig.~\ref{fig:fig_synth}d,e do not match $\eta_y(\vecx)$, which agrees with our result that the focal loss is not strictly proper. 
More precisely, the value of the $\max_y q^{\gamma}_y(\vecx)$ is lower than $\max_y \eta_y(\vecx)$, which indicates that $\vecq^{\gamma}(\vecx)$ is $\veceta$UC. 
With a larger $\gamma$, we can observe this trend more significantly by looking at KLD and the expected calibration error (ECE)~\citep{naeini2015obtaining,guo2017calibration}, where low ECE indicates good \emph{empirical confidence}.

One well-known approach to improve confidence estimation in neural networks is temperature scaling (TS)~\citep{guo2017calibration}. 
We applied TS with negative log-likelihood (NLL) as the validation objective (\textbf{TS}$_\mathrm{NLL}$) to the MLPs trained with the focal loss. 
We can see from Fig.~\ref{fig:fig_synth}f,g that while \textbf{TS}$_\mathrm{NLL}$ made the $q^{\gamma}_y(\vecx)$ move closer to $\eta_y(\vecx)$, a large gap between them still exists, suggesting that \textbf{TS}$_\mathrm{NLL}$ fails to obtain the true class-posterior probability.

By using the transformation $\Psibf$, we can plot Fig.~\ref{fig:fig_synth}h,i and see that $\Psibf(\vecq^\gamma(\vecx))$ can improve the quality of the estimation, where both KLD and ECE are almost zero.
Recall that $\Psibf$ can be applied without any additional data or changing decision rule, thus the ERR remains exactly the same. 
This synthetic experiment demonstrates that the simplex outputs of neural networks trained with the focal loss is likely to be $\veceta$UC, and this can be effectively fixed using the transformation $\Psibf$.

\section{Experimental results}
In this section, we perform experiments to study the behavior of the focal loss and validate the effectiveness of $\Psibf$ under different training paradigms. 
To do so, we use the CIFAR10~\citep{cifar10} and SVHN~\citep{svhn} datasets as the benchmark datasets.
The details of the experiments are as follows.

\textbf{Models:} 
To see the influence of the model complexity on classifiers trained with the focal loss, we used the residual network (ResNet) family~\citep{he2016deep}, \ie, ResNet$L$ with $L=8,20,44,110$, where complexity increases as $L$ increases.

\textbf{Methods:} 
We compared the networks that use $\Psibf$ after the softmax layer to those that do not.
Note that both methods have the same accuracy since $\Psibf$ does not affect the decision rule (Prop.~\ref{prop:maintain}). 
We used the focal loss with $\gamma\in \{0, 1, 2, 3\}$ in this experiment, and conducted 10 trials for each experiment setting. 
\begin{figure*}
\hspace{-0.2em}
\includegraphics[width=\textwidth]{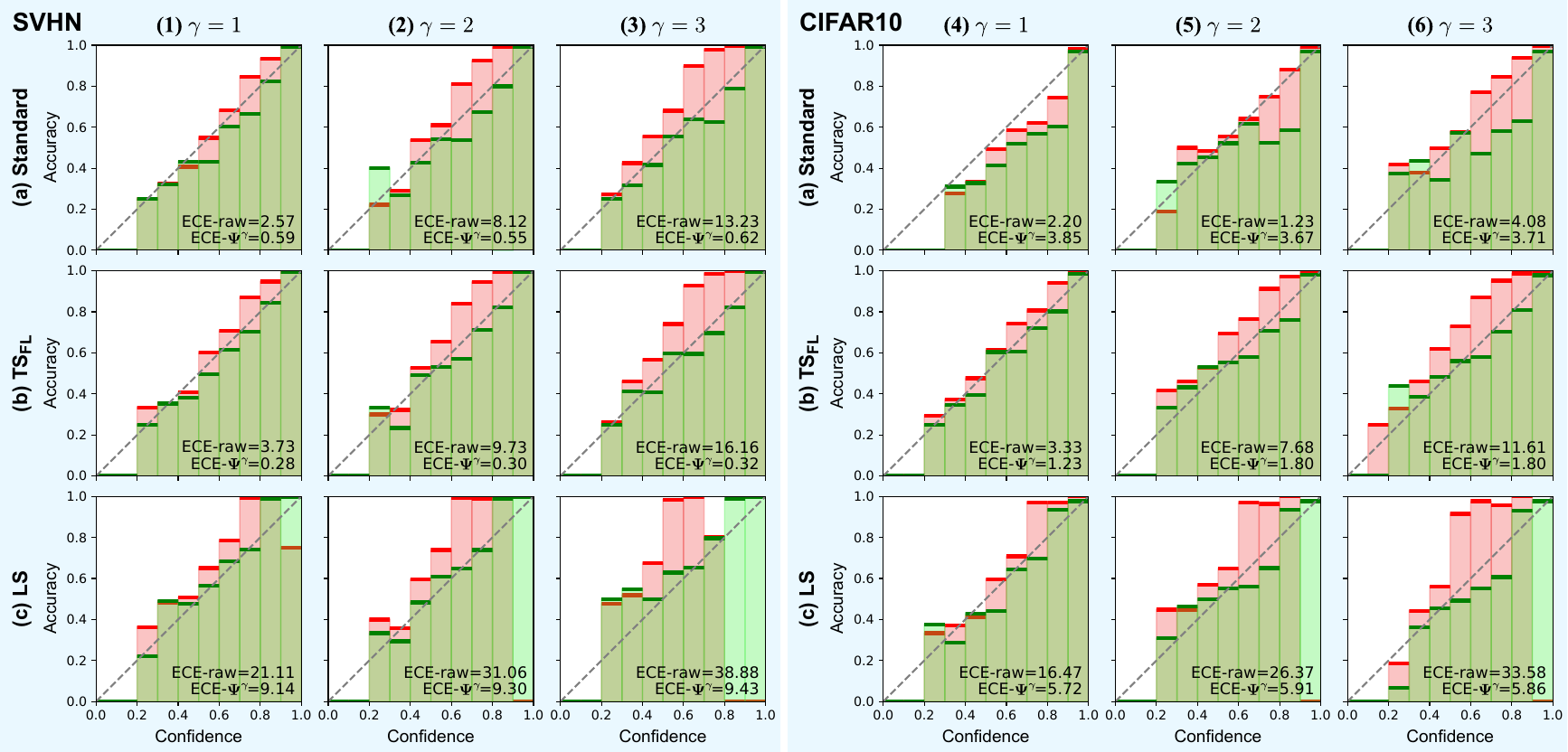}
\caption{\label{fig:reliabilitydiagram} Reliability diagrams of ResNet110 trained with $\focalloss, \gamma=1,2,3$ on SVHN and CIFAR10 datasets.
ECE-$\Psibf$ (\resp, ECE-raw) denotes the ECE of the networks that use (\resp, do not use) $\Psibf$ and their diagrams are plotted in green (\resp, red).
Each row shows the results of different training paradigms: (a) \textbf{Standard}, (b) $\textbf{TS}_{\textrm{FL}}$, and (c) \textbf{LS}.
See Sec.~\ref{sec:ece-three} for details.}
\end{figure*}
\textbf{Evaluation metrics:} Since true class-posterior probability labels are not available, a common practice is to use ECE to evaluate the quality of prediction confidence~\citep{naeini2015obtaining,guo2017calibration}.  
In this paper, we used $10$ as the number of bins.
ECE-$\Psibf$ (\resp, ECE-raw) denotes the ECE of the networks that use (\resp, do not use) $\Psibf$.
We found that NLL is highly correlated with ECE and we report full results on more evaluation metrics and models in Appx.~\ref{app:exp}. 

\textbf{Hyperparameters:} 
For all models, the number of epochs was $200$ for CIFAR10 and $50$ for SVHN. 
The batch size was $128$.
We used SGD with momentum of $0.9$, where the initial learning rate was $0.1$, which was then divided by $10$ at epoch $80$ and $150$ for CIFAR10 and at epoch $25$ and $40$ for SVHN. 
The weight decay parameter was $5 \times 10^{-4}$.

\subsection{ECE of different training paradigms}
\label{sec:ece-three}
We trained models using three different paradigms: (1) \textbf{Standard} uses one-hot ground truth vectors, which is known to be susceptible to overconfidence~\citep{guo2017calibration}; 
(2) $\textbf{TS}_{\textrm{FL}}$ post-processes the output of \textbf{Standard} with TS that uses the focal loss in the validation objective; and
(3) \textbf{LS} uses label smoothing to smoothen one-hot labels to soft labels, which has been reported to alleviate the overconfidence issue in DNNs~\citep{muller2019does}. 
The label smoothing parameter was $0.1$.

Fig.~\ref{fig:reliabilitydiagram} shows the reliability diagrams for ResNet110 trained with the focal loss using different $\gamma$. 
We can see that $\Psibf$ substantially improves ECE for most settings.
This demonstrates that our theoretically-motivated transformation $\Psibf$ can be highly relevant in practice.
For \textbf{LS}, ECE-raw drastically increases as $\gamma$ increases, whereas the value of $\gamma$ does not significantly affect ECE-$\Psibf$. 
Next, in \textbf{TS}$_\mathrm{FL}$, if $\Psibf$ is not applied, we can see that ECE-raw degrades compared with that of \textbf{Standard}.
On the other hand, our transformation $\Psibf$ can further improve the performance of \textbf{Standard}. 
This could be due to \textbf{TS}$_\mathrm{FL}$ giving a more accurate estimate of the focal risk minimizer $\vecq^{\gamma,*}$, but $\vecq^{\gamma,*}$ does not coincide with the true class-posterior probability $\veceta$ if $\Psibf$ is not applied, as proven in Thm.~\ref{thm:recover}.
Apart from \textbf{Standard} in CIFAR10, underconfident bins (\ie, the bins that align above the diagonal of the reliability diagram) can be observed especially when $\gamma$ is large.
The results indicate that the focal loss is susceptible to be underconfident as $\gamma$ increases, which agrees with our analysis that the focal loss is not strictly proper (Thm.~\ref{thm:not-proper}) and prone to $\veceta$UC (Cor.~\ref{cor:onehalfunder}). 

\subsection{Why does $\Psibf$ not always improve ECE?}

\begin{figure*}
\centering
\includegraphics[width=\textwidth]{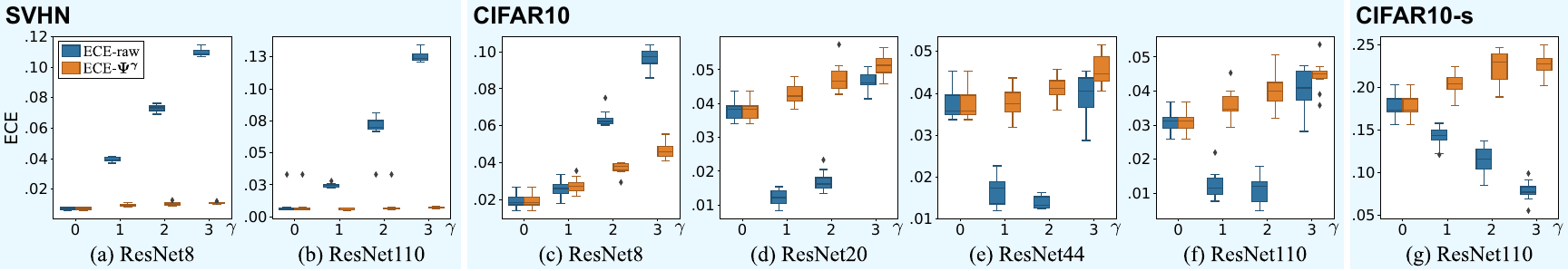}
\caption{\label{fig:boxplot1} Box plots of ECEs for \textbf{Standard} with varying $\gamma$ using different models on (a-b) SVHN, (c-f) CIFAR10, and (g) CIFAR10-s.
}
\end{figure*}

\begin{figure}
\centering
\includegraphics{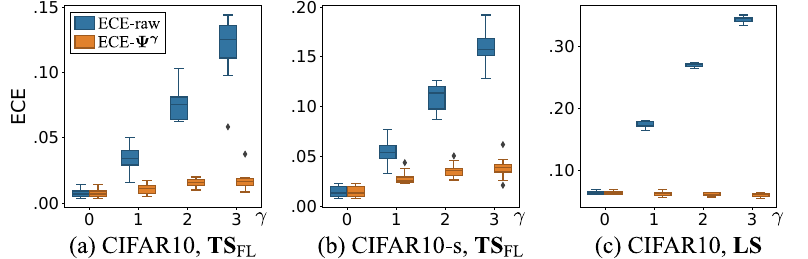}
\caption{\label{fig:boxplot2} Box plots of ECEs with varying $\gamma$ for ResNet110 using \textbf{TS}$_\mathrm{FL}$ for CIFAR10 and CIFAR10-s, and \textbf{LS} for CIFAR10. It can be observed that using the transformation $\Psibf$ is preferable.}
\vspace{-0.4em}
\end{figure}
In Fig.~\ref{fig:reliabilitydiagram}, although our transformation $\Psibf$ can greatly improve the performance for \textbf{Standard} in SVHN, it worsens the performance for \textbf{Standard} in CIFAR10. 
This demonstrates that our proposed transformation does not always improve the performance in practice, which could occur when the focal risk minimizer $\vecq^{\gamma,*}$ is not successfully learned.
Note that if $\vecq^{\gamma,*}$ is obtained, the transformation $\Psibf$ is \emph{the only mapping to obtain the true class-posterior probability $\veceta$ from $\vecq^{\gamma,*}$, i.e., $\Psibf \circ \vecq^{\gamma,*}$ = $\veceta$} (Thm.~\ref{thm:recover}). 

Here, we take a closer look at the scenario where $\Psibf$ could be less effective. 
We hypothesize that there are two potential reasons: (1) DNNs can overfit the one-hot vector, which leads to overconfident prediction~\citep{guo2017calibration}.
By using one-hot vectors as labels, perfectly minimizing the empirical risk implies making the confidence score close to a one-hot vector. 
(2) The amount of data could be insufficient for correctly estimating the true class-posterior probability. 

To justify our claim, we conducted experiments with different models on SVHN, CIFAR10, and CIFAR10-s, where CIFAR10-s is CIFAR10 that uses only $10\%$ of training data for each class.
Note that SVHN has a larger number of data than CIFAR10, and ResNet$L$ is more complex as $L$ increases. 
Fig.~\ref{fig:boxplot1} illustrates ECEs with different dataset size and models. 
In Fig.~\ref{fig:boxplot1}b,f,g, where the same model was used, we can observe that $\Psibf$ becomes less effective as the dataset size gets smaller.
Also, in Fig.~\ref{fig:boxplot1}c-f, where the different models were used in CIFAR10, $\Psibf$ becomes less effective as the model becomes more complex.
Therefore, the results agree with our hypotheses that the model complexity and dataset size play a role in the effectiveness of $\Psibf$.
Note that $\Psibf$ is still effective in SVHN regardless of the model size in our experiments, as can be seen in Fig.~\ref{fig:boxplot1}a,b, since the size of SVHN may be sufficiently large to accurately estimate $\vecq^{\gamma,*}$ for complex models.

It is also insightful to observe the best value of $\gamma$ in different settings.
More precisely, we see from Fig.~\ref{fig:boxplot1}c-g that the best $\gamma$ for CIFAR10 is $\gamma=0$ for  ResNet8, $\gamma=1$ for ResNet20, and $\gamma=2$ for ResNet110, while the best $\gamma$ for CIFAR10-s and ResNet110 is $\gamma=3$. 
Therefore, we can conclude that the best $\gamma$ increases as the data size decreases or the model becomes more complex.
Nevertheless, For \textbf{LS} and \textbf{TS}$_\mathrm{FL}$, we observe that a larger $\gamma$ always leads to worse performance and $\Psibf$ can effectively mitigate this problem for every dataset, as illustrated in Fig.~\ref{fig:boxplot2}.

\subsection{Discussion}
\label{sec:exp-dis}
Recently,~\citet{mukhoti2020calibrating} studied the relation between the focal loss and the confidence issue of DNNs, 
and showed that without post-processing, training with the focal loss can achieve lower ECE than that of the cross-entropy loss.
Our results indicate that this is not always the case (see SVHN for example).
In Appx.~\ref{app:exp}, we provide additional experimental results on $30$ datasets to show that the focal loss is less desirable compared with the cross-entropy loss in most datasets, and that $\Psibf$ can successfully improve ECE to be comparable with that of the cross-entropy loss.
Nevertheless, the focal loss can also outperform the cross-entropy loss as shown in Fig.~\ref{fig:boxplot1}, which agrees with the previous work~\citep{mukhoti2020calibrating}.  
This could occur when classifiers (especially DNNs) suffer from overconfidence due to empirical estimation~\citep{guo2017calibration}.
Since the focal loss tends to give an $\veceta$UC classifier, there may exist a sweet spot for $\gamma > 0$ that gives the best ECE because the overconfident and underconfident effects cancel each other out.

In addition, it has been observed that applying TS w.r.t. NLL or ECE on a classifier trained with the focal loss can be empirically effective to reduce ECE~\citep{guo2017calibration,mukhoti2020calibrating}.
Nevertheless, for a classifier trained with the focal loss, Fig.~\ref{fig:fig_synth} illustrates that using such heuristics may fail to recover the true class-posterior probability.
Theoretically, since TS only tunes one scalar to optimize the validation objective, it may suffer from model misspecification and could fail to achieve the optimal NLL/ECE w.r.t.~all measurable functions~\citep{reid2010,williamson2016composite}, \ie, it may fail to adjust $\vecq^{\gamma,*}$ to $\veceta$. 
Our $\Psibf$ is the only mapping that can recover the true class-posterior probability $\veceta$ given focal risk minimizer $\vecq^{\gamma,*}$.

\section{Conclusions}
We proved that the focal loss is classification-calibrated but not strictly proper.
We further investigated and pointed out that focal loss can give both underconfident and overconfident classifiers.
Then, we proposed a transformation that can theoretically recover the true class-posterior probability from the focal risk minimizer.
Experimental results showed that the proposed transformation can improve the performance of class-posterior probability estimation. 
\section*{Acknowledgment}
We would like to thank Zhenguo Wu, Yivan Zhang, Zhenghang Cui, and Han Bao for helpful discussion.
Nontawat Charoenphakdee was supported by MEXT scholarship and Google PhD Fellowship program.
Nuttapong Chairatanakul was supported by MEXT scholarship. 
Part of this work is conducted as research activities of AIST - Tokyo Tech Real World Big-Data Computation Open Innovation Laboratory (RWBC-OIL).
Masashi Sugiyama was supported by JST CREST Grant Number JPMJCR18A2.

\appendix
\onecolumn
\begin{figure}[h]
\includegraphics[scale=0.65]{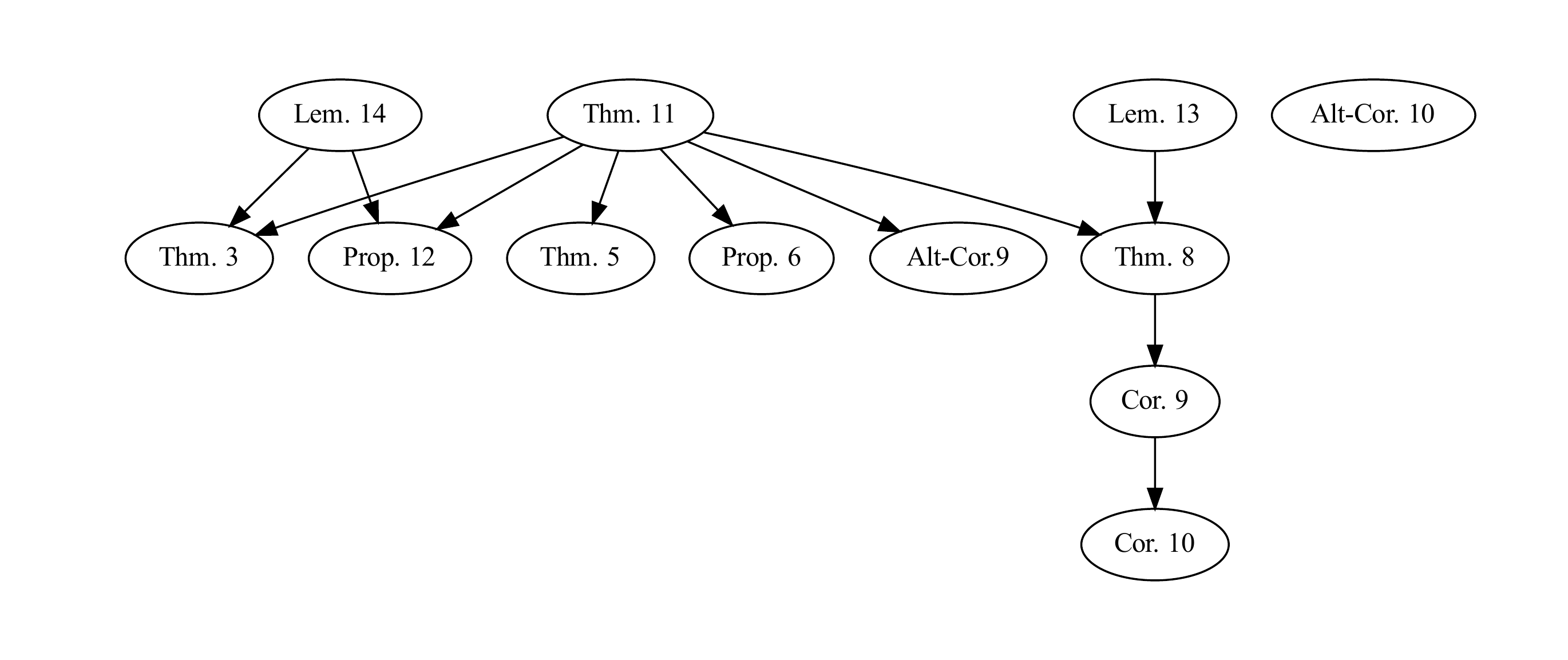}
\caption{\label{fig:proof-dependency} 
The dependency of the proofs. An arrow from node $A$ to node $B$ indicates that the result of node $A$ is required to prove the result of node $B$. Alt-Cor.~\ref{cor:onehalfunder} (\resp, Alt-Cor.~\ref{cor:binaryunder})  denotes an alternative proof of Cor.~\ref{cor:onehalfunder} (\resp, Cor.~\ref{cor:binaryunder}).}
\end{figure}
\section{Proofs}
\label{app:proof}
In this section, we provide the proofs of the results given in the main paper.  To keep the notation uncluttered, sometimes we omit $\vecx$ and use $\eta_i$ and $q_i$ to denote the true class-posterior probability of class $i$ and the score function of class $i$ for the focal loss, where $\gamma$ corresponds to $q_i$ is also omitted but it can be straightforwardly inferred by the context. 
Figure~\ref{fig:proof-dependency} indicates the dependency of the proofs. 
For example, to prove Thm.~\ref{thm:calib}, we can utilize the result of  Lem.~\ref{lem:hstrictly} and Thm.~\ref{thm:recover}.

\subsection*{Proof index:}
\begin{itemize}
    \item Sec.~\ref{app:proof-recover}: Proof of Thm.~\ref{thm:recover}: Recovering class-posterior probability from the focal loss minimizer
    \item Sec.~\ref{app:lemvarphi}: Lem.~\ref{lem:varphiprop}: Properties of $\varphi^\gamma$
    \item Sec.~\ref{app:proof-hstrictly}: Lem.~\ref{lem:hstrictly}: $h^\gamma$ is a strictly increasing function
    \item Sec.~\ref{app:proof-calib}: Proof of Thm.~\ref{thm:calib}: Focal loss is classification-calibrated
    \item Sec.~\ref{app:proof-not-proper}: Proof of Thm.~\ref{thm:not-proper}: Focal loss is not strictly proper
    \item Sec.~\ref{app:proof-focalcorrect}: Proof of Prop.~\ref{thm:focalcorrect}: Where risk minimizer correctly gives the true class-posterior probability
    \item Sec.~\ref{app:proof-underover}: Proof of Thm.~\ref{thm:underover}: Focal loss gives under/overconfident classifier
    \item Sec.~\ref{app:proof-onehalfunder}: Proof of Cor.~\ref{cor:onehalfunder}: Focal loss gives an underestimation of the true class-posterior probability
    \item Sec.~\ref{app:proof-binaryunder}: Proof of Cor.~\ref{cor:binaryunder}: Focal loss gives underconfident classifier in binary classification
    \item Sec.~\ref{app:proof-maintain}: Proof of Prop.~\ref{prop:maintain}: Transformation $\Psibf$ preserves the decision rule
    \item Sec.~\ref{app:proof-alt-onehalfunder}: Alternative proof of Cor.~\ref{cor:onehalfunder}: $\vecq^{\gamma,*}$ is $\veceta$UC if $ \frac{1}{2} \leq \max_y q_y^{\gamma,*}(\vecx) < 1$ and $\vecq^{\gamma,*}(\vecx) \notin \subunif$
    \item Sec.~\ref{app:proof-alt-binaryunder}: Alternative proof of Cor.~\ref{cor:binaryunder}: $\veceta$UC property for $\vecq^{\gamma,*}$ in binary classification where $K=2$
\end{itemize}
\newpage
\subsection{Proof of Thm.~\ref{thm:recover}: Recovering class-posterior probability from the focal loss minimizer}
\label{app:proof-recover}
\begin{proof}
In order to derive a transformation $\Psibf$ that recovers the true class-posterior probability $\eta_i=p(y=i|\vecx)$ for all $i$ from the focal loss minimizer $\vecq^*$\footnote{We omit dependence on $\gamma$ for brevity.}, first consider the following optimization formulation which optimizes $W^{\focalloss}$:
\begin{align}
    \underset{\vecq}{\text{minimize}}
    & 
    -\sum_{i=1}^{K} \eta_i (1-q_i)^\gamma \log q_i \label{eq:minFLobj}
    \\
    \text{subject to}
    &
    \sum_{i=1}^{K}q_i = 1,
    \\
    &
    \vecq \geq \boldsymbol{0}_K. \label{eq:minFLconst2}
\end{align}
Note that this optimization problem is convex with a bounded feasible set, thus an optimal solution exists. Recall that $\vecq^*$ denotes the minimizer of the above optimization problem. Without loss of generality, assume $\eta_i=p(y=i|\vecx)>0$ for all $i$\footnote{If there exists a class $j$ with $\eta_j=p(y=j|\vecx)=0$, then we have $q_j^*=0$. To see this, first let us define $\psi^\gamma(v)=-(1-v)^\gamma\log v$. We can see that $\frac{d}{d v}\psi^\gamma(v) < 0$ for $v>0$. This means that if $q_j^*>0$, then we can transfer $q^*_j$ to other class $k$ with $\eta_k>0$, \eg, $q^*_k:= q^*_k+q^*_j$, then the objective in \eqref{eq:minFLobj} would decrease, which means the original $\vecq^*$ is not the optimum.}.
Observe that $\vecq^*$ must have $q^*_i>0$ for all $i$ since any $q^*_i=0$ will make the objective goes to infinity. With this fact, we can say that the equality in \eqref{eq:minFLconst2} never holds at optimum. 
Therefore, by complementary slackness, the Lagrangian multipliers for constraints in \eqref{eq:minFLconst2} would be zero~\cite{boyd2004convex}, and we can consider the following Lagrangian equation:
\begin{align}
    \mathcal{L}(\vecq, \lambda) = -\sum_{i=1}^{K} \eta_i (1-q_i)^\gamma \log q_i + \lambda\left(\sum_{i=1}^{K}q_i - 1\right)
\end{align}
where $\lambda$ is the Lagrangian multiplier for the equality constraint. Next, we take the derivative with respect to $q_i$ and set to $0$, then solve for $\lambda$ at the optimum $\vecq^*$:
\begin{align}
    \left.\frac{\partial}{\partial q_i}\mathcal{L}(\vecq, \lambda)\right\vert_{\vecq=\vecq^*} = 0 
    & =\eta_i \gamma(1-q^*_i)^{\gamma-1} \log q^*_i - \eta_i\frac{(1-q^*_i)^{\gamma}}{q^*_i}+ \lambda 
    \\
    \lambda 
    &= \eta_i\left(\frac{(1-q^*_i)^{\gamma}-\gamma(1-q^*_i)^{\gamma-1}q^*_i\log q^*_i}{q^*_i} \right)
    \\
    \eta_i\label{eq:minFLetalambda}
    & =
    \frac{\lambda q^*_i}{(1-q^*_i)^{\gamma}-\gamma(1-q^*_i)^{\gamma-1}q^*_i\log q^*_i} 
    \\
    1 = 
    \sum_{i=1}^K\eta_i
    &= \lambda\sum_{i=1}^K \frac{q^*_i}{(1-q^*_i)^{\gamma}-\gamma(1-q^*_i)^{\gamma-1}q^*_i\log q^*_i} 
    \\
    \lambda
    &=
    \frac{1}{\sum_{i=1}^K \frac{q^*_i}{(1-q^*_i)^{\gamma}-\gamma(1-q^*_i)^{\gamma-1}q^*_i\log q^*_i} }.
\end{align}
By replacing the above $\lambda$ in \eqref{eq:minFLetalambda}, we can write $\eta_i$ as a function of $\vecq^*$ as:
\begin{align}
    \eta_i & = \frac{\frac{ q^*_i}{(1-q^*_i)^{\gamma}-\gamma(1-q^*_i)^{\gamma-1}q^*_i\log q^*_i}}{\sum_{j=1}^K \frac{q^*_j}{(1-q^*_j)^{\gamma}-\gamma(1-q^*_j)^{\gamma-1}q^*_j\log q^*_j} }
    \\
    &= \frac{\frac{q^*_i}{\varphi^\gamma(q^*_i)}}{\sum_{j=1}^K\frac{q^*_j}{\varphi^\gamma(q^*_j)}} \\
    &= \boldsymbol{\Psi}_i^\gamma(q^*),
\end{align}
where $\varphi^\gamma(v)=(1-v)^{\gamma}-\gamma(1-v)^{\gamma-1}v\log v$ is the same function defined in \eqref{eq:varphi}. 
As a result, given $\vecq^{\gamma,*}$, one can recover the true class-posterior probability $\veceta$, by using the transformation $\Psibf$.
\end{proof}

\subsection{Lem.~\ref{lem:varphiprop}: Properties of $\varphi^\gamma$}
\label{app:lemvarphi}
We present the following lemma, which describes the properties of the function $\varphi^\gamma:[0,1]\rightarrow\R$, defined as $\varphi(v)=(1-v)^{\gamma}-\gamma(1-v)^{\gamma-1}v\log v$, which plays a vital role in the analysis of the focal loss. 

\begin{lemma}
\label{lem:varphiprop}
(Properties of $\varphi^\gamma$) The function $\varphi^\gamma:[0,1]\to\R$  for all $\gamma>0$ has the following properties:
\begin{enumerate}
    \item $\varphi^\gamma(0)=1$ and $\varphi^\gamma(1)=0$,
    \item $\frac{d}{dv}\varphi^\gamma(v)$ changes sign from positive to negative only once at a point $\tilde{v}\in(0,1)$. In other words, there exists a unique $\tilde{v}\in(0,1)$ such that 
    \begin{enumerate}
        \item  $\frac{d}{dv}\varphi^\gamma(v)>0$ for all $v<\tilde{v}$,
        \item  $\frac{d}{dv}\varphi^\gamma(v)=0$ for $v=\tilde{v}$,
        \item  $\frac{d}{dv}\varphi^\gamma(v)<0$ for all $v>\tilde{v}$.
    \end{enumerate}
    \item There exists a unique $\acute{v} \in (0, 0.5)$ such that $\varphi^\gamma(\acute{v})=1$.
    \item $\varphi^\gamma$ has a unique maximum $\tilde{v}$, where $\tilde{v}\in(0,\acute{v})$.
\end{enumerate}
\end{lemma}

\begin{proof}
\textbf{\underline{Item 1}}: 
We can see that
\begin{align}
    \varphi^\gamma(0)
    &=(1-0)^{\gamma}-\gamma(1-0)^{\gamma-1}\cdot0\cdot\log 0 = 1-0 = 1,
    \\
    \varphi^\gamma(1)
    &=(1-1)^{\gamma}-\gamma(1-1)^{\gamma-1}\cdot1\cdot\log 1 = 0-0 = 0.
\end{align}

\textbf{\underline{Item 2}}: 
To show that $\varphi^\gamma$ changes sign from positive to negative only once in $(0,1)$, 
first we take its derivative and rearrange:
\begin{align}
    \frac{d}{dv}\varphi^\gamma(v)
    &=-\gamma(1-v)^{\gamma-1}+\gamma(\gamma-1)(1-v)^{\gamma-2}v \log v-\gamma(1-v)^{\gamma-1}\log v -\gamma(1-v)^{\gamma-1}
    \\
    &=\gamma(1-v)^{\gamma-2}\left((v-1)+(\gamma-1)v\log v+(v-1)\log v+(v-1)\right)
    \\
    &=\underbrace{\gamma(1-v)^{\gamma-2}}_{=:t(v)\geq0}\underbrace{\left(2v-2-\log v+\gamma v\log v\right)}_{=:s(v)}.\label{eq:varphiFactorSignChange}
\end{align}
From the above, we can see that $t(v)>0$ for $v\in(0,1)$, thus we only need to show that $s(v)$ changes sign only once in for $\frac{d}{dv}\varphi^\gamma(v)$ to also changes sign once in $(0,1)$. 
To see that, notice that $s(v)$ is convex since its second order derivative is always positive:
\begin{align}
    \frac{d}{dv}s(v) 
    &= 2-\frac{1}{v}+\gamma\log v+\gamma,
    \\
    \frac{d^2}{dv^2}s(v)
    &=\frac{1}{v^2}+\frac{\gamma}{v} > 0 \text{ for all } v\in (0,1).
\end{align}
Also, we can compute the following:
\begin{align}
    s(0) &= \infty, \\
    s(1) &= 0, \\
    \frac{d}{dv}s(1) &= 1+\gamma > 0 \text{ for all } v\in (0,1).
\end{align}
From $s(1)=0$ and $\frac{d}{dv}s(1)>0$, we know that there exists $\hat{v}\in(0,1)$ such that $s(\hat{v})<0$. 
With such $\hat{v}$ and that $s(0)=\infty$,  by the intermediate value theorem, there exists $\tilde{v}\in(0,\hat{v})$ such that $s(\tilde{v})=0$. 
Since $s(v)$ is convex in $v$, this $\tilde{v}$ is unique. Therefore, $s(v)$ changes sign only once (from positive to negative) in the range $(0,1)$ at $\tilde{v}$. 
As a result, $\frac{d}{dv}\varphi^\gamma(v)$ also changes sign only once (from positive to negative) at $\tilde{v}\in(0,1)$ (recall Eq.~\eqref{eq:varphiFactorSignChange}). 
This also implies  $\frac{d}{dv}\varphi^\gamma(v)=0$ at $\tilde{v}$ and that $\varphi^\gamma$ has a unique maximum at $\tilde{v}$.

\textbf{\underline{Item 3}}: First, note that $\varphi^\gamma(0.5)=1$ when $\gamma=0$. Next, we can show that $\varphi^\gamma(0.5)$ is a decreasing function in $\gamma$:
\begin{align}
     \left.\frac{d}{d\gamma}\varphi^\gamma(v)\right|_{v=\frac{1}{2}}
     &= \left[(1-v)^\gamma\log v - (1-v)^{\gamma-1}v\log v -\gamma (1-v)^{\gamma-1}v\log^2 v\right]_{v=\frac{1}{2}}
     \\
     & = -\gamma 0.5^\gamma \log^2 0.5 
     \\
     &\leq 0,\label{eq:propVarphiItem3Ineq}
\end{align}
where the equality in Eq.~\eqref{eq:propVarphiItem3Ineq} holds only when $\gamma=0$. 
This implies that we have $\varphi^\gamma(0.5)<1$ for all $\gamma>0$. 
Since for all $\gamma>0$, we have $\varphi^\gamma(0)=1$ and $\frac{d}{dv}\varphi^\gamma(0)=\infty$, there exists $\check{v}\in(0,1)$ where $\varphi^\gamma(\check{v})>1$. 
Thus, by the intermediate value theorem, there exists $\acute{v}\in(\check{v},0.5)\subset(0, 0.5)$ such that $\varphi^\gamma(\acute{v})=1$. 
Since $\frac{d}{dv}\varphi^\gamma(v)$ changes sign from positive to negative only once and together with the above result from the intermediate value theorem, this value $\acute{v}$ must lie on the descending side of $\varphi^\gamma$ and thus has to be unique.

\textbf{\underline{Item 4}}: Since from \textbf{Item 2}, we know that  $\frac{d}{dv}\varphi^\gamma(v)$ changes sign from positive to negative only once at a point $\tilde{v}\in(0,1)$, \ie, $\frac{d}{dv}\varphi^\gamma(\tilde{v})=0$, thus $\varphi^\gamma$ has a unique maximum at $\tilde{v}$. This fact together with that $\varphi^\gamma(0)=1$ (\textbf{Item 1}) and that there exists $\acute{v}\in(0, 0.5)$ with $\varphi^\gamma(\acute{v})=1$ (\textbf{Item 3}), we conclude that the unique maximum $\tilde{v}$ of $\varphi^\gamma$ must be in the range $(0,\acute{v})$.
\end{proof}

\subsection{Lem.~\ref{lem:hstrictly}: $h^\gamma$ is a strictly increasing function}
\label{app:proof-hstrictly}
We present the following lemma, which is highly useful for proving that the focal loss is classification-calibrated and the transformation $\Psibf$ does not change the classifier's decision rule.
\begin{lemma}
\label{lem:hstrictly}
For any $\gamma > 0$ and $v \in (0,1)$,
\begin{align} 
 h^\gamma(v)&=\frac{v}{\varphi^\gamma(v)} = \frac{v}{(1-v)^\gamma-\gamma(1-v)^{\gamma-1}v\log v}
\end{align}
is a strictly increasing function, meaning that $h^\gamma(u) > h^\gamma(v)$ if and only if $u > v$.
\end{lemma}
\begin{proof}
Since $h^\gamma$ is differentiable, it suffices to prove that $h^\gamma$ is strictly increasing if $\frac{d}{dv}h^\gamma(v) > 0$.
By taking the derivative of $h^\gamma(v)$, we have
\begin{align}
    \frac{d}{dv}h^\gamma(v) &= \frac{1}{(1-v)^\gamma-\gamma(1-v)^{\gamma-1}v\log v} - \frac{-2\gamma v(1-v)^{\gamma-1} + (\gamma-1)\gamma v^2 (1-v)^{\gamma-2}\log v-\gamma (1-v)^{\gamma-1}v\log v}{[(1-v)^\gamma-\gamma(1-v)^{\gamma-1}v\log v]^2} \\
    &= \frac{(1-v)^\gamma-\gamma(1-v)^{\gamma-1}v\log v +2\gamma v(1-v)^{\gamma-1} - (\gamma-1)\gamma v^2 (1-v)^{\gamma-2}\log v+\gamma (1-v)^{\gamma-1}v\log v}{[(1-v)^\gamma-\gamma(1-v)^{\gamma-1}v\log v]^2} \\
        &= \frac{(1-v)^\gamma+2\gamma v(1-v)^{\gamma-1} - (\gamma-1)\gamma v^2 (1-v)^{\gamma-2}\log v}{[(1-v)^\gamma-\gamma(1-v)^{\gamma-1}v\log v]^2}.
\end{align}
Since the goal is to show that $\frac{d}{dv}h^\gamma(v) > 0$, the denominator of $\frac{d}{dv}h^\gamma(v)$ can be ignored because $[(1-v)^\gamma-\gamma(1-v)^{\gamma-1}v\log v]^2 > 0$ for $v \in (0,1)$. 
We denote $\phi_1^\gamma$ the numerator of $\frac{d}{dv}h^\gamma$ as follows: 
\begin{align}
    \phi_1^\gamma(v) = (1-v)^\gamma+2\gamma v(1-v)^{\gamma-1} - (\gamma-1)\gamma v^2 (1-v)^{\gamma-2}\log v.
\end{align}
Now it suffices to show that $\phi_1^\gamma(v) > 0$ for all $\gamma > 0$ to prove that $\frac{d}{dv}h^\gamma(v) > 0$.
We split the proof into two cases, which are the case where $\gamma \geq 1$ and $ 0 < \gamma < 1$.

\noindent \textbf{Case 1:} $\gamma \geq 1$. 

It is straightforward to see that $\phi_1^\gamma(v) >0$ because $(1-v)^\gamma > 0$, $2\gamma v(1-v)^{\gamma-1} >0$ and $- (\gamma-1)\gamma v^2 (1-v)^{\gamma-2}\log v \geq 0$ for $v \in (0,1). $
This is because of the sum of two positive quantities and one nonnegative quantity must be positive.

\noindent \textbf{Case 2:} $0 \leq \gamma < 1$. 

We begin by expressing $\phi_1^\gamma(v)$ as follows:
\begin{align}
    \phi_1^\gamma(v) &= (1-v)^\gamma+2\gamma v(1-v)^{\gamma-1} - (\gamma-1)\gamma v^2 (1-v)^{\gamma-2}\log v \\
    &= (1-v)^{\gamma-2} [(1-v)^{2} + 2\gamma v(1-v) - (\gamma-1)\gamma v^2 \log v] \\
    &= (1-v)^{\gamma-2} [(1-v)^{2} + 2\gamma v(1-v) - \gamma^2 v^2 \log v + \gamma v^2 \log v].
\end{align}
Since we want to show that $\phi_1^\gamma(v) > 0$, we can ignore a positive value $(1-v)^{\gamma-2}$ and prove that
\begin{align}
    (1-v)^{2} + 2\gamma v(1-v) - \gamma^2 v^2 \log v + \gamma v^2 \log v > 0.
\end{align}

Since $-\gamma^2 v^2 \log v \geq 0$, we may ignore this term and it is sufficient to prove that 
\begin{align}
    \phi_2^\gamma(v) = (1-v)^{2} + 2\gamma v(1-v)  + \gamma v^2 \log v > 0.
\end{align}

By substitution, we have
\begin{align}
    \phi_2^\gamma(0) = 1, \\
    \phi_2^\gamma(1) = 0.
\end{align}
Then, we show that $\phi_2^\gamma$ is a decreasing function by showing that $\frac{d}{d_v}\phi_2^\gamma(v) < 0$. 

The derivative of $\phi_2^\gamma(v)$ can be expressed as:
\begin{align}
    \frac{d}{d_v}\phi_2^\gamma(v) &= \frac{d}{d_v} (1-v)^{2} + \frac{d}{d_v} 2\gamma v(1-v)  + \frac{d}{d_v} \gamma v^2 \log v \\ 
    &= -2(1-v) + 2\gamma - 4 \gamma v + \gamma v + 2\gamma v \log v \\ 
    &= 2v-2+2\gamma-3\gamma v+2\gamma v\log v
\end{align}

By substitution, we have
\begin{align}
    \frac{d}{d_v}\phi_2^\gamma(0)= -2+2\gamma < 0, \\
    \frac{d}{d_v}\phi_2^\gamma(1)= -\gamma < 0.
\end{align}

Moreover, $\frac{d}{d_v}\phi_2^\gamma$ is convex because
\begin{align}
     \frac{d^2}{dv^2}\phi_2^\gamma(v) &= 2-3\gamma+ 2\gamma[1 + \log v], \\
    \frac{d^3}{dv^3}\phi_2^\gamma(v) &= \frac{2\gamma}{v} > 0.
\end{align}

Based on the fact that $\frac{d}{d_v}\phi_2^\gamma(0) <0$, $\frac{d}{d_v}\phi_2^\gamma(1)<0$, and  $\frac{d}{d_v}\phi_2^\gamma$ is convex, we can conclude that $\frac{d}{d_v}\phi_2^\gamma(v) <0$ for $v\in(0,1)$ because it must be less than $\max(\frac{d}{d_v}\phi_2^\gamma(0), \frac{d}{d_v}\phi_2^\gamma(1))$~\cite{boyd2004convex}. 
Therefore, $\frac{d}{d_v}\phi_2^\gamma(v) < 0$ and thus $\phi_2^\gamma$ is a decreasing function. 

Next, because $\phi_2^\gamma(0)=1$, $\phi_2^\gamma(1)=0$, and $\phi_2^\gamma$ is a decreasing function, we know that $\phi_2^\gamma(v) > 0$ for $v \in (0,1)$, which proves that $\phi_1^\gamma(v) > 0$ for $0< \gamma <1$. 

By combining the results of \textbf{Case 1} and \textbf{Case 2}, we have $\phi_1^\gamma(v) > 0$  for all $\gamma > 0$, which yields $\frac{d}{dv}h^\gamma(v) > 0$. 
Therefore, $h^\gamma$ is a strictly increasing function.
\end{proof}

\subsection{Proof of Thm.~\ref{thm:calib}: Focal loss is classification-calibrated}
\label{app:proof-calib}
\begin{proof}
To prove that the focal loss is classification-calibrated, we combine the result of Thm.~\ref{thm:recover} and the the existing result which suggests that a surrogate loss is classification-calibrated if it has satisfies the strictly order-preserving property~\cite{zhang2004statisticalmulti}. 

The order-preserving property suggests that for any $\vecx$, the pointwise conditional risk $W^{\ell}$ has the risk minimizer $\vecq^{\ell,*}(\vecx)$ such that $ q^{\ell,*}_i(\vecx) < q^{\ell,*}_j(\vecx) \Rightarrow \eta_i(\vecx) < \eta_j(\vecx)$, then a loss function $\ell$ is classification-calibrated~\cite{zhang2004statisticalmulti}.

From Thm.~\ref{thm:recover}, we know that
 \begin{align}
     \veceta(\vecx)=
     \Psibf(\vecq^{\gamma,*}(\vecx)),
    \end{align}
     where 
     \begin{align}
     \small
     \Psibf(\vecv) &= [\Psi_1^\gamma(\vecv), \ldots, \Psi_K^\gamma(\vecv)]^\top
     \hfill, \\
     \Psi_i^\gamma(\vecv) &= \frac{h^\gamma(v_i)}{\sum_{l=1}^K h^\gamma(v_l)}
    \hfill ,\\
        h^\gamma(v)&=\frac{v}{\varphi^\gamma(v)} = \frac{v}{(1-v)^\gamma-\gamma(1-v)^{\gamma-1}v\log v}.
    \end{align}

From Lem.~\ref{lem:hstrictly}, we know that $h^\gamma$ is a strictly increasing function. 
Thus, we have 
\begin{align}
    q^{\gamma,*}_i(\vecx) < q^{\gamma,*}_j(\vecx) \Rightarrow h^\gamma(q^{\gamma,*}_i(\vecx)) < h^\gamma(q^{\gamma,*}_j(\vecx))  .
\end{align}
Given $\vecx$, the denominator of $\Psi^\gamma_i(\vecq^{\gamma,*}(\vecx))$, i.e., $\sum_{l=1}^K q^{\gamma,*}_i(\vecx)$ is identical for all classes.
Also, the numerator of $\Psi^\gamma_i(\vecq^{\gamma,*}(\vecx))$ is a strictly increasing function $h^\gamma(q^{\gamma,*}_i(\vecx))$. 
Based on these facts, we have
\begin{align}
        q^{\gamma,*}_i(\vecx) < q^{\gamma,*}_j(\vecx) \Rightarrow \Psi^\gamma_i(\vecq^{\gamma,*}(\vecx)) < \Psi^\gamma_j(\vecq^{\gamma,*}(\vecx))  .
\end{align}

Since $\Psibf(\vecq^{\gamma,*})$ is equal to $\veceta(\vecx)$ and note that $\Psi^\gamma_i(\vecq^{\gamma,*}(\vecx)) = \eta_i(\vecx)$, we have 
\begin{align}
\label{eq:order-preserving}
      q^{\gamma,*}_i(\vecx) < q^{\gamma,*}_j(\vecx)  \Rightarrow \eta_i(\vecx) < \eta_j(\vecx).
\end{align}

Eq.~\eqref{eq:order-preserving} indicates that the focal loss satisfies the strictly order-preserving property for all $\gamma \geq 0$, which is sufficient to conclude that that the focal loss is classification-calibrated.
\end{proof}
Note that $\argmax_y q^{\gamma,*}_y(\vecx)=\argmax_y \eta_y(\vecx)$ indicates that the decision rule of the focal risk minimizer is equivalent to that of the Bayes-optimal classifier.
As a result, the Bayes-optimal classifier can be achieved by minimizing the focal risk minimizer, \ie, $R^{\zoloss}(f^{\vecq^{\gamma,*}}) = R^{\zoloss}(f^{\zoloss,*})$. 

\subsection{Proof of Thm.~\ref{thm:not-proper}: Focal loss is not strictly proper}
\label{app:proof-not-proper}
\begin{proof}
Recall that a loss $\ell: \Delta^K \times \Delta^K \to \R$ is strictly proper if $\ell(\vecu,\vecv)$ is minimized if and only if $\vecu=\vecv$ by the definition of strict properness. 
We will prove that the focal loss is not strictly proper for all $\gamma >0$ by showing a counterexample that the focal loss can be minimized when $\vecu \neq \vecv$. 

By the definition of the focal loss:
\begin{equation}
    \focalloss(\vecu, \vecv) = -\sum_{i=1}^K v_i(1-u_i)^\gamma \log(u_i).
\end{equation}
For any $\vecx$, we have
\begin{align}
    \focalloss(\vecq(\vecx), \veceta(\vecx)) &= -\sum_{y=1}^K \eta_y(\vecx) (1-q_y(\vecx))^\gamma \log(q_y(\vecx)) \\
    &= \sum_{y \in \domy} \eta_y(\vecx) \focalloss\big( \vecq(\vecx), \vece_y \big) \\
    &=  W^{\focalloss} \big(\vecq(\vecx); \veceta(\vecx) \big).
\end{align}
It can be observed that $\focalloss(\vecq(\vecx), \veceta(\vecx))$ coincides with the pointwise conditional risk w.r.t. the focal loss $ W^{\focalloss}\big(\vecq(\vecx); \veceta(\vecx) \big)$. 
Note that the simplex $\vecq$ that minimizes $ W^{\focalloss}$ is the focal risk minimizer $\vecq^{\gamma,*}$. 
Based on Thm.~\ref{thm:recover}, we know that although $\vecq^{\gamma,*}(\vecx)$ minimizes $\focalloss(\vecq(\vecx), \veceta(\vecx))$, we have $\vecq^{\gamma,*}\neq \veceta$ because $\Psibf$ that transforms $\vecq^{\gamma,*}$ to the true class-posterior probability is not an identity function unless $\gamma \neq 0$. 
This counterexample is sufficient to conclude that the focal loss is not strictly proper for $\gamma > 0$ since the focal loss $\focalloss(\vecu, \vecv)$ can be minimized when $\vecu \neq \vecv$, which contradicts the definition of strict properness.
\end{proof}
\subsection{Proof of Prop.~\ref{thm:focalcorrect}: Where risk minimizer correctly gives the true class-posterior probability}
\label{app:proof-focalcorrect}
\begin{proof}
Recall $\subunif = \{\vecv \in \Delta^K: v_i \in \{0, \max_j v_j\}\}$. 
From Thm.~\ref{thm:recover}, we know that  $\Psibf(\vecv) = [\Psi_1^\gamma(\vecv), \ldots, \Psi_K^\gamma(\vecv)]^\top$, $\Psi_i^\gamma(\vecv) = \frac{h^\gamma(v_i)}{\sum_{l=1}^K h^\gamma(v_l)}$,
        $h^\gamma(v)=\frac{v}{\varphi^\gamma(v)} = \frac{v}{(1-v)^\gamma-\gamma(1-v)^{\gamma-1}v\log v}$.
We will show that If $\vecq^{\gamma,*}(\vecx) \in \subunif$, then $\vecq^{\gamma,*}(\vecx) = \veceta(\vecx)$ by proving that $\Psi_i^\gamma(\vecq^{\gamma,*}(\vecx)) = \eta_i(\vecx) = q_i^{\gamma,*}(\vecx)$ for all $i \in \domy$.

\textbf{Case 1:} $\Psi_i^\gamma(\vecq^{\gamma,*}(\vecx)) =q_i^{\gamma,*}(\vecx)=0$. 

Since 
\begin{align}
    h^\gamma(q_i^{\gamma,*}(\vecx)) = \frac{q_i^{\gamma,*}(\vecx)}{\varphi^\gamma(q_i^{\gamma,*}(\vecx))} = \frac{0}{\varphi^\gamma(0)} = \frac{0}{1} = 0,
\end{align}
we have
\begin{align}
    \Psi_i^\gamma(\vecq^{\gamma,*}(\vecx)) &= \frac{h^\gamma(q_i^{\gamma,*}(\vecx))}{\sum_{l=1}^K h^\gamma(q_l^{\gamma,*}(\vecx))} \\ &=  \frac{h^\gamma(0)}{\sum_{l=1}^K h^\gamma(q_l^{\gamma,*}(\vecx))} \\ &= 0 \\
    &= q_i^{\gamma,*}(\vecx).
\end{align}

\textbf{Case 2:} $\Psi_i^\gamma(\vecq^{\gamma,*}(\vecx)) =q_i^{\gamma,*}(\vecx)=\max_j q_j^{\gamma,*}(\vecx)$.

Since other $q_y$ where $y\neq i$ can be  only either  $q_y=q_i=\max_j v_j$ or $q_y = 0$. 
Let $k \leq K$ be a number of classes that is non-zero.
Thus, we have $\max_j q_j^{\gamma,*}(\vecx) = \frac{1}{k}$.
It can be observed that 
\begin{align}
    \sum_{l=1}^K h^\gamma(q_l^{\gamma,*}(\vecx)) &= kh^\gamma(\frac{1}{k}) + (K-k) h^\gamma(0) \\
    &= kh^\gamma(\frac{1}{k}) + (K-k)0 \\
    &= kh^\gamma(\frac{1}{k}).
\end{align}
Therefore, we have
\begin{align}
    \Psi_i^\gamma(\vecq^{\gamma,*}(\vecx)) &= \frac{h^\gamma(\frac{1}{k})}{\sum_{l=1}^K h^\gamma(q_l^{\gamma,*}(\vecx))} \\
    &= \frac{h^\gamma(\frac{1}{k})}{kh^\gamma(\frac{1}{k})} \\ &= \frac{1}{k} \\
    &= \max_j q_j^{\gamma,*}(\vecx)\\
     &=q_i^{\gamma,*}(\vecx).
\end{align}

We can focus on a uniform vector over a subset of classes and zeros otherwise. 

\noindent By combining \textbf{Case 1} and \textbf{Case 2}, we can conclude that if $\vecq^{\gamma,*}(\vecx) \in \subunif$, we have $\vecq^{\gamma,*}(\vecx) = \veceta(\vecx)$.

\end{proof}
\subsection{Proof of Thm.~\ref{thm:underover}: Focal loss gives under/overconfident classifier}
\label{app:proof-underover}
\begin{proof}
Consider the focal loss $\focalloss$ where $\gamma > 0$.
Define $\tau^{\gamma}_\mathrm{oc} = \argmax_v \varphi^\gamma(v)$ and $\tau^{\gamma}_\mathrm{uc} \in (0,1)$ such that $\varphi^\gamma(\tau^{\gamma}_\mathrm{uc})=1$. 
If $\max_y q^{\gamma,*}_y(\vecx) \neq \frac{1}{K}$, we have
\begin{enumerate}
    \item $0 < \tau^{\gamma}_\mathrm{oc} < \tau^{\gamma}_\mathrm{uc} < 0.5$.
    \item $\vecq^{\gamma,*}$ is $\veceta$OC for $  \max_y q^{\gamma,*}_y(\vecx) \in(0, \tau^{\gamma}_\mathrm{oc}]$.
    \item $\vecq^{\gamma,*}$ is $\veceta$UC for $  \max_y q^{\gamma,*}_y(\vecx) \in[\tau^{\gamma}_\mathrm{uc}, 1 )$ .
\end{enumerate}
\textbf{\underline{Item 1}}: $0 < \tau^{\gamma}_\mathrm{oc} < \tau^{\gamma}_\mathrm{uc} < 0.5$.

For any $\gamma > 0$, from Lem.~\ref{lem:varphiprop} tells us that $\varphi^\gamma(0)=1$, there exists a unique $\acute{v} \in (0, 0.5)$ such that $\varphi^\gamma(\acute{v})=1$, and $\frac{d}{dv}\varphi^\gamma(v)<0$ for all $v>\tilde{v}$. $\varphi^\gamma$ also has a unique maximum in the range $(0,\acute{v})$. 
We will use these facts prove that $0 < \tau^{\gamma}_\mathrm{oc} < \tau^{\gamma}_\mathrm{uc} < 0.5$ by showing that $\tau^{\gamma}_\mathrm{uc} < 0.5$ and $0 < \tau^{\gamma}_\mathrm{oc} < \tau^{\gamma}_\mathrm{uc}$.

\textbf{Case 1:} $\tau^{\gamma}_\mathrm{uc} < 0.5$.

From the definition of $\tau^{\gamma}_\mathrm{uc}$, we can conclude that $\tau^{\gamma}_\mathrm{uc} = \acute{v} < 0.5$ since $\acute{v} \in (0, 0.5)$. 
Note that there does not exist $v\in [0.5,1)$ such that $\varphi^\gamma(v)=1$ because $\varphi^\gamma(0.5) < 1$ for all $\gamma > 0$. 
This is due to the fact that $\frac{d}{dv}\varphi^\gamma(v)<0$ for all $v>\tilde{v}$, i.e., $\varphi^\gamma$ is a decreasing function when $v > \acute{v}$.

\textbf{Case 2:} $0 < \tau^{\gamma}_\mathrm{oc} < \tau^{\gamma}_\mathrm{uc}$. 

From the definition of $\tau^{\gamma}_\mathrm{oc}$, it is a unique maximum of $\varphi^\gamma$ and therefore  $\tau^{\gamma}_\mathrm{oc} \in (0,\acute{v})$, which is equivalent to $ \tau^{\gamma}_\mathrm{oc} \in (0,\tau^{\gamma}_\mathrm{uc})$. 
This implies that $0< \tau^{\gamma}_\mathrm{oc} < \tau^{\gamma}_\mathrm{uc}$.

By combining \textbf{Case 1} and \textbf{Case 2}, we can conclude that $0 < \tau^{\gamma}_\mathrm{oc} < \tau^{\gamma}_\mathrm{uc} < 0.5$.

\noindent \textbf{\underline{Item 2}}: $\vecq^{\gamma,*}$ is $\veceta$OC if $  \max_y q^{\gamma,*}_y(\vecx) \in(\frac{1}{K}, \tau^{\gamma}_\mathrm{oc}]$.
Recall the definition of $\veceta$OC:
\begin{equation}
     \max_y q^{\ell,*}_{y}(\vecx) - \max_y \eta_y(\vecx) > 0.
\end{equation}

Without loss of generality let us define $i \in \argmax_y q^{\ell,*}_{y}(\vecx)$. Since the focal loss is classification-calibrated suggested in Thm.~\ref{thm:calib}, the max-index of $q^{\ell,*}$ and $\veceta$ are identical.

From Thm.~\ref{thm:recover}, we can rewrite $ \eta_i(\vecx)$ as
\begin{align}
     \eta_i(\vecx) &= \Psi_i^\gamma(\vecq^{\gamma,*}(\vecx)) \\
     &=  \Psi_i^\gamma(\vecq^{\gamma,*}(\vecx))\\
     &=  \frac{h^\gamma(q_i^{\gamma,*}(\vecx))}{\sum_{l=1}^K h^\gamma(q_l^{\gamma,*}(\vecx))}.
\end{align}
Note that 
\begin{align}
        h^\gamma(v)&=\frac{v}{\varphi^\gamma(v)}\hspace{-.5px}=\hspace{-.5px}\frac{v}{(1-v)^\gamma-\gamma(1-v)^{\gamma-1}v\log v}.
\end{align}

By the definition of $\veceta$OC, we have
\begin{align}
q_i^{\gamma,*}(\vecx) - \frac{h^\gamma(q_i^{\gamma,*}(\vecx))}{\sum_{l=1}^K h^\gamma(q_l^{\gamma,*}(\vecx))} &> 0 \\
    q_i^{\gamma,*}(\vecx) &> \frac{h^\gamma(q_i^{\gamma,*}(\vecx))}{\sum_{l=1}^K h^\gamma(q_l^{\gamma,*}(\vecx))} \\
    &=\frac{\frac{q_i^{\gamma,*}(\vecx)}{\varphi^\gamma(q_i^{\gamma,*}(\vecx))}}{\sum_{l=1}^K h^\gamma(q_l^{\gamma,*}(\vecx))} \\
\end{align}

By dividing both sides by $q_i^{\gamma,*}(\vecx)$, we have.
\begin{align}
    1 &> \frac{\frac{1}{\varphi^\gamma(q_i^{\gamma,*}(\vecx))}}{\sum_{l=1}^K h^\gamma(q_l^{\gamma,*}(\vecx))} \\
   \sum_{l=1}^K h^\gamma(q_l^{\gamma,*}(\vecx)) &> \frac{1}{\varphi^\gamma(q_i^{\gamma,*}(\vecx))}
\end{align}

By the definition of $h^\gamma$, we have
\begin{align}
    \sum_{l=1}^K q_l^{\gamma,*}(\vecx) \frac{1}{\varphi^\gamma(q_l^{\gamma,*}(\vecx))} &> \frac{1}{\varphi^\gamma(q_i^{\gamma,*}(\vecx))}
    \label{eq:app-uccond}
\end{align}
Next we will prove that Ineq.~\eqref{eq:app-uccond} is true to verify that $\vecq^{\gamma,*}(\vecx)$ is $\veceta$OC.

Because the convex combination is not smaller the minimum value~\cite{boyd2004convex}, to prove that Ineq.~\eqref{eq:app-uccond} holds, it suffices to show that for all $l \in \domy$:
\begin{align}
    \frac{1}{\varphi^\gamma(q_l^{\gamma,*}(\vecx))} \geq \frac{1}{\varphi^\gamma(q_i^{\gamma,*}(\vecx))},
\end{align}
which is equivalent to
\begin{align}
    \varphi^\gamma(q_i^{\gamma,*}(\vecx)) \geq \varphi^\gamma(q_l^{\gamma,*}(\vecx)),
    \label{eq:app-uccondsimple}
\end{align}
and there exists at least one $l$ such that the strict inequality holds, i.e., $\varphi^\gamma(q_i^{\gamma,*}(\vecx)) > \varphi^\gamma(q_l^{\gamma,*}(\vecx))$. 
Since $\vecq^{\gamma,*}(\vecx) \notin \subunif$, there exists $l$ such that $q_l^{\gamma,*}(\vecx)) \neq q_i^{\gamma,*}(\vecx))$ and $q_l^{\gamma,*}(\vecx))$ is non-zero. 
From Lem.~\ref{lem:varphiprop}, we know that $\frac{d}{dv}\varphi^\gamma(v)>0$ for all $v \in (0,\tau^{\gamma}_\mathrm{oc})$. 
Therefore $\varphi^\gamma$ is an increasing function in $(0, \tau^{\gamma}_\mathrm{oc}]$.
Since $q_l^{\gamma,*}(\vecx) \leq q_i^{\gamma,*}(\vecx)$, for $q_i^{\gamma,*}(\vecx) \in (\frac{1}{K}, \tau^{\gamma}_\mathrm{oc}]$, we have
\begin{align}
    \varphi^\gamma(q_i^{\gamma,*}(\vecx)) \geq \varphi^\gamma(q_l^{\gamma,*}(\vecx))
\end{align}
for all $l \in \domy$, where the equality holds only if $q_i^{\gamma,*}(\vecx) = q_j^{\gamma,*}(\vecx)$.
Thus, we have proven that Ineq.\eqref{eq:app-uccond} holds,
which indicates that $\vecq^{\gamma,*}$ is $\veceta$OC if $  \max_y q^{\gamma,*}_y(\vecx) \in(\frac{1}{K}, \tau^{\gamma}_\mathrm{oc}]$

\textbf{\underline{Item 3}}:  $\vecq^{\gamma,*}$ is $\veceta$UC if $  \max_y q^{\gamma,*}_y(\vecx) \in[\tau^{\gamma}_\mathrm{uc}, 1)$. 
Recall the definition of $\veceta$UC:
\begin{equation}
     \max_y q^{\ell,*}_{y}(\vecx) - \max_y \eta_y(\vecx) < 0.
\end{equation}

By using exactly the same technique for proving \textbf{Item 2} but we flip the sign to validate $\veceta$-underconfidence. 
We know that it suffices to prove that for all $l \in \domy$,
\begin{align}
    \varphi^\gamma(q_i^{\gamma,*}(\vecx)) \leq \varphi^\gamma(q_l^{\gamma,*}(\vecx))
\end{align}
and there exists at least one $l$ such that the strict inequality holds, i.e., $\varphi^\gamma(q_i^{\gamma,*}(\vecx)) < \varphi^\gamma(q_l^{\gamma,*}(\vecx))$.

Recall Lem.~\ref{lem:varphiprop} indicates that  $\frac{d}{dv}\varphi^\gamma(v)>0$ for all $v \in (0,\tau^{\gamma}_\mathrm{oc})$.
Thus, $\varphi^\gamma$ is an increasing function in $(0,\tau^{\gamma}_\mathrm{oc})$. 
Moreover, $\frac{d}{dv}\varphi^\gamma(v)<0$ for all $v \in (\tau^{\gamma}_\mathrm{oc},1)$.
Thus, $\varphi^\gamma$ is a decreasing function in $(\tau^{\gamma}_\mathrm{oc}, 1)$.
Also, we know that $\varphi^\gamma(0) = \varphi^\gamma(\tau^{\gamma}_\mathrm{uc})=1$.
Furthermore, Thm.~\ref{def:underover} indicates that $\tau^{\gamma}_\mathrm{oc}<\tau^{\gamma}_\mathrm{uc}$. 
We will use these facts to prove our result.

Since $\varphi^\gamma(0) = \varphi^\gamma(\tau^{\gamma}_\mathrm{uc})=1$, $\varphi^\gamma$ is increasing in $(0,\tau^{\gamma}_\mathrm{oc})$, and decreasing in $(\tau^{\gamma}_\mathrm{oc}, 1)$, we have $ \varphi^\gamma(v) > 1$ for $v\in (0,\tau^{\gamma}_\mathrm{uc})$. 
Note that $q_l^{\gamma,*}(\vecx) \leq q_i^{\gamma,*}(\vecx)$, for $q_i^{\gamma,*}(\vecx)$.
We know that $\vecq^{\gamma,*}(\vecx) \notin \subunif$, which implies that there exists $l$ such that $q_l^{\gamma,*}(\vecx)) \neq q_i^{\gamma,*}(\vecx))$ and $q_l^{\gamma,*}(\vecx))$ is non-zero. 
For such $q_l^{\gamma,*}(\vecx)$ , we have
\begin{align}
    \varphi^\gamma(q_i^{\gamma,*}(\vecx)) < \varphi^\gamma(q_l^{\gamma,*}(\vecx)).
\end{align}

As a result, for $q_i^{\gamma,*}(\vecx) \in [\tau^{\gamma}_\mathrm{uc}, 1)$, we have
\begin{align}
    \varphi^\gamma(q_i^{\gamma,*}(\vecx)) \leq \varphi^\gamma(q_l^{\gamma,*}(\vecx))
\end{align}
for all $l \in \domy$, and there are only two possibilities that the equality holds:
\begin{enumerate}
\item $q_i^{\gamma,*}(\vecx) = q_j^{\gamma,*}(\vecx)$ 
\item $q_i^{\gamma,*}(\vecx) = \tau^{\gamma}_\mathrm{uc}$ and $q_j^{\gamma,*}(\vecx) = 0$.
\end{enumerate}
Note that there exists $q_l^{\gamma,*}(\vecx))$ such that $   \varphi^\gamma(q_i^{\gamma,*}(\vecx)) < \varphi^\gamma(q_l^{\gamma,*}(\vecx))$ since $\vecq^{\gamma,*}(\vecx) \notin \subunif$. 
Thus, we can conclude that $\vecq^{\gamma,*}$ is $\veceta$UC if $  \max_y q^{\gamma,*}_y(\vecx) \in [\tau^{\gamma}_\mathrm{uc}, 1)$.

\end{proof}
\subsection{Proof of Cor.~\ref{cor:onehalfunder}: Focal loss gives an underestimation of the true class-posterior probability}
\label{app:proof-onehalfunder}
\begin{proof}
From Thm.~\ref{thm:underover}, we know that for all $\gamma >0$, we have  $0 \leq \tau^{\gamma}_\mathrm{oc} < \tau^{\gamma}_\mathrm{uc} < 0.5$.
Moreover, Thm.~\ref{thm:underover} also tells us that $\vecq^{\gamma,*}$ is $\veceta$UC if $ \max_y q^{\gamma,*}_y(\vecx) \in [\tau^{\gamma}_\mathrm{uc}, 1)$.
Since $\tau^{\gamma}_\mathrm{uc} < 0.5$, it is straightforward to see that $(0.5, 1) \subset [\tau^{\gamma}_\mathrm{uc}, 1)$ and therefore $\vecq^{\gamma,*}$ is $\veceta$UC if $  \max_y q^{\gamma,*}_y(\vecx) \in(0.5, 1)$.
\end{proof}
\subsection{Proof of Cor.~\ref{cor:binaryunder}: Focal loss gives underconfident classifier in binary classification}
\begin{proof}
\label{app:proof-binaryunder}
We know that $\max_y \vecq^{\gamma,*}(\vecx) \geq 0.5$ in binary classification since $\max_y \vecq^{\gamma,*}(\vecx) \geq \frac{1}{K}$, and $K=2$. 
As a result, unless the label distribution is uniform, i.e., $\max_y \eta_y (\vecx) = 0.5$ or the label distribution is deterministic, i.e., $\max_y \eta_y (\vecx) = 1$, $\vecq^{\gamma,*}$ must always be $\veceta$UC since $\max_y \eta_y (\vecx) \in (0.5,1)$ as proven in Cor.~\ref{cor:onehalfunder}.
\end{proof}
\subsection{Proof of Prop.~\ref{prop:maintain}: Transformation $\Psibf$ preserves the decision rule}
\label{app:proof-maintain}
\begin{proof}
From Thm.~\ref{thm:recover}, we know that
     $\Psibf(\vecv) = [\Psi_1^\gamma(\vecv), \ldots, \Psi_K^\gamma(\vecv)]^\top$, $\Psi_i^\gamma(\vecv) = \frac{h^\gamma(v_i)}{\sum_{l=1}^K h^\gamma(v_l)}$, and
        $h^\gamma(v)=\frac{v}{\varphi^\gamma(v)}$. 
Note that the denominator of $\Psi^\gamma_i(\vecv$, i.e., $\sum_{l=1}^K v_i(\vecv)$ is identical for all classes.
Thus, it suffices to look at the numerator to determine which index has the largest value. 
It can be observed that the numerator $\Psi^\gamma_i(\vecq^{\gamma,*}(\vecx))$ is a strictly increasing function $h^\gamma(q^{\gamma,*}_i(\vecx))$, as proven in Lem.~\ref{lem:hstrictly}. 
Based on these facts, we have
\begin{align}
         \argmax_i\ \Psi_i^\gamma(\vecv)= \argmax_i v_i.
\end{align}
Note that this holds for any simplex input, not only the output of the focal risk minimizer $\vecq^{\gamma,*}$.
\end{proof}

\subsection{Alternative proof of Cor.~\ref{cor:onehalfunder}: $\vecq^{\gamma,*}$ is $\veceta$UC if $ \frac{1}{2} \leq \max_y q_y^{\gamma,*}(\vecx) < 1$ and $\vecq^{\gamma,*}(\vecx) \notin \subunif$}
\label{app:proof-alt-onehalfunder}
Here, we present an alternative proof of Cor.~\ref{cor:onehalfunder}, which is independent from Thm.~\ref{thm:underover}, Lem.~\ref{lem:varphiprop} and Lem.~\ref{lem:hstrictly}.

\begin{proof}
Let $q=\max_y q_y^{\gamma,*}(\vecx)$ be the highest score in the simplex output from the focal risk minimizer $\vecq^{\gamma,*}$ and $\eta = \max_y \eta_y(\vecx)$ be the true class-posterior probability of the most probable class. 
 
In the multiclass cases, we show that $\vecq^{\gamma, *}$ is $\veceta$UC when $\frac{1}{2} \leq q < 1$. That is, for $q \in [\frac{1}{2},1)$, we have
\begin{align}
\eta-q > 0.
\end{align}

From Thm.~\ref{thm:recover}, we can replace $\eta$ and rewrite the above inequality as
\begin{align}
    \frac{\frac{q}{(1-q)^\gamma-\gamma(1-q)^{\gamma-1}q\log q}}{\sum_{l=1}^K\frac{q_l}{(1-q_l)^\gamma-\gamma(1-q_l)^{\gamma-1}q_l\log q_l}}-q 
    &> 0
    \\
    \frac{q}{(1-q)^\gamma-\gamma(1-q)^{\gamma-1}q\log q}
    & 
    > q\sum_{l=1}^K\frac{q_l}{(1-q_l)^\gamma-\gamma(1-q_l)^{\gamma-1}q_l\log q_l}
    \\
    \frac{1}{\varphi(q)}
    & 
    > \sum_{l=1}^K\frac{q_l}{\varphi(q_l)},\label{eq:multIneq}
\end{align}
where we $\varphi(q)=(1-q)^\gamma-\gamma(1-q)^{\gamma-1}q\log q$. 
Note that since $q_l<1$, we have $\varphi(q_l)>0$. 
In words, Ineq.~\eqref{eq:multIneq} says it suffices to show that $\frac{1}{\varphi(q)}$ is larger than the \emph{convex} combination of $\frac{1}{\varphi(q_l)}$. 
This is true when $\frac{1}{\varphi(q)}\geq\frac{1}{\varphi(q_l)}$ for all $l$, and at least one strict inequality holds for some $l \in \domy$.
Note that this is a sufficient condition \emph{but not a necessary condition} to prove that $\vecq^{\gamma,*}$ is $\veceta$UC.
Nevertheless, we will show that this condition holds and thus Cor.~\ref{cor:onehalfunder} can be proven based on this condition.
Thus, in the following proof, we focus on showing that when $q \in [\frac{1}{2},1)$ then for all $l$, we have
\begin{align}
    \varphi(q_l)\geq \varphi(q).\label{eq:varphi_ineq}
\end{align}
To show the above inequality, we split the proof into two cases:
\begin{itemize}
\item \textbf{Case 1}:  $\varphi(a) > \varphi(b)$ when $1 > b>a\geq\frac{1}{2}$, \ie, $\varphi$ is a decreasing function when its argument is greater or equal to $\frac{1}{2}$.
\item \textbf{Case 2}: $\varphi(a) > \varphi(\frac{1}{2})$ when $a \in [0, \frac{1}{2})$, \ie, when $a< \frac{1}{2}$, $\varphi(a)$ is lowerbounded by $\varphi(\frac{1}{2})$. 
\end{itemize}
If the above two cases hold, then Ineq.~\eqref{eq:varphi_ineq} holds. 
Next, we present the proof for each case.
 
\textbf{Case 1}: $\varphi(a) > \varphi(b)$ when $1 > b>a\geq\frac{1}{2}$.
 
To show $\varphi(a)$ is decreasing for $a \in [ \frac{1}{2},1)$, we will show that its derivative is smaller than zero. This can be seen by manipulating the derivative of $\varphi$ as follows:
\begin{align}
    \frac{d}{da}\varphi(a)
    &=-\gamma(1-a)^{\gamma-1}+\gamma(\gamma-1)(1-a)^{\gamma-2}a\log a-\gamma(1-a)^{\gamma-1}\log a-\gamma(1-a)^{\gamma-1}
    \\
    &=\gamma(1-a)^{\gamma-2}(-(1-a)+(\gamma-1)a\log a-(1-a)\log a-(1-a))
    \\
    &=\gamma(1-a)^{\gamma-2}((\gamma a-1)\log a-2+2a)
    \\
    &=\underbrace{\gamma(1-a)^{\gamma-2}}_{>0}(\underbrace{\gamma a\log a}_{<0} + \underbrace{(-\log a - 2 + 2a)}_{<0})\label{eq:prf_multi_case1}
    \\
    & <0.
\end{align}
In Eq.~\eqref{eq:prf_multi_case1}, we know $-\log a - 2 + 2a\leq0$ by noting its derivative is $2-\frac{1}{a}\geq0$ for $a \in [ \frac{1}{2},1)$, meaning it is an increasing function in $a \in [ \frac{1}{2},1)$. Therefore, its supremum must be at $1$ with the value $-\log 1 - 2 + 2(1)=0$. Since its supremum is zero, all other values in the range must be smaller than zero, making the term negative.
This suffices to show that $\varphi(a)$ is decreasing for $a \in [ \frac{1}{2},1)$, and concludes that \textbf{Case 1} holds.
 
\textbf{Case 2}: $\varphi(a) > \varphi(\frac{1}{2})$ when $a \in [0, \frac{1}{2})$.

Let us reexpress the above expression, 
\begin{align}
    \varphi(a) 
    &> \varphi(\frac{1}{2})
    \\
    (1-a)^\gamma-\gamma(1-a)^{\gamma-1}a\log a 
    & > \frac{1}{2^\gamma}-\gamma\frac{1}{2^\gamma}\log \frac{1}{2}
    \\
    (1-a)^\gamma(1-\gamma(1-a)^{-1}a\log a)
    & > \frac{1}{2^\gamma}\left(1-\gamma\log \frac{1}{2}\right)
    \\
    \frac{(1-a)^\gamma(1-\gamma(1-a)^{-1}a\log a)}{\frac{1}{2^\gamma}\left(1-\gamma\log \frac{1}{2}\right)} & > 1
    \\
    \underbrace{\frac{(1-a)^\gamma\left(\frac{1}{2}+a\right)^\gamma}{\frac{1}{2^\gamma}}}_{=:s_1}
    \underbrace{\frac{(1-\gamma(1-a)^{-1}a\log a)}{\left(\frac{1}{2}+a\right)^\gamma\left(1-\gamma\log \frac{1}{2}\right)}}_{=:s_2}
    & > 1,\label{eq:prf_multi_case2_split}
\end{align}
 where in the last inequality we multiply $\left(\frac{1}{2}+a\right)^\gamma$ to both numerator and denominator\footnote{We came up with $\left(\frac{1}{2}+a\right)^\gamma$ by trial and error.}. 
 Next, we show that both $s_1>1$ and $s_2>1$.
 
 For $s_1$, we can see that for $a\in(0, \frac{1}{2})$,
 \begin{align}
     s1 
     & = \frac{(1-a)^\gamma\left(\frac{1}{2}+a\right)^\gamma}{\frac{1}{2^\gamma}}
     \\
     & = \frac{(1-a)^\gamma(1+2a)^\gamma}{\underbrace{\frac{1}{2^\gamma}2^\gamma}_{=1}}
     \\
     & = (1-a)^\gamma(1+2a)^\gamma
     \\
     & = ((1-a)(1+2a))^\gamma
     \\
     & = (1-a+2a-2a^2)^\gamma
     \\
     & = (1+a-2a^2)^\gamma\label{eq:eq:prf_multi_case2_s1}
     \\
     & > 1.
 \end{align}
 The last inequality comes from the fact that $1+a-2a^2>0$, since it is quadratic with a negative coefficient on the $a^2$ term (\ie, an upside-down U-curve) and it takes value of $1$ at $a=0$ and $a=\frac{1}{2}$. Thus, this quadratic term is larger than $1$ in the range $a\in(0, \frac{1}{2})$. This shows that in Ineq.~\eqref{eq:prf_multi_case2_split} we have $s_1>1$ for $a\in(0, \frac{1}{2})$.
 
 For $s_2$, our goal is to show that for $a\in(0, \frac{1}{2})$, we have
 \begin{align}
     \frac{1-\gamma(1-a)^{-1} a\log a}{\left(\frac{1}{2}+a\right)^\gamma\left(1-\gamma\log \frac{1}{2}\right)} > 1.\label{eq:prf_multi_case2_s2}
 \end{align}
 This can be rearranged as
  \begin{align}
     1-\frac{\gamma a\log a}{(1-a)}
     &>  \left(\frac{1}{2}+a\right)^\gamma\left(1-\gamma\log \frac{1}{2}\right)
     \\
     1-\frac{\gamma a\log a}{(1-a)} - \left(\frac{1}{2}+a\right)^\gamma\left(1-\gamma\log \frac{1}{2}\right)
     &> 0
     \\
     \frac{1}{\left(\frac{1}{2}+a\right)^\gamma}-\frac{\gamma a\log a}{(1-a)\left(\frac{1}{2}+a\right)^\gamma} - \left(1-\gamma\log \frac{1}{2}\right)
     & > 0.\label{eq:prf_multi_case2_s1_g}
 \end{align}
 Let us denote the left-hand side of the above in equality as the following:
 \begin{align}
 g(\gamma)=\frac{1}{\left(\frac{1}{2}+a\right)^\gamma}-\frac{\gamma a\log a}{(1-a)\left(\frac{1}{2}+a\right)^\gamma} - \left(1-\gamma\log \frac{1}{2}\right).
 \end{align}
 Based on the above derivation, showing Ineq.~\eqref{eq:prf_multi_case2_s2} is equivalent to showing that $g(\gamma)>0$ for $\gamma>0$ and $0<a<\frac{1}{2}$. To show $g(\gamma)>0$ for $\gamma>0$, we observe the following properties of $g$:
 \begin{itemize}
    \item $g$ is convex in $\gamma$. This is because it is a sum of convex functions,~\ie, $\left(\frac{1}{2}+a\right)^{-\gamma}$ and  $-\frac{\gamma a\log a}{(1-a)\left(\frac{1}{2}+a\right)^\gamma}$ can be shown to be convex by noting their second order derivatives are positive for all $\gamma>0$, while $- (1-\gamma\log \frac{1}{2})$ is linear in $\gamma$.
     \item $g(0)=0$.
 \end{itemize}
With these properties, we only need to show that the derivative of $g$ at $0$ is nonnegative,~\ie, $\frac{dg}{d\gamma}(0)\geq0$, since convexity guarantees that the derivative of $g$ would only increase in $\gamma$ for $\gamma\geq 0$ and since $g(0)=0$ we will have $g(\gamma)\geq0$ for all $\gamma\geq0$. 
Thus, next, we show that
\begin{align}
0 
&\leq \frac{d}{d\gamma}g(0) \label{eq:prf_multi_case2_dg0}
\\
 &= \left[-\frac{\log(\frac{1}{2}+a)}{(\frac{1}{2}+a)^\gamma}-\frac{a\log a}{(1-a)(\frac{1}{2}+a)^\gamma}+\frac{\gamma a\log a \log(\frac{1}{2}+a)}{(1-a)(\frac{1}{2}+a)^\gamma}+\log\frac{1}{2}\right]_{\gamma=0}
 \\
 &=-\log\left(\frac{1}{2}+a\right)-\frac{a\log a}{(1-a)}+\log\frac{1}{2}
  \\
 &=-\log(1+2a)-\frac{a\log a}{(1-a)}.\label{eq:prf_multi_case2_dg}
\end{align}
Showing the above inequality is equivalent to showing
\begin{align}
    -(1-a)\log(1+2a)-a\log a \geq 0.
\end{align}
Let us denote the above expression as $\phi(a)=-(1-a)\log(1+2a)-a\log a$, where its derivative and second order derivative are
\begin{align}
    \frac{d}{da}\phi(a)
    &=\log(1+2a)-\frac{2(1-a)}{1+2a}-\log a-1
    \\
    &=\log(1+2a)-\frac{3}{1+2a}-\log a, \text{ and}\label{eq:prf_multi_case2_da1}
    \\
    \frac{d^2}{da^2}\phi(a)
    &= \frac{2}{1+2a}+\frac{6}{(1+2a)^2}-\frac{1}{a}
    \\
    &=\frac{4a-1}{a(1+2a)^2}.\label{eq:prf_multi_case2_da2}
\end{align}
From Eq.~\eqref{eq:prf_multi_case2_da2}, we can see that $\frac{d^2}{da^2}\phi(a) \geq 0$ for $a\in[\frac{1}{4},\frac{1}{2}]$ and $\frac{d^2}{da^2} \phi(a) \leq 0$ for  $a\in(0,\frac{1}{4}]$. This means $\frac{d\phi}{da}$ is increasing for $a\in[\frac{1}{4},\frac{1}{2}]$ and decreasing for $a\in(0,\frac{1}{4}]$. Also, we see that $\frac{d}{da}\phi(0)=\infty$, $\frac{d}{da}\phi(\frac{1}{4})=-0.2082$, and $\frac{d}{da}\phi(\frac{1}{2})=-0.1137$. By the intermediate value theorem, there exists an $a_0\in(0, \frac{1}{4})$ such that $\frac{d}{da}\phi(a_0)=0$, which means $\phi$ is increasing for $a\in (0,a_0)$ and decreasing for $a\in (a_0,\frac{1}{2})$. Therefore, the minimum of $\phi$ in $[0,\frac{1}{2}]$ can only be at either $0$ or $\frac{1}{2}$. 
By computing the values at both ends, we have
\begin{align}
    \min\left\{\phi(0), \phi\left(\frac{1}{2}\right)\right\}=\min\{0, 0\} = 0 \geq 0.
\end{align}
Thus, $\phi(a) \geq 0$ for $a\in[0,\frac{1}{2}]$. To wrap up, this result means Ineq.~\eqref{eq:prf_multi_case2_dg0},~\ie, $\frac{dg}{d\gamma}(0)\geq0$, holds. Together with $g$ being convex and $g(0)=0$, this means $g(a)\geq0$ for $a\in[0,\frac{1}{2}]$, making Ineq.~\eqref{eq:prf_multi_case2_s1_g} and therefore Ineq.~\eqref{eq:prf_multi_case2_s2}
 hold. This completes the proof for $s_2$.
 
 With the above lowerbounds of $s_1$ and $s_2$, we conclude that \textbf{Case 2} holds.

By combining \textbf{Case 1} and \textbf{Case 2}, we have proven Cor.~\ref{cor:onehalfunder}, which states that $\vecq^{\gamma,*}(\vecx)$ is $\eta$UC when $\frac{1}{2} \leq q < 1$ if $\vecq^{\gamma,*}(\vecx) \notin \subunif$.
 \end{proof}
 \subsection{Alternative proof of Cor.~\ref{cor:binaryunder}: $\veceta$UC property for $\vecq^{\gamma,*}$ in binary classification where $K=2$}
 \label{app:proof-alt-binaryunder}
Here, we present an alternative proof of Cor.~\ref{cor:binaryunder}, which is independent from Thm.~\ref{thm:recover}, Thm.~\ref{thm:underover}, Lem.~\ref{lem:varphiprop} and Lem.~\ref{lem:hstrictly}.
Therefore, this proof does not explicitly utilize the function $\varphi$.

\begin{proof}
Let $q=\max_y q_y^{\gamma,*}(\vecx)$ be the highest score in the simplex output from the focal risk minimizer $\vecq^{\gamma,*}$ and $\eta = \max_y \eta_y(\vecx)$ be the true class-posterior probability of the most probable class.  
Note that $q > \frac{1}{2}$ since $\vecq^{\gamma,*}(\vecx) \notin \subunif$ and there are only two classes. 
Also, in binary classification, it is restricted that the output of the simplex of the other class must be $1-q$ since $\vecq^{\gamma,*}(\vecx) \in \Delta^K$.

First, we relate $\eta$ to the focal risk minimizer $\vecq^{\gamma,*}(\vecx)$. 
Since we know that the focal loss is classification calibrated from our Thm.~\ref{thm:calib}, we have $\argmax_y q^{\gamma,*}_y(\vecx) = \argmax_y \eta_y(\vecx)$ .
As a result, the pointwise conditional risk w.r.t. the focal risk minimizer $\vecq^{\gamma,*}$ can be written as follows:
\begin{align}
    W^{\focalloss} \big(\vecq^{\gamma,*}(\vecx); \veceta(\vecx) \big) &= \sum_{y \in \domy} \eta_y(\vecx) \focalloss\big( \vecq(\vecx), \vece_y \big) \\
    &= \eta(-(1-q)^\gamma \log q) + (1-\eta)(-q^\gamma \log (1-q)).
\end{align}
For simplicity in the binary case, let us define
\begin{align}
    W^{\focalloss,*} \big(q,\eta) = \eta(-(1-q)^\gamma \log q) + (1-\eta)(-q^\gamma \log (1-q))
\end{align}
Since the focal loss is differentiable, $W^{\focalloss,*} \big(q,\eta)$ is also differentiable, and we can express $\frac{d}{dq} W^{\focalloss,*}$ as
\begin{align}
    \frac{d}{dq} W^{\focalloss,*} \big(q,\eta)&= \eta \left[(1-q)^{\gamma-1}\gamma \log q - \frac{(1-q)^\gamma}{q}\right]+(1-\eta) \left[ \frac{q^\gamma}{1-q}- q^{\gamma-1}\gamma \log (1-q)\right] \\ 
    &= \eta \left[ (1-q)^{\gamma-1}\gamma \log q - \frac{(1-q)^\gamma}{q} - \frac{q^\gamma}{1-q} + q^{\gamma-1}\gamma \log (1-q) \right] +\frac{q^\gamma}{1-q} - q^{\gamma-1}\gamma \log (1-q)
\end{align}
Since $\vecq^{\gamma,*}$ minimizes $W^{\focalloss}$ by the definition of the focal risk minimizer. 
Also, the binary pointwise conditional risk $W^{\focalloss,*} \big(q,\eta)$ is convex in $q$ given $\eta$.
We know that $\frac{d}{dq} W^{\focalloss,*} \big(q,\eta)=0$ and the following equation holds:
\begin{align}
  \eta \left[ (1-q)^{\gamma-1}\gamma \log q - \frac{(1-q)^\gamma}{q} - \frac{q^\gamma}{1-q} + q^{\gamma-1}\gamma \log (1-q) \right] +\frac{q^\gamma}{1-q} - q^{\gamma-1}\gamma \log (1-q) = 0 \\
   \eta \left[ (1-q)^{\gamma-1}\gamma \log q - \frac{(1-q)^\gamma}{q} - \frac{q^\gamma}{1-q} + q^{\gamma-1}\gamma \log (1-q) \right]  =   q^{\gamma-1}\gamma \log (1-q) -\frac{q^\gamma}{1-q} \\
   \eta   =   \frac{q^{\gamma-1}\gamma \log (1-q) -\frac{q^\gamma}{1-q}}{ (1-q)^{\gamma-1}\gamma \log q - \frac{(1-q)^\gamma}{q} - \frac{q^\gamma}{1-q} + q^{\gamma-1}\gamma \log (1-q) } \\
   \eta = \frac{ \frac{q^\gamma}{1-q} -q^{\gamma-1}\gamma \log (1-q)}{-(1-q)^{\gamma-1}\gamma \log q + \frac{(1-q)^\gamma}{q} + \frac{q^\gamma}{1-q} - q^{\gamma-1}\gamma \log (1-q) }.
\end{align}
Therefore, we can relate $\eta$ and $q$ as follows:
\begin{align}
\label{eq:recover-binary}
    \eta = \frac{\frac{q^\gamma}{1-q} - \gamma q^{\gamma-1}\log(1-q)}{\frac{q^\gamma}{1-q} - \gamma q^{\gamma-1}\log(1-q) + \frac{(1-q)^\gamma}{q} - \gamma (1-q)^{\gamma-1}\log q}.
\end{align}
As a result, the true class-posterior probability of the most probable class $\eta$ can be recovered from the highest score of the simplex output from the focal risk minimizer $q$ by Eq.~\eqref{eq:recover-binary}. 

Next, to prove that $\vecq^{\gamma^*}$ is $\veceta$UC, it suffices to prove that $\eta - q > 0$ for $q>\frac{1}{2}$, which means the minimizer of the focal loss underestimates the true class-posterior probability of the most probable class, \ie, $q < \eta$. 
Thus for $q > \frac{1}{2}$, we have
\begin{align}
\frac{\frac{q^\gamma}{1-q} - \gamma q^{\gamma-1}\log(1-q)}{\frac{q^\gamma}{1-q} - \gamma q^{\gamma-1}\log(1-q) + \frac{(1-q)^\gamma}{q} - \gamma (1-q)^{\gamma-1}\log q} - q &> 0 \\
  \frac{\frac{q^\gamma}{1-q} - \gamma q^{\gamma-1}\log(1-q) -
  q \left(\frac{q^\gamma}{1-q} - \gamma q^{\gamma-1}\log(1-q) + \frac{(1-q)^\gamma}{q} - \gamma (1-q)^{\gamma-1}\log q \right)
  }{\frac{q^\gamma}{1-q} - \gamma q^{\gamma-1}\log(1-q) + \frac{(1-q)^\gamma}{q} - \gamma (1-q)^{\gamma-1}\log q} &>0. \label{eq:deno-binary-ignore}
\end{align}

Since every term in the denominator of Ineq.~\eqref{eq:deno-binary-ignore} is positive, it can be ignored because it is sufficient to only check if the numerator is positive to prove that the fraction is positive. 
Therefore, we have
\begin{align}
\frac{q^\gamma}{1-q} - \gamma q^{\gamma-1}\log(1-q) -
   q \left(\frac{q^\gamma}{1-q} - \gamma q^{\gamma-1}\log(1-q) + \frac{(1-q)^\gamma}{q} - \gamma (1-q)^{\gamma-1}\log q \right)
    &>0 \\
    (1-q) \left[ \frac{q^\gamma}{1-q} - \gamma q^{\gamma-1}\log(1-q) \right] - q \left[ \frac{(1-q)^\gamma}{q} - \gamma (1-q)^{\gamma-1}\log q \right] &>0 \\ 
     (1-q) \left[ \frac{q^\gamma}{1-q} - \gamma q^{\gamma-1}\log(1-q) \right] -  \left[ (1-q)^\gamma - \gamma q (1-q)^{\gamma-1}\log q \right] &>0 \\ 
      q^\gamma - \gamma (1-q) q^{\gamma-1}\log(1-q)  -   (1-q)^\gamma + \gamma q (1-q)^{\gamma-1}\log q &>0 .  \label{eq:underbinary}
\end{align}
Next, we will show that Ineq.~\eqref{eq:underbinary} always holds when $q>\frac{1}{2}$ for all $\gamma>0$. 
We split the proof into two cases, the first case when $\gamma \geq 1$ and the second case when $0<\gamma<1$.

\noindent \textbf{Case 1} ($\gamma \geq 1$): We can also express Ineq.~\eqref{eq:underbinary} as

\begin{align}
q^\gamma - \gamma (1-q) q^{\gamma-1}\log(1-q)  -   (1-q)^\gamma + \gamma q (1-q)^{\gamma-1}\log q &>0 \\
q^\gamma - \gamma (1-q) q^{\gamma-1}\log(1-q)  &>(1-q)^\gamma - \gamma q (1-q)^{\gamma-1}\log q \\
\frac{q^{\gamma-1}}{(1-q)^{\gamma-1}}
\frac{\left[q-\gamma(1-q)\log(1-q)\right]}{\left[(1-q)-\gamma q\log q\right]} &> 1. 
\label{ineq:proof_case1_rewrite}
\end{align}
Because $q>\frac{1}{2}>1-q$ and $\gamma >= 1$, we have 
\begin{align}
\label{ineq:proof_case1_cond1}
    \frac{q^{\gamma-1}}{(1-q)^{\gamma-1}} \geq 1.
\end{align}   
Next, we will show that  the following inequality
\begin{align}
\label{ineq:proof_case1_cond2}
    \frac{q-\gamma(1-q)\log(1-q)}{(1-q)-\gamma q\log q} &> 1
\end{align}
can be rewritten as
\begin{align}
\frac{q-\gamma(1-q)\log(1-q)}{(1-q)-\gamma q\log q} &> 1 \\
q-\gamma(1-q)\log(1-q) &> (1-q)-\gamma q\log q \\
(q-\gamma(1-q)\log(1-q))-((1-q)-\gamma q\log q) &> 0\\
\underbrace{2q-1}_{> 0}+\gamma\underbrace{(-(1-q)\log(1-q)+ q\log q)}_{> 0} &> 0. 
\end{align} 
Therefore, to prove that Ineq.~\eqref{ineq:proof_case1_cond2} holds, it suffices to prove that $2q-1 > 0$ and $-(1-q)\log(1-q)+ q\log q > 0$.
Since $q > \frac{1}{2}$, it is straightforward to see that $2q-1 > 0$ for all $q > \frac{1}{2}$.
Next, we prove that $-(1-q)\log(1-q)+ q\log q > 0$:
\begin{align}
 -(1-q)\log(1-q)+ q\log q &> 0 \\
 -\log(1-q)+ \frac{1}{1-q} q\log q &> 0 \\
 -\frac{1}{q}\log(1-q)+ \frac{1}{1-q} \log q &> 0 \\
 -\frac{1}{q} \left[ - \sum_{i=1}^{\infty} \frac{q^i}{i} \right] + \frac{1}{1-q} \left[- \sum_{i=1}^{\infty} \frac{(1-q)^i}{i}   \right] &> 0 \\
\left[ \sum_{i=1}^{\infty} \frac{q^{i-1}}{i} \right] +  \left[- \sum_{i=1}^{\infty} \frac{(1-q)^{i-1}}{i}   \right] &> 0 \\
  \sum_{i=1}^{\infty} \frac{q^{i-1} - (1-q)^{i-1}}{i}   &> 0.\\
  \sum_{i=2}^{\infty} \frac{q^{i-1} - (1-q)^{i-1}}{i}   &> 0.
\end{align}
Note that we can write $\log (1-q) = \sum_{i=1}^{\infty} - \frac{q^i}{i}$ (Maclaurin series) and $q^{i-1} - (1-q)^{i-1} >0$ since $q>\frac{1}{2}$ for $i \geq 2$. 
Thus, by combining Ineqs.~\eqref{ineq:proof_case1_cond1} and~\eqref{ineq:proof_case1_cond2}, we prove that Ineq.~\eqref{eq:underbinary} holds for the case where $\gamma \geq 1$.
Therefore, we can conclude that if $\vecq^{\gamma,*}(\vecx) \notin \subunif$, then $\vecq^{\gamma,*}$ is $\veceta$UC for all $\gamma \geq 1$.

\noindent \textbf{Case 2} ($0<  \gamma < 1 $): First, we add $\gamma q - \gamma q + \gamma (1-q) - \gamma (1-q)$, which equals to zero, to Ineq.~\eqref{eq:underbinary}:
\begin{align}
\label{ineq:proof_case2_rewrite}
    q^\gamma - \gamma (1-q) q^{\gamma-1}\log(1-q)  -   (1-q)^\gamma + \gamma q (1-q)^{\gamma-1}\log q + \gamma q - \gamma q + \gamma (1-q) - \gamma (1-q)>0.
\end{align}

We then split Ineq.~\eqref{ineq:proof_case2_rewrite} into two parts, which are
\begin{align}
\label{ineq:proof_case2_cond1} q^\gamma  -   (1-q)^\gamma  - \gamma q + \gamma (1-q) >0,
\end{align}
and
\begin{align}
- \gamma (1-q) q^{\gamma-1}\log(1-q)   + \gamma q (1-q)^{\gamma-1}\log q + \gamma q  - \gamma (1-q)>0\\
-  (1-q) q^{\gamma-1}\log(1-q)   +  q (1-q)^{\gamma-1}\log q +  q  -  (1-q)>0.  
\label{ineq:proof_case2_cond2}  
\end{align}
We will prove \textbf{Case 2} by showing that both Ineqs.~\eqref{ineq:proof_case2_cond1} and~\eqref{ineq:proof_case2_cond2} hold simultaneously, which suggests that Ineq.~\eqref{eq:underbinary} must also hold.

Next, we will show that Ineq.~\eqref{ineq:proof_case2_cond1} holds. 
First, it can be observed that the left-hand side of Ineq.~\eqref{ineq:proof_case2_cond1} is zero when $q=\frac{1}{2}$.
Next, the derivative with respect to $q$ can be expressed as follows: 
\begin{align}
    \frac{d}{dq} \left[q^\gamma  -   (1-q)^\gamma  - \gamma q + \gamma (1-q)\right] = \gamma q^{\gamma-1} + \gamma (1-q)^{\gamma-1} - 2\gamma q,
\end{align}
which can be shown to always be positive for $0 < \gamma < 1$: 
\begin{align}
\gamma q^{\gamma-1} + \gamma (1-q)^{\gamma-1} - 2\gamma q > 0 \\
 q^{\gamma-1} +  (1-q)^{\gamma-1} - 2 q > 0\\
 \frac{1}{q^{1-\gamma}} - q +\frac{1}{(1-q)^{1-\gamma}} -q> 0 \label{eq:case2-binary}
\end{align}
because  $\frac{1}{q^{1-\gamma}} >1 $ and $\frac{1}{(1-q)^{1-\gamma}} > 1$ for $0<\gamma<1$. 
Therefore, $\frac{1}{q^{1-\gamma}}+\frac{1}{(1-q)^{1-\gamma}}>2 > 2q$, which indicates that Ineq.~\eqref{eq:case2-binary} holds.
By combining the fact that the left-hand side of Ineq.~\eqref{ineq:proof_case2_cond1} is zero when $q=\frac{1}{2}$, and the left-hand side of Ineq.~\eqref{ineq:proof_case2_cond1} is increasing as $q$ increases.
Ineq.~\eqref{ineq:proof_case2_cond1} must be more than zero for all $q>\frac{1}{2}$.
Thus, we can conclude that  Ineq.~\eqref{ineq:proof_case2_cond1} holds for all $q > \frac{1}{2}$.

Next, we will show that Ineq.~\eqref{ineq:proof_case2_cond2} holds by showing that its left-hand side of Ineq.~\eqref{ineq:proof_case2_cond2} is an increasing function w.r.t.~$\gamma$. 

\noindent The following derivative 
\begin{align}
    \frac{d}{d \gamma}\left[  -  (1-q) q^{\gamma-1}\log(1-q)   +  q (1-q)^{\gamma-1}\log q +  q  -  (1-q)\right]
\end{align}
is equal to
\begin{align}
   -(1-q)q^{\gamma -1} \log q \log (1-q) + q(1-q)^{\gamma-1} \log(1-q) \log q.
\end{align}
We can see that
\begin{align}
-(1-q)q^{\gamma -1} \log q \log (1-q) + q(1-q)^{\gamma-1} \log(1-q) \log q &> 0 \\
-(1-q)q^{\gamma -1}+ q(1-q)^{\gamma-1} &> 0 .  \label{ineq:proof_case2_derivative_inequality}
\end{align}
Ineq.~\eqref{ineq:proof_case2_derivative_inequality} holds because $q>1-q$ and $(1-q)^{\gamma-1} > q^{\gamma-1}$ for $0<\gamma < 1$. 
This suggests that the left-hand side of Ineq.~\eqref{ineq:proof_case1_cond2} is an increasing function w.r.t. $\gamma$. 
Next, we show that Ineq.~\eqref{ineq:proof_case2_rewrite} holds for $\gamma = 0$. By substitute $\gamma=0$ to Ineq.~\eqref{ineq:proof_case2_cond2}, we have
\begin{align}
-   \frac{(1-q)\log(1-q)}{q}   +   \frac{q\log q}{(1-q)} +  q  -  (1-q)>0 \\
-   (1-q) \left(\frac{\log(1-q)}{q} +1\right)   +  q\left(\frac{\log q}{(1-q)} +  1 \right) >0.
\end{align}
By using the fact that $\log(q) = -\sum_{i=1}^\infty \frac{(1-q)^i}{i}$, we have
\begin{align}
    -   (1-q) \left(\frac{-\sum_{i=1}^\infty \frac{q^i}{i}}{q} +1\right)   +  q\left(\frac{-\sum_{i=1}^\infty \frac{(1-q)^i}{i}}{(1-q)} +  1 \right) >0.
\end{align}
Dividing by $q(1-q)$ gives
\begin{align}
-   \frac{1}{q} \left(\frac{-\sum_{i=1}^\infty \frac{q^i}{i}}{q} +1\right)   +  \frac{1}{1-q} \left(\frac{-\sum_{i=1}^\infty \frac{(1-q)^i}{i}}{(1-q)} +  1 \right) &>0 \\
-\frac{1}{q^2} \left(- \left(\sum_{i=1}^\infty \frac{q^i}{i}\right) +q\right)   +  \frac{1}{(1-q)^2} \left(\left(-\sum_{i=1}^\infty \frac{(1-q)^i}{i}\right) +  (1-q) \right) &>0 \\
-\frac{1}{q^2} \left(-\sum_{i=2}^\infty \frac{q^i}{i}\right)   +  \frac{1}{(1-q)^2} \left(-\sum_{i=2}^\infty \frac{(1-q)^i}{i}  \right) &>0 \\
\left(\sum_{i=2}^\infty \frac{q^{i-2}}{i}\right)   +   \left(-\sum_{i=2}^\infty \frac{(1-q)^{i-2}}{i}  \right) &>0 \\
\left(\sum_{i=2}^\infty \frac{q^{i-2}- (1-q)^{i-2}}{i}  \right) &>0. 
\label{ineq:case2_cond2_final}
\end{align}
Since $q>1-q$ and $i>=2$, Ineq.~\eqref{ineq:case2_cond2_final} holds, which indicates that Ineq.~\eqref{ineq:proof_case2_cond2} holds for $\gamma=0$. 
And we know that the left-hand size of Ineq.~\eqref{ineq:proof_case2_cond2} is an increasing function as $\gamma$ increases for $0<\gamma <1$ from Ineq.~\eqref{ineq:proof_case2_derivative_inequality}. 
As a result, Ineq.~\eqref{ineq:case2_cond2_final} holds for $0<\gamma < 1$.
Therefore, we can conclude that if $\vecq^{\gamma,*}(\vecx) \notin \subunif$, then $\vecq^{\gamma,*}(\vecx)$ is $\veceta$UC for all $0<\gamma<1$.   

By combining \textbf{Case 1} and \textbf{Case 2}, we have proven Cor.~\ref{cor:binaryunder}, which states that $\vecq^{\gamma,*}(\vecx)$ is $\eta$UC unless it is uniform or a one-hot vector.
\end{proof}
\newpage
\section{Evaluation metrics}
\label{app:metrics}
In practice, hard labels are given.
As a result, it is not straightforward to measure the quality of class-posterior probability estimation. 
In this section, we review evaluation metrics that are used in this paper to evaluate the quality of prediction confidence given hard labels.
\subsection{Expected calibration error (ECE)}
The expected calibration error can be defined as follows~\cite{naeini2015obtaining,guo2017calibration}:
\begin{align}
\label{eq:ece}
    \text{ECE}=\frac{1}{n}\sum_{j=1}^{N_B} |B_j|\left\vert\text{acc}(B_j)-\text{conf}(B_j)\right\vert,
\end{align}
where $|B_j|$ is the number of samples in bin $j$, $\text{acc}(B_j)$ and $\text{conf}(B_j)$ are the average accuracy and the average confidence of data in bin~$j$, and $N_B$ is the number of bins.
In this paper, we used $N_B=10$.
To allocate the data in the different bins, we use the following procedure.
First, we rank the test data by the prediction confidence of our classifier. 
Then we put the test data in the bin based on their confidence scores.
More precisely, the first bin is for data with prediction confidence in the range of $0-0.1$, the second one is in the range of $0.1-0.2$, and the tenth one is in the range of $0.9-1.0$.  
Then, we can calculate ECE by calculating the average accuracy and the average confidence of data in each bin and then combine them based on Eq.~\eqref{eq:ece}.
\subsection{Negative log-likelihood (NLL)}
Given data with hard labels $\mathcal{D} = \{(\vecx_i, y_i) \}_{i=1}^n$ and a classifier $\vecq$, one can calculate NLL as follows:
\begin{equation}
    \text{NLL} = - \sum_{i=1}^{n} \log(q_{y_i}(\vecx_i)).
\end{equation}
It is simple since it can be calculated in a pointwise manner. 
\subsection{Classwise expected calibration error (CW-ECE)}
CW-ECE is defined as follows~\cite{kull2019beyond}:
\begin{align}
    \text{CW-ECE} = \frac{1}{K}\sum_{l=1}^K \sum_{j=1}^{N_B}\frac{|B_{j,l}|}{n} \left\vert \text{prop}_l(B_{j,l})-\text{conf}_l(B_{j,l}) \right\vert,
\end{align}
where $|B_{j,l}|$ denotes the number of samples in bin $(j,l)$, $\text{prop}_l(B_{j,l})$ denotes the proportion of class $l$ in bin $B_{j,l}$, and $\text{conf}_l(B_{j,l})$ denotes the average confidence of predicting class $l$ in bin $B_{j,l}$. 
Unlike ECE, we rank the confidence score based on each class, not the maximum confidence score.
Moreover, CW-ECE does not use the average accuracy but simply the proportion of the class in each bin to compare with the average confidence of the class in that bin.
The original motivation of CW-ECE is to mitigate the problem of ECE that it only focuses on the class with the highest confidence. 
\newpage
\section{Additional experimental results}
\label{app:exp}
In this section, we provide additional experimental results.
We report the classification error of all ResNets we used on SVHN, CIFAR10, CIFAR10-s, and CIFAR100. 
Then, we report ECE and accuracy on additional datasets to support our claim in Sec.~\ref{sec:exp-dis} using a linear-in-input model and neural networks with one hidden layer.
Next, we report reliability diagrams for all ResNet models we used for SVHN, CIFAR10, CIFAR10-s, and CIFAR100. 
Finally, we report the performance of all ResNet models we used for SVHN, CIFAR10, CIFAR10-s, and CIFAR100 using three evaluation metrics, which are the expected calibration error (ECE), negative log-likelihood (NLL), and classwise expected calibration error (CW-ECE).
\subsection{Classification error of different datasets and models}
Table~\ref{app:tab-err} shows the classification error of all ResNet models we used with different training paradigms on SVHN, CIFAR10, CIFAR10-s, and CIFAR100.
Since the purpose of showing this table is for reference, we do not bold any numbers.
\begin{table}
\caption{Mean and standard error of the classification error under different datasets, training paradigms, models, and losses.
The purpose of showing this table is for reference.}
\label{app:tab-err}
\centering
\vspace{0.5em}
\begin{tabular}{lllcccc}
\toprule
Dataset &Training paradigm& Model & CE & FL-$1$ & FL-$2$& FL-$3$ \\
\midrule
\multirow{8}{*}{SVHN}&\multirow{4}{*}{\textbf{Standard}/\textbf{TS}}
& ResNet8
& $3.83(0.07)$ & $3.98(0.08)$ & $4.06(0.15)$ 
& $4.19(0.11)$ \\
&& ResNet20
& $2.01(0.07)$ & $2.05(0.05)$ & $2.09(0.05)$ 
& $2.16(0.05)$ \\
&& ResNet44
& $1.92(0.07)$ & $1.92(0.08)$ & $2.02(0.05)$ 
& $2.05(0.06)$ \\
&& ResNet110
& $1.90(0.09)$ & $1.99(0.08)$ & $2.00(0.08)$ 
& $2.07(0.07)$ \\
\cmidrule{2-7}
&\multirow{4}{*}{\textbf{LS}}
& ResNet8
& $3.83(0.10)$ & $3.90(0.14)$ & $3.90(0.06)$ 
& $3.94(0.07)$ \\
&& ResNet20
& $1.98(0.05)$ & $2.04(0.06)$ & $2.02(0.07)$ 
& $2.05(0.06)$ \\
&& ResNet44
& $1.87(0.04)$ & $1.88(0.06)$ & $1.90(0.07)$ 
& $1.94(0.08)$ \\
&& ResNet110
& $1.89(0.11)$ & $1.89(0.05)$ & $1.90(0.03)$ 
& $1.90(0.09)$ \\
\hline
\multirow{8}{*}{CIFAR10}&\multirow{4}{*}{\textbf{Standard}/\textbf{TS}}
& ResNet8
& $12.58(0.39)$ & $12.83(0.47)$ & $13.28(0.45)$ 
& $13.95(0.59)$ \\
&& ResNet20
& $7.42(0.40)$ & $7.53(0.31)$ & $7.74(0.42)$ 
& $8.01(0.29)$ \\
&& ResNet44
& $6.21(0.44)$ & $6.17(0.36)$ & $6.47(0.31)$ 
& $6.78(0.37)$ \\
&& ResNet110
& $5.68(0.32)$ & $5.87(0.36)$ & $5.82(0.33)$ 
& $6.42(0.61)$ \\
\cmidrule{2-7}
&\multirow{4}{*}{\textbf{LS}}
& ResNet8
& $12.58(0.30)$ & $12.51(0.33)$ & $13.06(0.37)$ 
& $13.12(0.25)$ \\
&& ResNet20
& $7.53(0.32)$ & $7.53(0.42)$ & $7.46(0.38)$ 
& $7.69(0.32)$ \\
&& ResNet44
& $6.41(0.29)$ & $6.35(0.28)$ & $6.49(0.36)$ 
& $6.67(0.31)$ \\
&& ResNet110
& $5.82(0.33)$ & $5.86(0.39)$ & $5.75(0.33)$ 
& $5.96(0.38)$ \\
\hline
\multirow{8}{*}{CIFAR10-s}&\multirow{4}{*}{\textbf{Standard}/\textbf{TS}}
& ResNet8
& $25.56(0.63)$ & $26.10(0.89)$ & $26.45(0.80)$ 
& $27.36(0.93)$ \\
&& ResNet20
& $21.44(0.51)$ & $21.98(0.66)$ & $22.47(0.52)$ 
& $23.01(0.48)$ \\
&& ResNet44
& $20.97(0.52)$ & $21.60(0.39)$ & $22.11(0.61)$ 
& $23.44(1.13)$ \\
&& ResNet110
& $24.91(1.54)$ & $27.08(2.28)$ & $27.78(2.36)$ 
& $28.12(1.76)$ \\
\cmidrule{2-7}
&\multirow{4}{*}{\textbf{LS}}
& ResNet8
& $24.98(0.82)$ & $25.48(0.78)$ & $26.16(0.54)$ 
& $26.46(0.90)$ \\
&& ResNet20
& $21.20(0.43)$ & $21.37(0.58)$ & $21.80(0.47)$ 
& $22.15(0.54)$ \\
&& ResNet44
& $21.09(0.37)$ & $21.28(0.39)$ & $21.95(0.71)$ 
& $22.48(0.68)$ \\
&& ResNet110
& $27.88(2.65)$ & $29.25(3.23)$ & $27.92(1.51)$ 
& $28.45(2.07)$ \\
\hline
\multirow{8}{*}{CIFAR100}&\multirow{4}{*}{\textbf{Standard}/\textbf{TS}}
& ResNet8
& $41.05(0.53)$ & $41.40(0.50)$ & $41.55(0.52)$ 
& $42.07(0.33)$ \\
&& ResNet20
& $31.71(0.46)$ & $32.38(0.40)$ & $32.96(0.34)$ 
& $33.40(0.45)$ \\
&& ResNet44
& $28.88(0.42)$ & $29.47(0.41)$ & $29.30(0.37)$ 
& $29.93(0.45)$ \\
&& ResNet110
& $25.10(0.68)$ & $25.30(0.91)$ & $25.58(0.81)$ 
& $26.22(0.53)$ \\
\cmidrule{2-7}
&\multirow{4}{*}{\textbf{LS}}
& ResNet8
& $41.42(0.42)$ & $41.44(0.31)$ & $41.72(0.44)$ 
& $42.15(0.38)$ \\
&& ResNet20
& $31.99(0.20)$ & $32.25(0.37)$ & $32.70(0.40)$ 
& $32.94(0.26)$ \\
&& ResNet44
& $28.89(0.38)$ & $28.84(0.50)$ & $29.17(0.43)$ 
& $29.25(0.47)$ \\
&& ResNet110
& $25.12(0.70)$ & $24.84(0.42)$ & $25.15(0.37)$ 
& $24.90(0.50)$ \\
\bottomrule
\end{tabular}
\end{table}
\FloatBarrier

\subsection{Experiments on additional datasets}
Here, we conducted experiments on additional $30$ datasets on models that arguably to be simpler than ResNet~\cite{he2016deep}, which are a linear-in-input model and a neural network with one hidden layer.

\textbf{Datasets:}
We used datasets from the UCI Data Repository~\cite{dataset3}.
We also used MNIST~\cite{lecun1998mnist}, Kuzushiji-MNIST (KMNIST)~\cite{clanuwat2018deep}, and Fashion-MNIST~\cite{xiao2017fashion}.

\textbf{Models:} 
We used two models in our experiments on additional datasets, which are a linear-in-input model and a neural network with one hidden layer.
Note that we do not compare the performance across the model but to validate the effectiveness of using the transformation $\Psibf$ for all settings.

\textbf{Methods:} 
We compared the models that use $\Psibf$ after the softmax layer to those that do not.
Note that both methods have the same accuracy since $\Psibf$ does not affect the decision rule (Prop.~\ref{prop:maintain}). 
We used the focal loss with $\gamma\in \{0, 1, 2, 3\}$ in this experiment, and conducted 10 trials for each experiment setting. 

\textbf{Evaluation metrics:} 
Since true class-posterior probability labels are not available, a common practice is to use ECE to evaluate the quality of prediction confidence~\cite{naeini2015obtaining,guo2017calibration}.  
We used $10$ as the number of bins.
CE denotes the method that uses the cross entropy loss and FL-$\gamma$ denotes the method that uses the focal loss with $\gamma$.
We denote FL-$\gamma$-$\Psibf$ for a method that applies $\Psibf$ to the output of the softmax layer before evaluating the confidence score when training with the focal loss.
Note that CE is equivalent to FL-$0$ and thus $\Psibf$ does not change the output of the trained classifier.

\textbf{Hyperparameters:} 
For both linear-in-input model and a neural network, the number of epochs was $50$ for all datasets and the batch size was $64$.
We used SGD with momentum of $0.9$, where the learning rate was $0.01$.
The weight decay parameter was $10^{-3}$.
For a neural network model, the number of nodes in a hidden layer was $64$. 

\textbf{Discussion:}
Tables~\ref{tab:linear-ece} and ~\ref{tab:mlp-ece} show ECE on all datasets using the same model for different losses.
It can be observed that $\Psibf$ can effectively improve the performance of the classifier trained with the focal loss in most cases although not all the cases. 
We can also see that as the $\gamma$ increases, ECE also increases for most datasets as well.
Tables~\ref{tab:linear-err} and ~\ref{tab:mlp-err} show the classification error for all methods for reference.  

\begin{table*}[t]
\centering
\caption{Mean and standard error of ECE over ten trials (rescaled to $0-100$).
We used \textbf{a linear-in-input model}.
Outperforming methods are highlighted in boldface using one-sided t-test with the significance level $5\%$.
}
\begin{tabular}{lccccccc}
\toprule
Dataset
& CE & FL-$1$ & FL-$2$& FL-$3$& FL-$1$-$\Psibf$ & FL-$2$-$\Psibf$ & FL-$3$-$\Psibf$ \\
\midrule
Australian
& $\mathbf{4.10(1.40)}$ & $8.83(2.37)$ & $15.47(1.46)$
& $19.46(1.85)$ & $\mathbf{4.45(1.84)}$ & $\mathbf{4.55(0.84)}$ & $\mathbf{4.67(1.42)}$ \\
Phishing
& $\mathbf{0.89(0.16)}$ & $6.48(0.53)$ & $12.05(0.42)$
& $16.74(0.67)$ & $\mathbf{0.83(0.23)}$ & $\mathbf{0.96(0.16)}$ & $\mathbf{0.94(0.38)}$ \\
Spambase
& $3.11(0.40)$ & $7.28(0.58)$ & $13.07(0.97)$
& $16.67(0.64)$ & $\mathbf{2.45(0.29)}$ & $\mathbf{2.76(0.52)}$ & $\mathbf{2.65(0.69)}$ \\
Waveform
& $\mathbf{1.42(0.19)}$ & $5.83(0.86)$ & $10.33(0.90)$
& $14.36(0.94)$ & $\mathbf{1.61(0.53)}$ & $1.72(0.36)$ & $1.75(0.20)$ \\
Twonorm
& $\mathbf{0.48(0.11)}$ & $2.51(0.35)$ & $6.13(0.61)$
& $10.05(0.55)$ & $0.56(0.12)$ & $\mathbf{0.52(0.09)}$ & $\mathbf{0.43(0.16)}$ \\
Adult
& $\mathbf{1.21(0.22)}$ & $6.12(0.38)$ & $11.76(0.42)$
& $15.82(0.34)$ & $1.64(0.41)$ & $1.94(0.42)$ & $2.16(0.41)$ \\
Banknote
& $4.37(0.31)$ & $9.93(0.56)$ & $15.69(0.54)$
& $21.07(0.69)$ & $3.43(0.24)$ & $\mathbf{2.99(0.31)}$ & $\mathbf{2.92(0.30)}$ \\
Phoneme
& $\mathbf{4.37(0.62)}$ & $7.25(0.51)$ & $11.92(0.67)$
& $14.60(0.58)$ & $\mathbf{4.29(0.71)}$ & $\mathbf{4.52(0.33)}$ & $\mathbf{4.66(0.44)}$ \\
Magic
& $\mathbf{1.56(0.35)}$ & $9.30(0.58)$ & $14.71(0.59)$
& $18.29(0.16)$ & $\mathbf{1.53(0.22)}$ & $\mathbf{1.64(0.49)}$ & $\mathbf{1.37(0.31)}$ \\
Gisette
& $1.52(0.18)$ & $\mathbf{1.41(0.23)}$ & $\mathbf{1.47(0.20)}$
& $\mathbf{1.50(0.24)}$ & $2.18(0.23)$ & $2.75(0.21)$ & $3.07(0.24)$ \\
USPS
& $\mathbf{0.37(0.11)}$ & $1.82(0.23)$ & $4.80(0.38)$
& $8.19(0.56)$ & $\mathbf{0.36(0.10)}$ & $\mathbf{0.32(0.09)}$ & $\mathbf{0.35(0.12)}$ \\
Splice
& $\mathbf{2.98(0.97)}$ & $6.37(0.87)$ & $13.41(0.97)$
& $16.98(1.69)$ & $\mathbf{2.84(0.76)}$ & $\mathbf{2.31(0.66)}$ & $3.43(1.10)$ \\
Banana
& $\mathbf{6.90(2.41)}$ & $8.94(3.09)$ & $\mathbf{6.85(2.69)}$
& $\mathbf{6.26(1.97)}$ & $8.07(3.86)$ & $\mathbf{7.11(3.73)}$ & $\mathbf{5.09(2.49)}$ \\
Ringnorm
& $3.43(0.49)$ & $9.86(0.97)$ & $14.15(0.87)$
& $17.04(0.50)$ & $\mathbf{3.15(0.72)}$ & $\mathbf{3.08(0.46)}$ & $\mathbf{3.00(0.49)}$ \\
Image
& $5.45(0.77)$ & $10.64(1.47)$ & $15.05(0.72)$
& $19.27(1.24)$ & $5.35(0.77)$ & $\mathbf{4.50(0.63)}$ & $\mathbf{4.63(0.94)}$ \\
Coil20
& $1.39(0.26)$ & $5.24(0.59)$ & $9.47(0.72)$
& $14.53(1.34)$ & $1.20(0.24)$ & $\mathbf{0.94(0.49)}$ & $\mathbf{0.87(0.33)}$ \\
Basehock
& $\mathbf{1.16(0.30)}$ & $2.32(0.40)$ & $3.76(0.67)$
& $5.82(0.69)$ & $\mathbf{1.39(0.47)}$ & $2.39(0.54)$ & $2.48(0.53)$ \\
Isolet
& $2.80(0.22)$ & $7.11(0.35)$ & $11.93(0.46)$
& $16.92(0.54)$ & $2.03(0.29)$ & $\mathbf{1.61(0.16)}$ & $\mathbf{1.51(0.29)}$ \\
W8a
& $\mathbf{0.36(0.11)}$ & $2.45(0.32)$ & $5.73(0.44)$
& $8.69(0.49)$ & $0.49(0.16)$ & $0.55(0.15)$ & $0.70(0.15)$ \\
Mushroom
& $0.09(0.04)$ & $1.06(0.05)$ & $3.18(0.14)$
& $5.96(0.26)$ & $0.06(0.01)$ & $\mathbf{0.04(0.02)}$ & $\mathbf{0.04(0.01)}$ \\
Artificial-character
& $\mathbf{3.52(0.39)}$ & $3.99(0.51)$ & $4.70(0.47)$
& $5.95(0.64)$ & $\mathbf{3.56(0.74)}$ & $\mathbf{3.27(0.53)}$ & $4.81(0.72)$ \\
Gas-drift
& $5.65(0.55)$ & $9.05(0.57)$ & $12.65(0.53)$
& $15.98(0.38)$ & $4.54(0.50)$ & $\mathbf{4.23(0.51)}$ & $\mathbf{4.12(0.49)}$ \\
Japanesevowels
& $4.28(0.14)$ & $8.88(0.64)$ & $12.88(0.30)$
& $16.57(0.26)$ & $2.89(0.51)$ & $1.90(0.18)$ & $\mathbf{1.35(0.20)}$ \\
Letter
& $10.19(0.33)$ & $14.82(0.42)$ & $18.26(0.24)$
& $20.88(0.60)$ & $7.95(0.43)$ & $6.13(0.27)$ & $\mathbf{5.19(0.43)}$ \\
Pendigits
& $4.72(0.33)$ & $11.14(0.44)$ & $16.29(0.33)$
& $20.61(0.55)$ & $3.71(0.38)$ & $2.68(0.24)$ & $\mathbf{1.84(0.36)}$ \\
Satimage
& $\mathbf{2.16(0.36)}$ & $6.39(0.37)$ & $10.19(0.92)$
& $13.65(0.50)$ & $\mathbf{1.92(0.29)}$ & $\mathbf{2.19(0.69)}$ & $2.64(0.36)$ \\
Vehicle
& $11.27(1.72)$ & $16.79(3.20)$ & $19.05(2.31)$
& $23.62(1.11)$ & $10.39(2.16)$ & $\mathbf{6.95(1.80)}$ & $\mathbf{8.00(1.53)}$ \\
MNIST
& $1.27(0.09)$ & $5.82(0.14)$ & $10.27(0.16)$
& $14.24(0.23)$ & $\mathbf{0.93(0.15)}$ & $1.09(0.15)$ & $1.43(0.14)$ \\
KMNIST
& $5.51(0.22)$ & $\mathbf{1.91(0.25)}$ & $5.41(0.35)$
& $9.17(0.40)$ & $6.45(0.48)$ & $7.88(0.39)$ & $8.90(0.44)$ \\
Fashion-MNIST
& $\mathbf{1.60(0.31)}$ & $3.38(0.53)$ & $7.46(0.42)$
& $10.78(0.95)$ & $2.89(0.42)$ & $4.07(0.36)$ & $4.85(0.92)$ \\
\bottomrule
\end{tabular}
\label{tab:linear-ece}
\vspace{-3mm}
\end{table*}

\begin{table*}[t]
\centering
\caption{Mean and standard error of the classification error over ten trials (rescaled to $0-100$).
We used \textbf{a linear-in-input model}.
The purpose of showing this table is for reference.
}
\begin{tabular}{lcccc}
\toprule
Dataset
& CE & FL-$1$ & FL-$2$& FL-$3$ \\
\midrule
Australian
& $14.03(1.47)$ & $13.54(1.27)$ & $13.36(1.33)$
& $13.65(1.34)$ \\
Phishing
& $7.24(0.20)$ & $7.22(0.30)$ & $7.47(0.20)$
& $7.41(0.24)$ \\
Spambase
& $7.86(0.45)$ & $8.21(0.34)$ & $7.96(0.62)$
& $8.52(0.48)$ \\
Waveform
& $12.00(0.37)$ & $12.24(0.44)$ & $12.39(0.61)$
& $12.55(0.84)$ \\
Twonorm
& $2.19(0.15)$ & $2.28(0.16)$ & $2.22(0.20)$
& $2.31(0.20)$ \\
Adult
& $15.46(0.17)$ & $15.75(0.20)$ & $15.71(0.26)$
& $16.00(0.26)$ \\
Banknote-authentication
& $2.33(0.39)$ & $1.90(0.64)$ & $2.24(0.21)$
& $2.13(0.43)$ \\
Phoneme
& $24.89(0.48)$ & $24.98(0.68)$ & $25.14(0.38)$
& $25.22(0.47)$ \\
Magic
& $20.91(0.17)$ & $20.98(0.37)$ & $21.16(0.35)$
& $20.90(0.19)$ \\
Gisette
& $2.62(0.25)$ & $2.98(0.26)$ & $3.37(0.24)$
& $3.53(0.24)$ \\
USPS
& $1.26(0.15)$ & $1.32(0.09)$ & $1.22(0.08)$
& $1.22(0.11)$ \\
Splice
& $16.04(0.92)$ & $16.28(0.60)$ & $15.65(0.59)$
& $16.20(1.06)$ \\
Banana
& $43.91(3.69)$ & $46.16(3.22)$ & $44.34(1.44)$
& $45.23(2.21)$ \\
Ringnorm
& $23.85(0.44)$ & $23.95(0.65)$ & $24.56(0.71)$
& $24.26(0.39)$ \\
Image
& $17.81(1.13)$ & $18.03(0.98)$ & $19.11(0.52)$
& $18.85(1.14)$ \\
Coil20
& $1.53(0.68)$ & $1.17(0.78)$ & $1.61(0.71)$
& $1.56(0.65)$ \\
Basehock
& $3.38(0.54)$ & $3.85(0.44)$ & $4.77(0.66)$
& $4.65(0.82)$ \\
Isolet
& $1.28(0.34)$ & $1.09(0.42)$ & $1.24(0.48)$
& $1.18(0.36)$ \\
W8a
& $1.38(0.06)$ & $1.47(0.09)$ & $1.49(0.08)$
& $1.64(0.06)$ \\
Mushroom
& $0.03(0.05)$ & $0.01(0.01)$ & $0.01(0.04)$
& $0.01(0.03)$ \\
Artificial-character
& $64.55(0.44)$ & $64.97(0.51)$ & $64.71(0.47)$
& $65.51(0.61)$ \\
Gas-drift
& $2.09(0.36)$ & $2.05(0.42)$ & $1.98(0.31)$
& $2.06(0.42)$ \\
Japanesevowels
& $7.01(0.13)$ & $6.71(0.46)$ & $6.64(0.24)$
& $6.80(0.27)$ \\
Letter
& $23.56(0.32)$ & $23.49(0.30)$ & $24.03(0.23)$
& $24.69(0.46)$ \\
Pendigits
& $6.74(0.21)$ & $6.25(0.23)$ & $6.37(0.19)$
& $6.66(0.31)$ \\
Satimage
& $15.38(0.40)$ & $14.63(0.31)$ & $14.93(0.50)$
& $14.91(0.35)$ \\
Vehicle
& $26.12(1.33)$ & $25.34(2.33)$ & $26.78(1.91)$
& $25.79(0.99)$ \\
MNIST
& $7.42(0.08)$ & $7.53(0.07)$ & $7.73(0.09)$
& $7.99(0.14)$ \\
KMNIST
& $30.23(0.21)$ & $30.20(0.40)$ & $30.64(0.29)$
& $31.06(0.41)$ \\
Fashion-MNIST
& $15.83(0.28)$ & $16.21(0.40)$ & $16.82(0.44)$
& $16.90(0.83)$ \\
\bottomrule
\end{tabular}
\label{tab:linear-err}
\vspace{-3mm}
\end{table*}

\begin{table*}[t]
\centering
\caption{Mean and standard error of ECE over ten trials (rescaled to $0-100$).
We used \textbf{neural networks with one hidden layer} as a model.
Outperforming methods are highlighted in boldface using one-sided t-test with the significance level $5\%$.
}
\begin{tabular}{lccccccc}
\toprule
Dataset
& CE & FL-$1$ & FL-$2$& FL-$3$& FL-$1$-$\Psibf$ & FL-$2$-$\Psibf$ & FL-$3$-$\Psibf$ \\
\midrule
Australian
& $\mathbf{4.31(1.21)}$ & $8.15(2.10)$ & $12.74(1.53)$
& $18.08(1.96)$ & $\mathbf{4.62(1.03)}$ & $\mathbf{4.48(1.15)}$ & $\mathbf{4.83(1.22)}$ \\
Phishing
& $\mathbf{0.80(0.17)}$ & $3.14(0.33)$ & $7.28(0.41)$
& $12.16(0.40)$ & $\mathbf{0.84(0.29)}$ & $\mathbf{0.82(0.23)}$ & $1.03(0.26)$ \\
Spambase
& $\mathbf{1.81(0.43)}$ & $3.92(0.62)$ & $9.61(0.51)$
& $14.78(0.66)$ & $\mathbf{1.79(0.51)}$ & $\mathbf{1.52(0.25)}$ & $\mathbf{1.52(0.29)}$ \\
Waveform
& $\mathbf{2.03(0.47)}$ & $3.76(0.36)$ & $7.95(0.84)$
& $11.84(0.58)$ & $\mathbf{1.79(0.38)}$ & $\mathbf{1.74(0.57)}$ & $\mathbf{1.95(0.46)}$ \\
Twonorm
& $\mathbf{0.87(0.24)}$ & $1.26(0.23)$ & $4.27(0.53)$
& $8.02(0.70)$ & $\mathbf{0.98(0.19)}$ & $\mathbf{0.95(0.22)}$ & $\mathbf{0.84(0.20)}$ \\
Adult
& $4.36(0.27)$ & $\mathbf{2.32(0.47)}$ & $8.00(0.28)$
& $12.83(0.36)$ & $3.99(0.41)$ & $3.56(0.20)$ & $3.06(0.30)$ \\
Banknote
& $1.81(0.13)$ & $5.20(0.49)$ & $10.99(0.69)$
& $17.50(1.18)$ & $\mathbf{1.64(0.19)}$ & $\mathbf{1.78(0.20)}$ & $2.18(0.36)$ \\
Phoneme
& $\mathbf{2.39(0.59)}$ & $8.04(0.92)$ & $12.88(1.03)$
& $16.47(0.76)$ & $\mathbf{1.97(0.55)}$ & $\mathbf{2.09(0.29)}$ & $\mathbf{2.01(0.54)}$ \\
Magic
& $\mathbf{1.00(0.25)}$ & $8.54(0.54)$ & $14.76(0.73)$
& $19.69(0.52)$ & $1.19(0.18)$ & $\mathbf{1.04(0.21)}$ & $1.24(0.30)$ \\
Gisette
& $1.20(0.21)$ & $\mathbf{0.53(0.14)}$ & $2.32(0.19)$
& $5.03(0.23)$ & $1.29(0.23)$ & $1.32(0.23)$ & $1.55(0.21)$ \\
USPS
& $\mathbf{0.35(0.06)}$ & $0.44(0.08)$ & $1.93(0.19)$
& $4.60(0.25)$ & $\mathbf{0.31(0.06)}$ & $\mathbf{0.30(0.08)}$ & $\mathbf{0.31(0.07)}$ \\
Splice
& $4.76(0.62)$ & $\mathbf{1.60(0.29)}$ & $5.40(0.83)$
& $12.11(1.34)$ & $5.07(0.74)$ & $5.14(0.65)$ & $4.80(0.82)$ \\
Banana
& $3.31(0.77)$ & $11.17(0.62)$ & $18.78(0.37)$
& $24.61(0.50)$ & $\mathbf{2.79(0.45)}$ & $3.59(0.64)$ & $5.07(0.72)$ \\
Ringnorm
& $\mathbf{0.59(0.17)}$ & $4.17(0.32)$ & $9.55(0.30)$
& $15.71(0.67)$ & $\mathbf{0.71(0.18)}$ & $\mathbf{0.72(0.14)}$ & $1.02(0.26)$ \\
Image
& $\mathbf{2.78(0.63)}$ & $8.64(0.84)$ & $15.56(0.96)$
& $21.21(0.63)$ & $\mathbf{2.91(0.47)}$ & $\mathbf{3.04(0.55)}$ & $\mathbf{3.07(0.64)}$ \\
Coil20
& $\mathbf{0.15(0.04)}$ & $1.20(0.16)$ & $3.49(0.27)$
& $7.75(0.42)$ & $\mathbf{0.18(0.07)}$ & $\mathbf{0.15(0.10)}$ & $\mathbf{0.14(0.03)}$ \\
Basehock
& $\mathbf{1.26(0.19)}$ & $4.74(0.55)$ & $9.94(0.65)$
& $14.59(0.66)$ & $\mathbf{1.11(0.31)}$ & $1.42(0.32)$ & $\mathbf{1.22(0.23)}$ \\
Isolet
& $\mathbf{0.60(0.18)}$ & $1.56(0.21)$ & $3.81(0.43)$
& $7.43(0.35)$ & $\mathbf{0.51(0.15)}$ & $\mathbf{0.56(0.09)}$ & $\mathbf{0.60(0.22)}$ \\
W8a
& $\mathbf{0.38(0.13)}$ & $1.31(0.22)$ & $3.65(0.32)$
& $6.64(0.37)$ & $\mathbf{0.38(0.11)}$ & $\mathbf{0.39(0.08)}$ & $\mathbf{0.30(0.14)}$ \\
Mushroom
& $\mathbf{0.05(0.02)}$ & $0.80(0.08)$ & $3.36(0.07)$
& $7.25(0.10)$ & $\mathbf{0.06(0.05)}$ & $\mathbf{0.07(0.05)}$ & $\mathbf{0.05(0.02)}$ \\
Artificial-character
& $\mathbf{3.34(0.93)}$ & $4.58(0.78)$ & $7.13(0.93)$
& $9.08(0.80)$ & $\mathbf{3.27(0.39)}$ & $3.97(0.97)$ & $4.79(0.72)$ \\
Gas-drift
& $0.77(0.08)$ & $3.23(0.22)$ & $6.48(0.41)$
& $10.18(0.36)$ & $\mathbf{0.64(0.13)}$ & $\mathbf{0.73(0.20)}$ & $0.82(0.08)$ \\
Japanesevowels
& $1.24(0.20)$ & $4.22(0.17)$ & $7.41(0.21)$
& $10.53(0.45)$ & $0.84(0.12)$ & $\mathbf{0.60(0.11)}$ & $0.70(0.14)$ \\
Letter
& $2.80(0.34)$ & $7.66(0.33)$ & $11.61(0.25)$
& $15.23(0.45)$ & $1.99(0.27)$ & $\mathbf{1.42(0.21)}$ & $\mathbf{1.59(0.25)}$ \\
Pendigits
& $1.01(0.10)$ & $3.98(0.35)$ & $8.38(0.39)$
& $12.79(0.43)$ & $0.80(0.19)$ & $\mathbf{0.74(0.12)}$ & $\mathbf{0.65(0.10)}$ \\
Satimage
& $\mathbf{1.30(0.36)}$ & $5.62(0.84)$ & $11.22(0.77)$
& $15.24(1.14)$ & $1.72(0.43)$ & $1.68(0.48)$ & $1.99(0.52)$ \\
Vehicle
& $\mathbf{5.13(1.44)}$ & $8.58(1.46)$ & $13.03(1.14)$
& $16.70(1.03)$ & $5.22(1.29)$ & $\mathbf{4.43(1.12)}$ & $\mathbf{4.37(0.68)}$ \\
MNIST
& $1.04(0.07)$ & $\mathbf{0.42(0.05)}$ & $1.67(0.07)$
& $3.39(0.14)$ & $1.16(0.06)$ & $1.22(0.09)$ & $1.34(0.09)$ \\
KMNIST
& $8.17(0.28)$ & $5.42(0.15)$ & $3.27(0.24)$
& $\mathbf{1.22(0.24)}$ & $8.12(0.16)$ & $8.62(0.25)$ & $9.02(0.29)$ \\
Fashion-MNIST
& $3.92(0.24)$ & $\mathbf{0.92(0.20)}$ & $4.14(0.59)$
& $7.36(0.79)$ & $4.04(0.35)$ & $4.47(0.46)$ & $5.05(0.66)$ \\
\bottomrule
\end{tabular}
\label{tab:mlp-ece}
\vspace{-3mm}
\end{table*}

\begin{table*}[t]
\centering
\caption{Mean and standard error of the classification error over ten trials (rescaled to $0-100$).
We used \textbf{neural networks with one hidden layer} as a model.
The purpose of showing this table is for reference.
}
\begin{tabular}{lccccccc}
\toprule
Dataset
& CE & FL-$1$ & FL-$2$& FL-$3$ \\
\midrule
Australian
& $13.42(1.24)$ & $13.68(0.95)$ & $13.80(1.46)$
& $13.80(1.28)$ \\
Phishing
& $4.46(0.21)$ & $4.66(0.36)$ & $4.77(0.25)$
& $5.49(0.16)$ \\
Spambase
& $6.20(0.40)$ & $6.15(0.43)$ & $6.33(0.36)$
& $6.36(0.40)$ \\
Waveform
& $9.05(0.26)$ & $9.21(0.26)$ & $9.08(0.52)$
& $9.24(0.48)$ \\
Twonorm
& $2.52(0.18)$ & $2.58(0.15)$ & $2.61(0.14)$
& $2.54(0.19)$ \\
Adult
& $16.27(0.20)$ & $16.17(0.25)$ & $16.02(0.24)$
& $15.72(0.22)$ \\
Banknote-authentication
& $0.60(0.50)$ & $0.71(0.55)$ & $0.71(0.54)$
& $1.18(0.90)$ \\
Phoneme
& $16.33(0.81)$ & $16.64(0.76)$ & $17.11(0.54)$
& $17.95(0.83)$ \\
Magic
& $13.23(0.32)$ & $13.29(0.31)$ & $13.69(0.49)$
& $13.66(0.14)$ \\
Gisette
& $2.16(0.25)$ & $2.34(0.23)$ & $2.38(0.29)$
& $2.61(0.21)$ \\
USPS
& $0.58(0.12)$ & $0.59(0.07)$ & $0.58(0.07)$
& $0.58(0.09)$ \\
Splice
& $9.68(0.49)$ & $10.15(0.74)$ & $10.70(1.06)$
& $11.58(1.17)$ \\
Banana
& $10.63(0.71)$ & $10.48(0.38)$ & $10.45(0.33)$
& $10.77(0.35)$ \\
Ringnorm
& $2.30(0.17)$ & $2.25(0.16)$ & $2.25(0.22)$
& $2.50(0.24)$ \\
Image
& $5.85(0.84)$ & $6.16(1.14)$ & $6.61(1.01)$
& $7.89(0.95)$ \\
Coil20
& $0.06(0.13)$ & $0.03(0.06)$ & $0.03(0.08)$
& $0.01(0.04)$ \\
Basehock
& $3.17(0.51)$ & $3.49(0.41)$ & $3.83(0.73)$
& $4.07(0.67)$ \\
Isolet
& $0.73(0.24)$ & $0.51(0.18)$ & $0.72(0.28)$
& $0.92(0.27)$ \\
W8a
& $1.05(0.04)$ & $1.05(0.05)$ & $1.06(0.05)$
& $1.04(0.05)$ \\
Mushroom
& $0.02(0.02)$ & $0.04(0.07)$ & $0.06(0.07)$
& $0.03(0.02)$ \\
Artificial-character
& $38.93(0.91)$ & $38.90(0.74)$ & $38.90(0.75)$
& $39.49(0.84)$ \\
Gas-drift
& $0.87(0.14)$ & $0.95(0.23)$ & $0.93(0.12)$
& $0.91(0.10)$ \\
Japanesevowels
& $3.07(0.33)$ & $3.17(0.20)$ & $3.27(0.14)$
& $3.69(0.26)$ \\
Letter
& $9.54(0.41)$ & $9.50(0.31)$ & $9.71(0.25)$
& $10.34(0.41)$ \\
Pendigits
& $1.07(0.14)$ & $1.10(0.15)$ & $1.16(0.17)$
& $1.24(0.14)$ \\
Satimage
& $10.70(0.26)$ & $11.00(0.42)$ & $11.28(0.41)$
& $11.11(0.52)$ \\
Vehicle
& $23.45(1.93)$ & $24.44(2.03)$ & $24.26(1.29)$
& $25.51(1.11)$ \\
MNIST
& $2.35(0.07)$ & $2.43(0.07)$ & $2.40(0.08)$
& $2.54(0.11)$ \\
KMNIST
& $13.03(0.25)$ & $12.91(0.17)$ & $13.28(0.27)$
& $13.61(0.30)$ \\
Fashion-MNIST
& $11.89(0.27)$ & $11.92(0.31)$ & $12.07(0.14)$
& $12.70(0.71)$ \\
\bottomrule
\end{tabular}
\label{tab:mlp-err}
\vspace{-3mm}
\end{table*}
\FloatBarrier

\subsection{Additional reliability diagrams for ResNet models}
Here, we present reliability diagrams for all ResNet models we used in this paper on the SVHN, CIFAR10, CIFAR10-s, and CIFAR100 datasets.
\subsubsection*{Reliability diagram index:}
\begin{itemize}
    \item Figures~\ref{app:fig-rel-svhn-res8}-\ref{app:fig-rel-svhn-res110}: Reliability diagrams for SVHN using ResNet8, ResNet20, ResNet44, and ResNet110.
    \item Figures~\ref{app:fig-rel-cifar10-res8}-\ref{app:fig-rel-cifar10-res110}: Reliability diagrams for CIFAR10 using ResNet8, ResNet20, ResNet44, and ResNet110.
    \item Figures~\ref{app:fig-rel-cifar10s-res8}-\ref{app:fig-rel-cifar10s-res110}: Reliability diagrams for CIFAR10-s using ResNet8, ResNet20, ResNet44, and ResNet110.
    \item Figures~\ref{app:fig-rel-cifar100-res8}-\ref{app:fig-rel-cifar100-res110}: Reliability diagrams for CIFAR100 using ResNet8, ResNet20, ResNet44, and ResNet110.
\end{itemize}

\FloatBarrier
\newpage
\begin{figure*}
    \centering
    \includegraphics[width=0.7\textwidth]{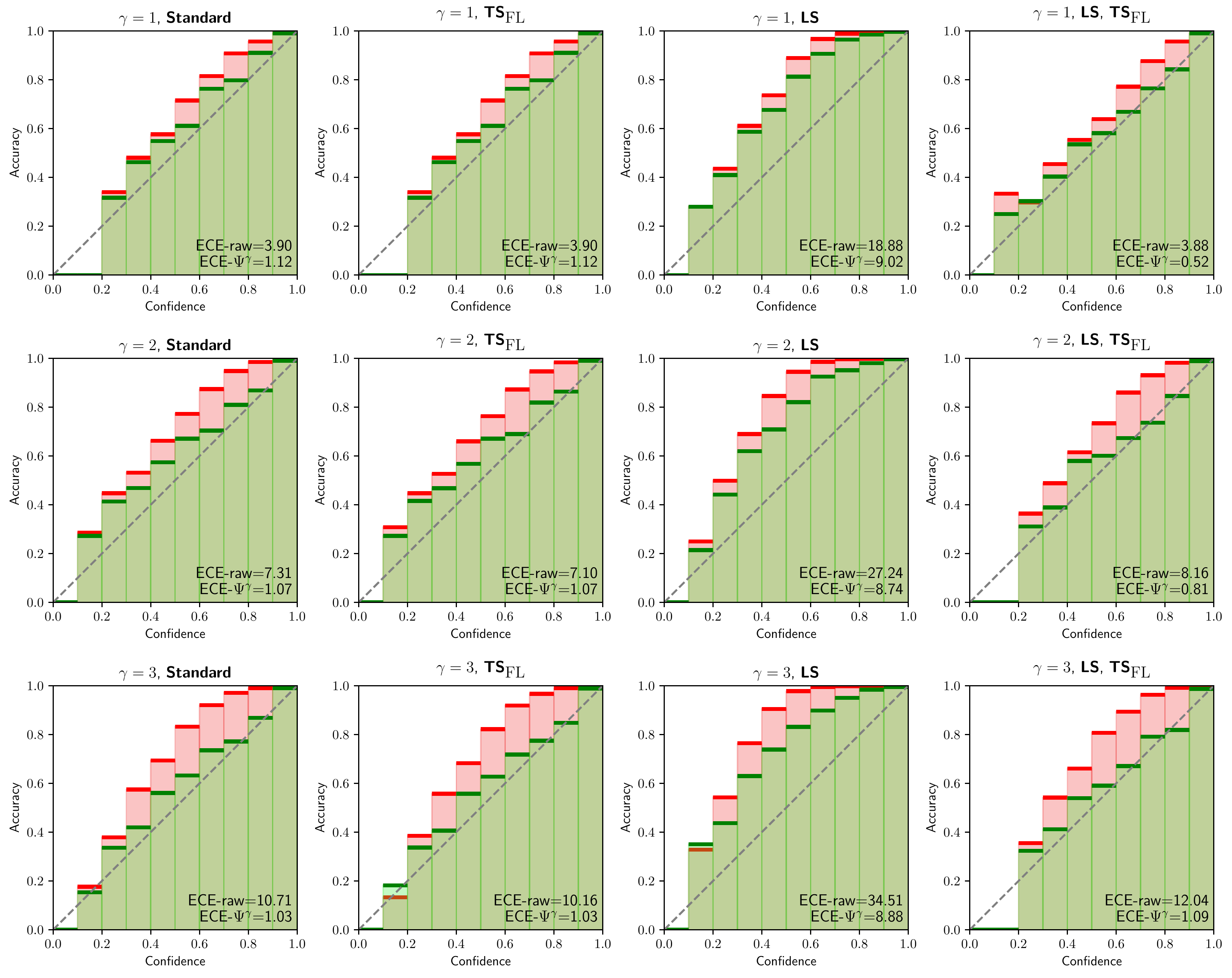}
    \caption{Reliability diagrams for SVHN and ResNet8. See each graph's title for the particular details.}
    \label{app:fig-rel-svhn-res8}
\end{figure*}         

\begin{figure*}
    \centering
    \includegraphics[width=0.7\textwidth]{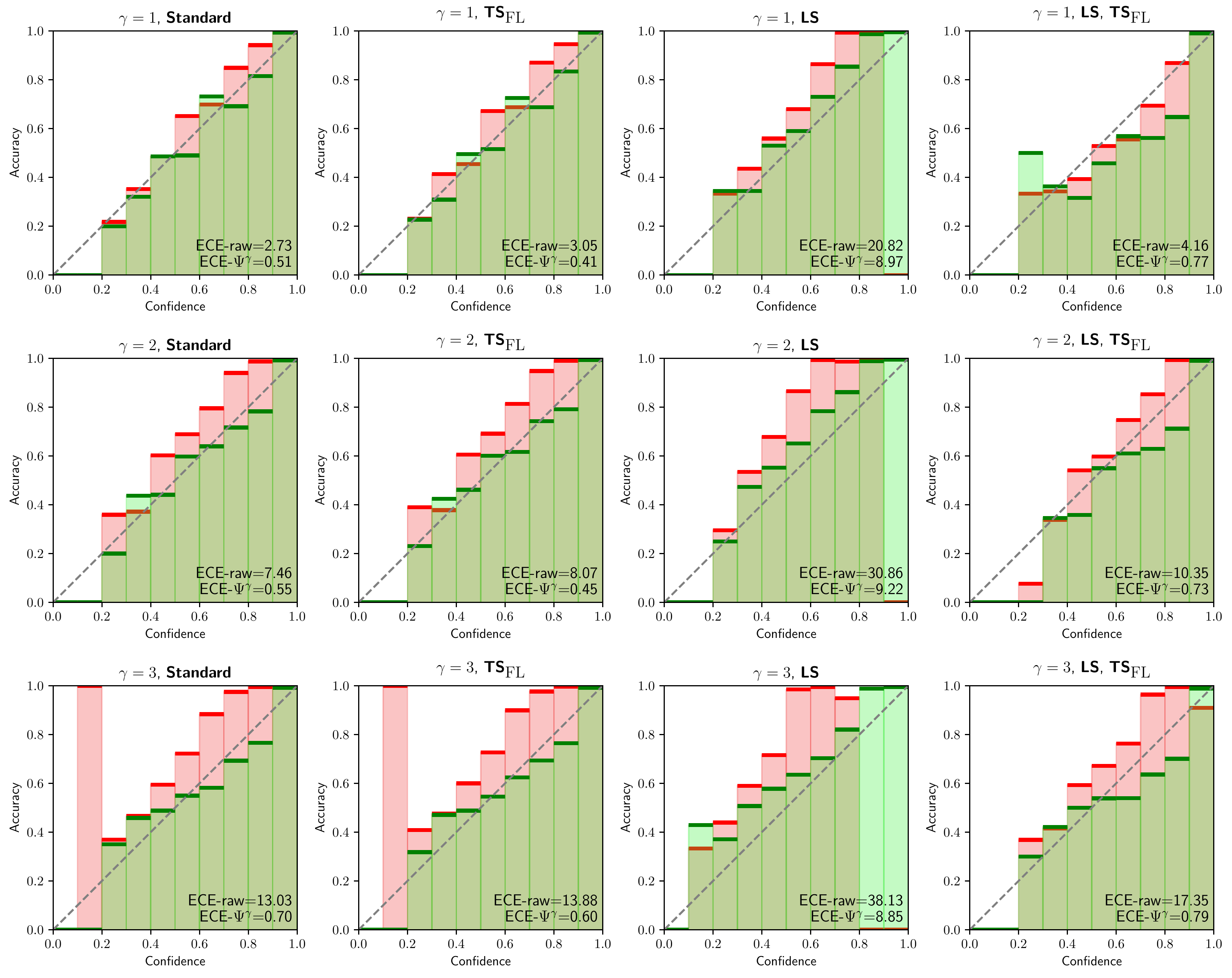}
    \caption{Reliability diagrams for SVHN and ResNet20. See each graph's title for the particular details.}
\end{figure*}    

\begin{figure*}
    \centering
    \includegraphics[width=0.7\textwidth]{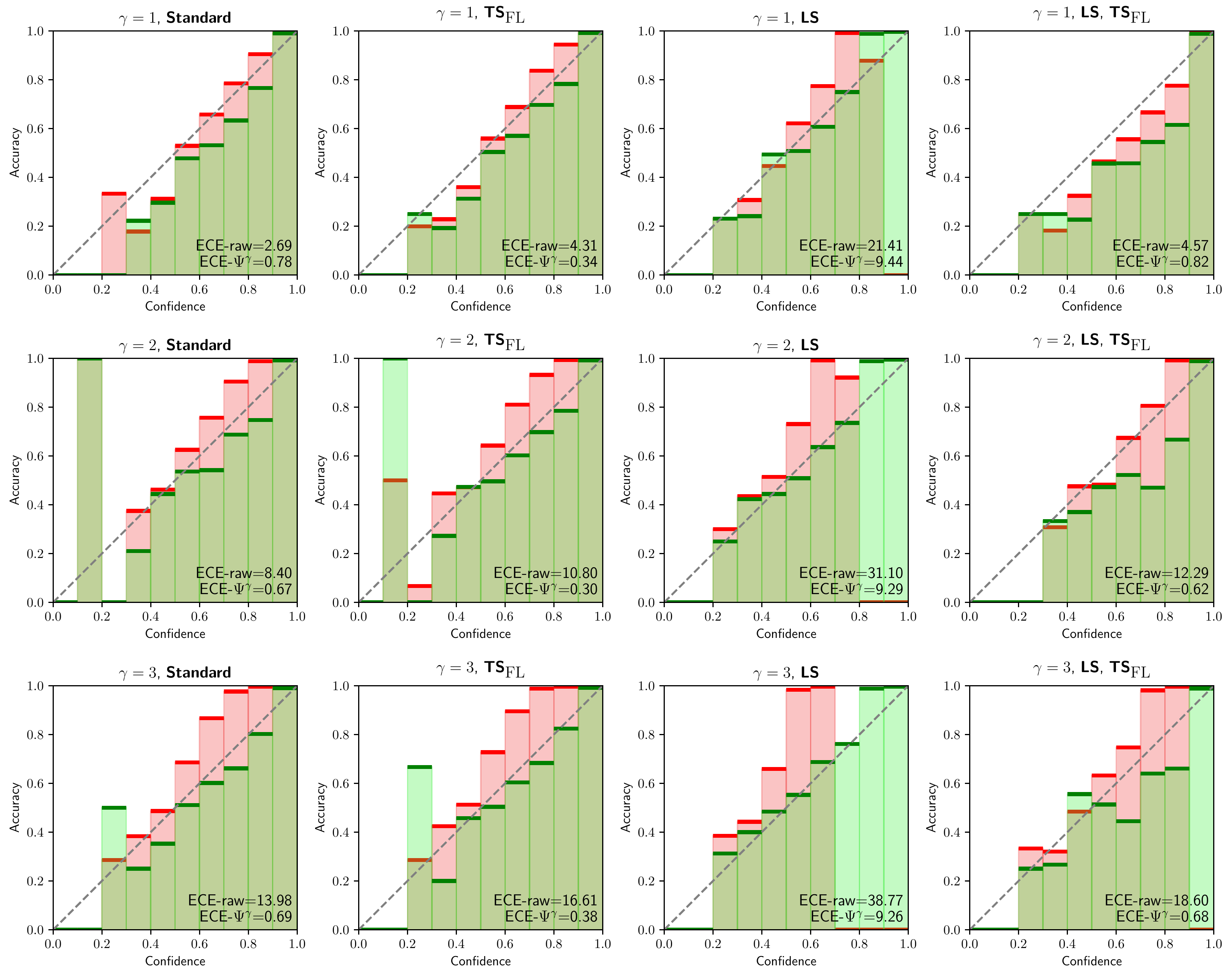}
    \caption{Reliability diagrams for SVHN and ResNet44. See each graph's title for the particular details.}
\end{figure*}    

\begin{figure*}
    \centering
    \includegraphics[width=0.7\textwidth]{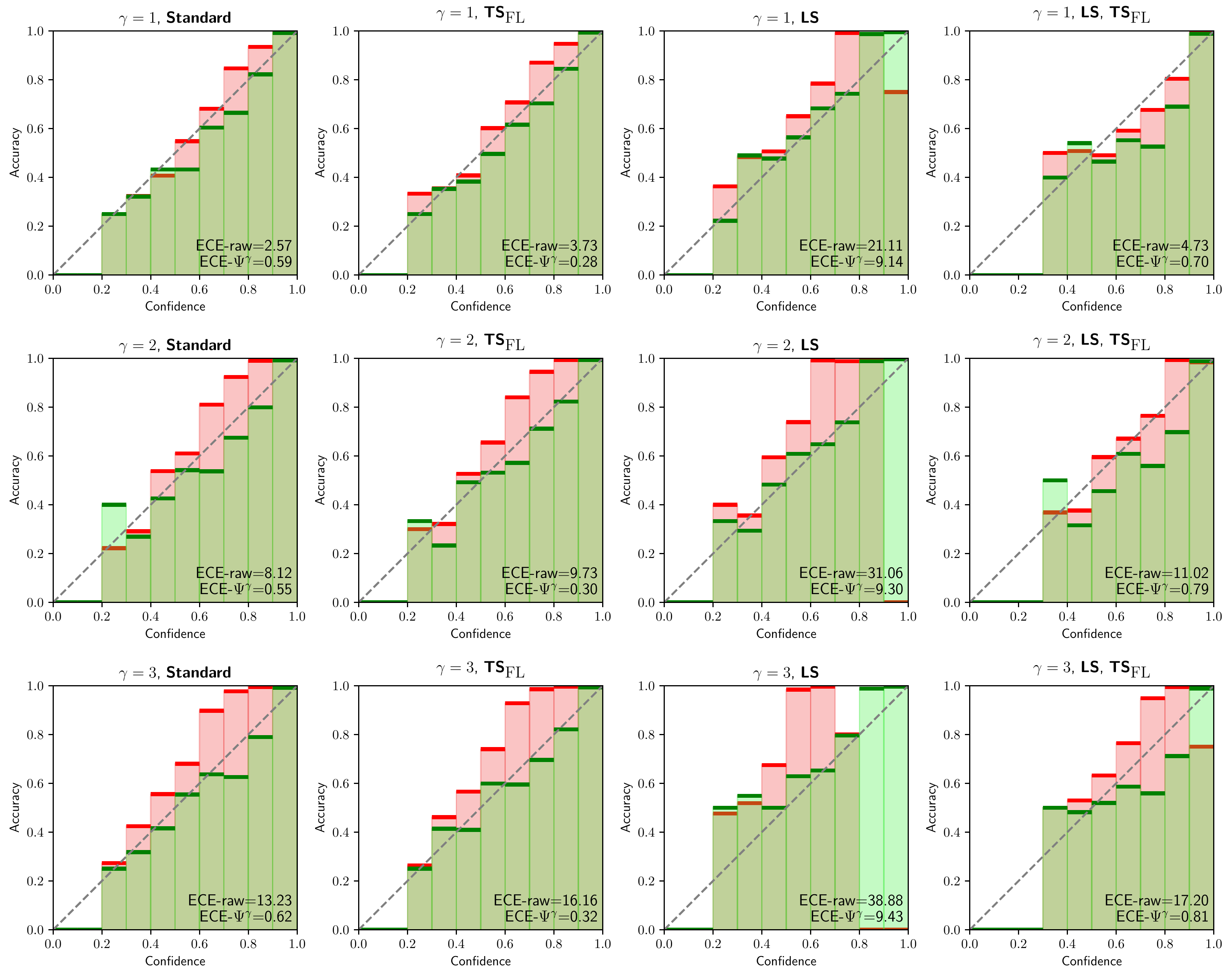}
    \caption{Reliability diagrams for SVHN and ResNet110. See each graph's title for the particular details.}
    \label{app:fig-rel-svhn-res110}
\end{figure*}

\begin{figure*}
    \centering
    \includegraphics[width=0.7\textwidth]{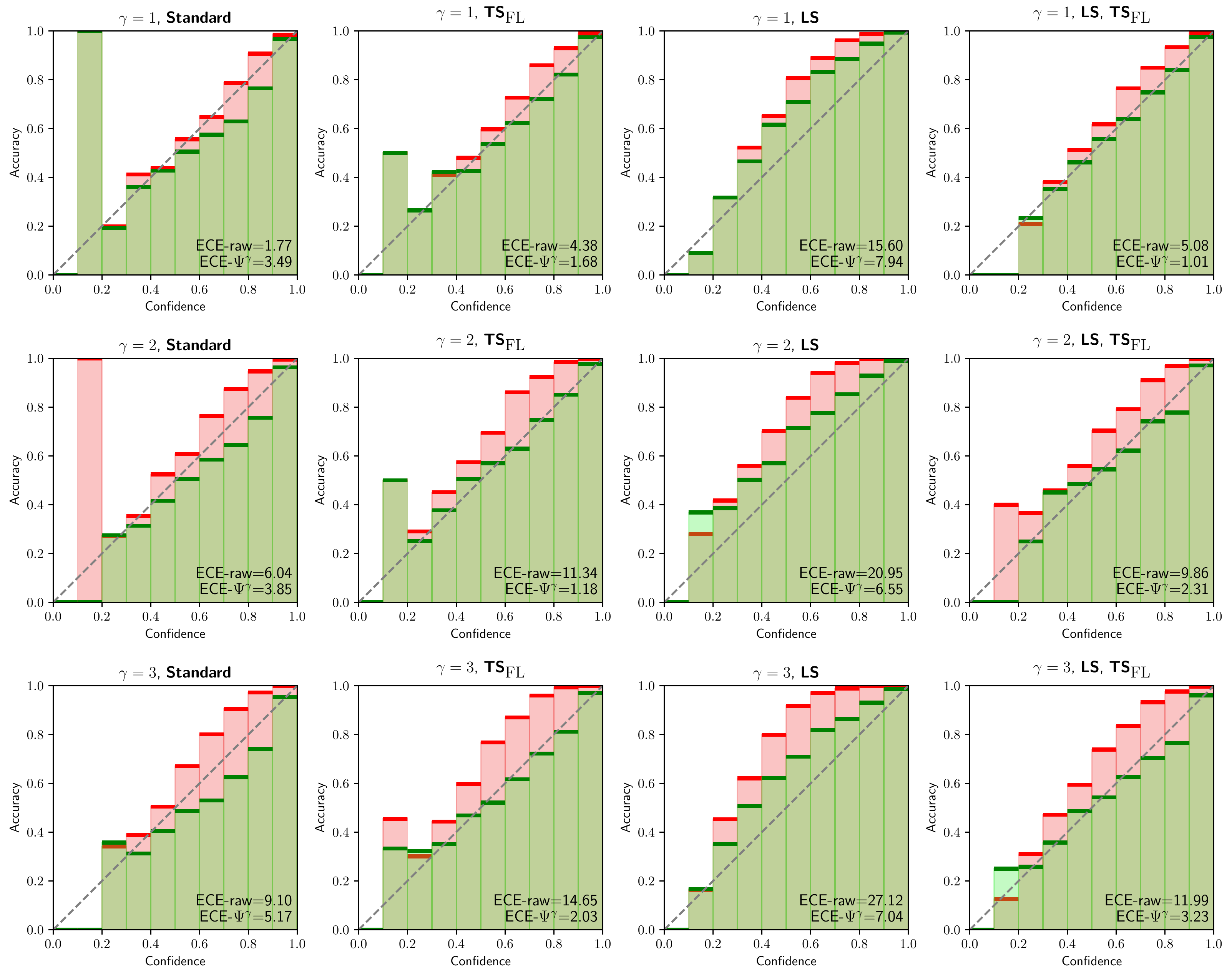}
    \caption{Reliability diagrams for CIFAR10 and ResNet8. See each graph's title for the particular details.}
    \label{app:fig-rel-cifar10-res8}
\end{figure*}         

\begin{figure*}
    \centering
    \includegraphics[width=0.7\textwidth]{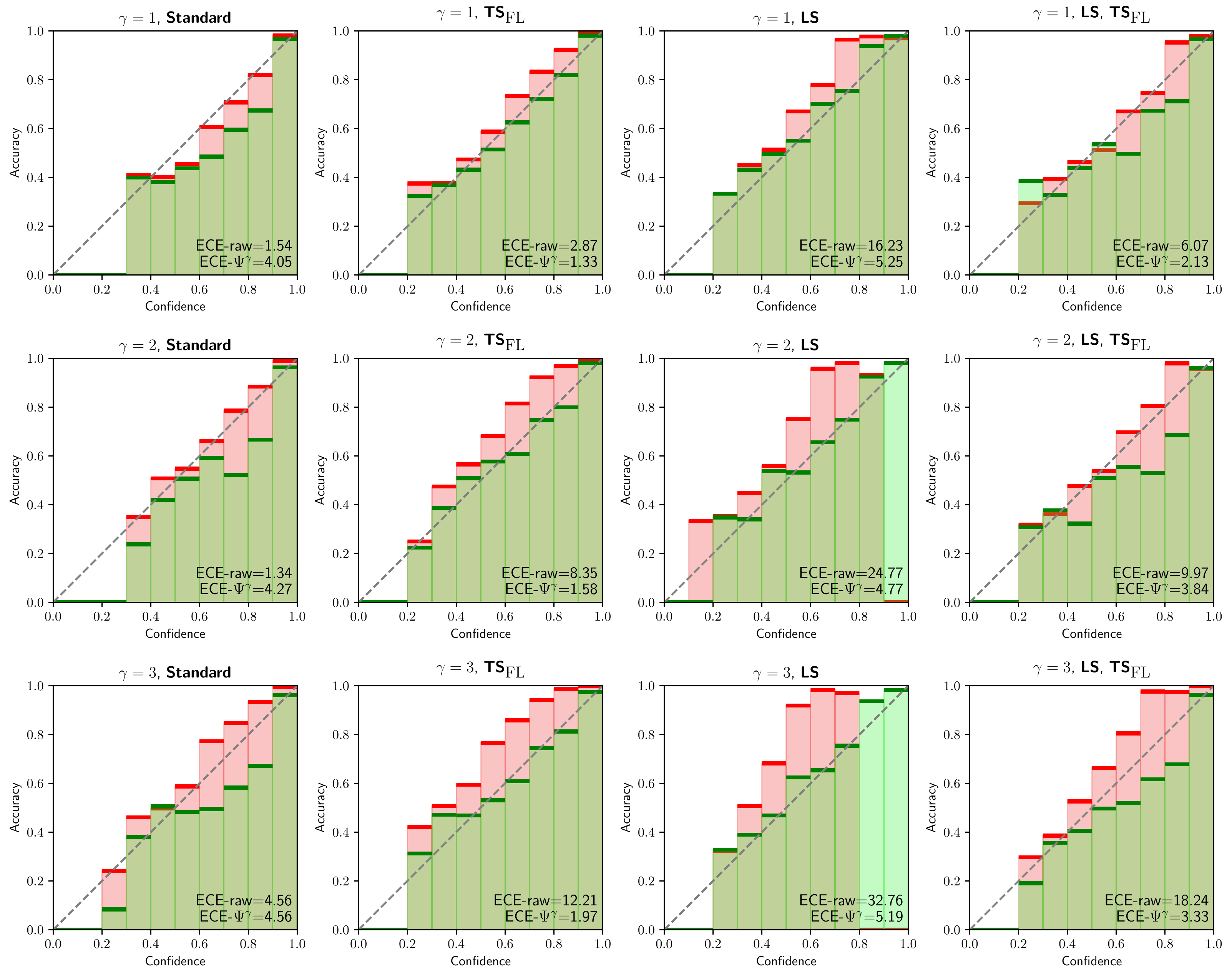}
    \caption{Reliability diagrams for CIFAR10 and ResNet20. See each graph's title for the particular details.}
\end{figure*}    

\begin{figure*}
    \centering
    \includegraphics[width=0.7\textwidth]{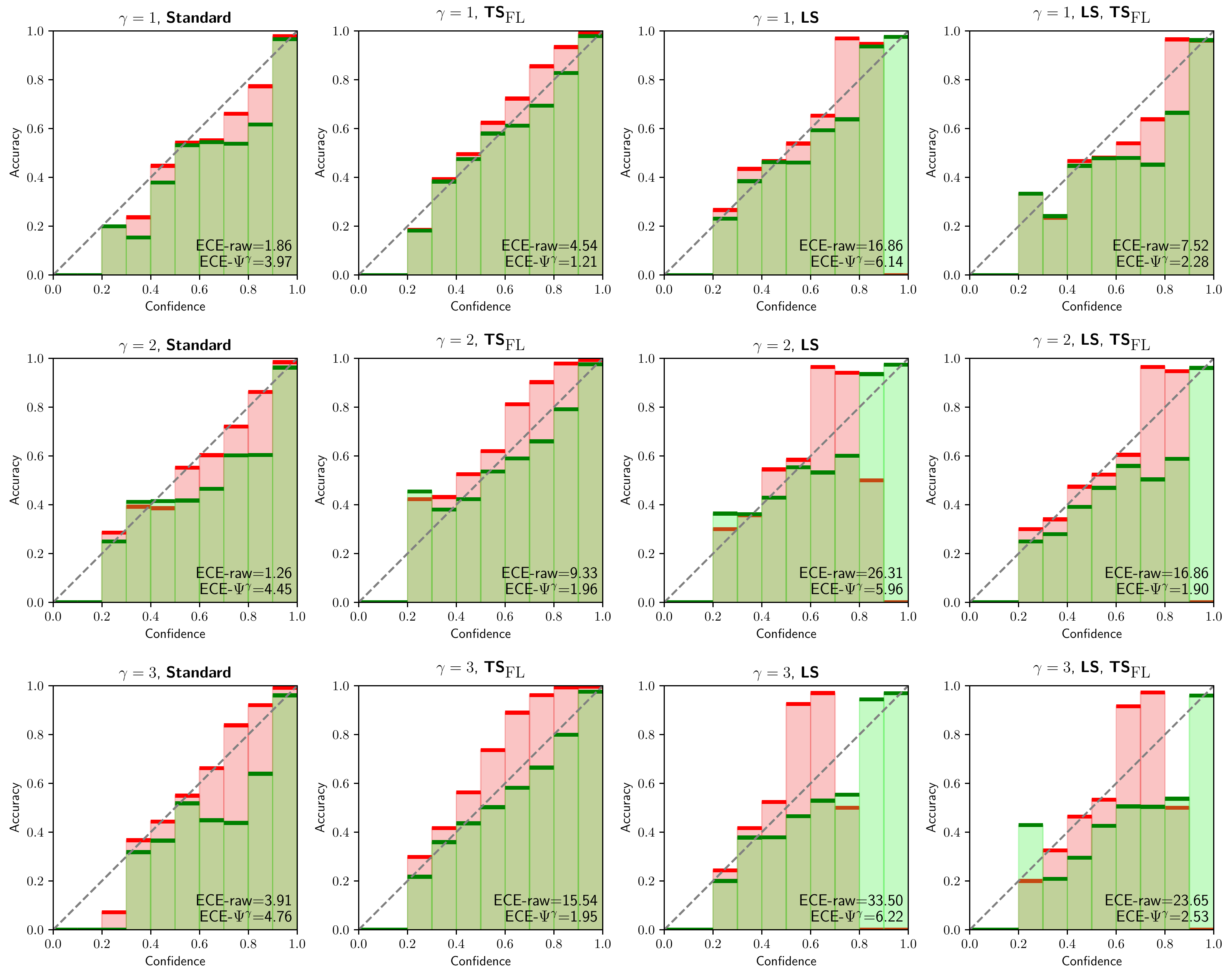}
    \caption{Reliability diagrams for CIFAR10 and ResNet44. See each graph's title for the particular details.}
\end{figure*}    

\begin{figure*}
    \centering
    \includegraphics[width=0.7\textwidth]{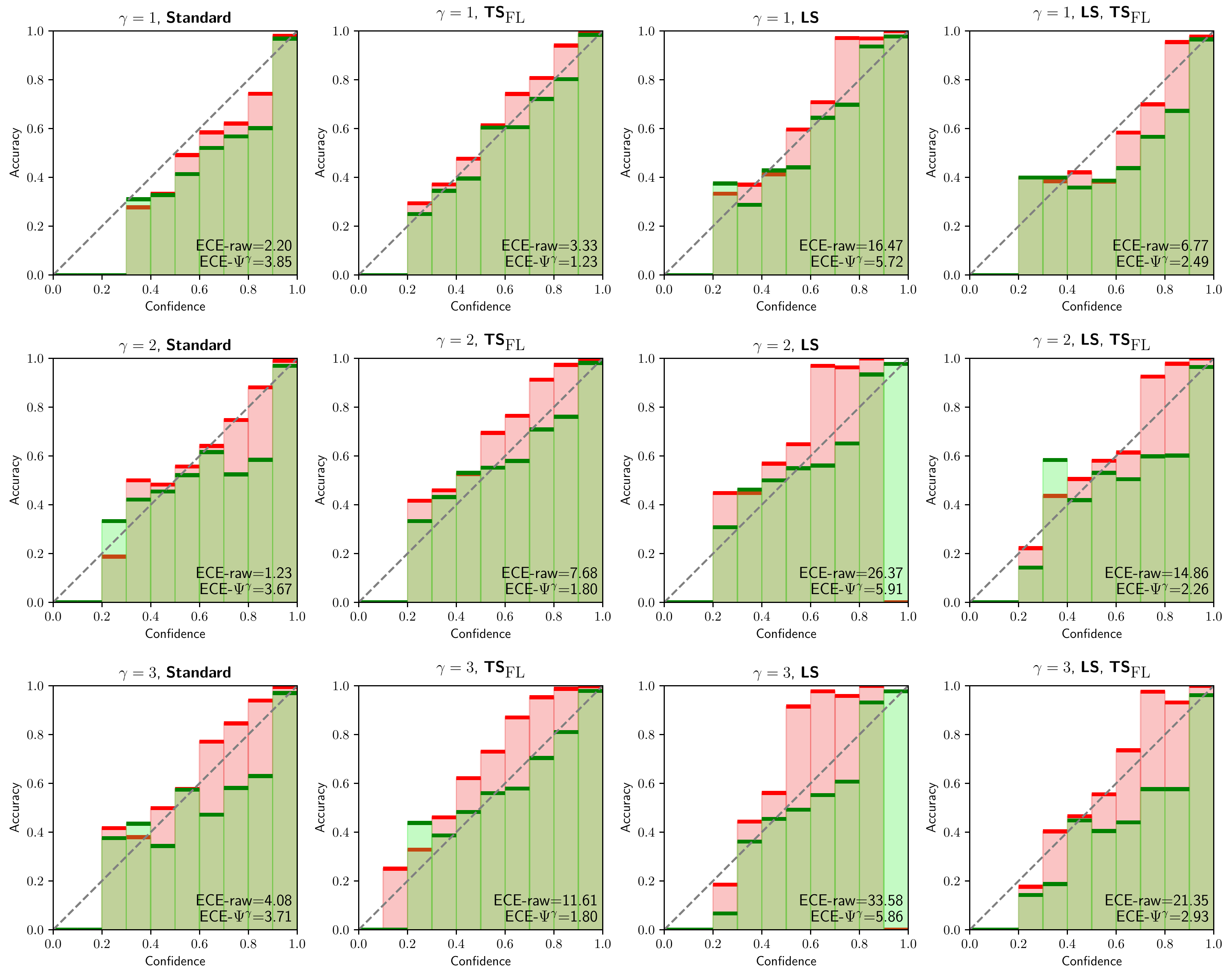}
    \caption{Reliability diagrams for CIFAR10 and ResNet110. See each graph's title for the particular details.}
    \label{app:fig-rel-cifar10-res110}
\end{figure*}

\begin{figure*}
    \centering
    \includegraphics[width=0.7\textwidth]{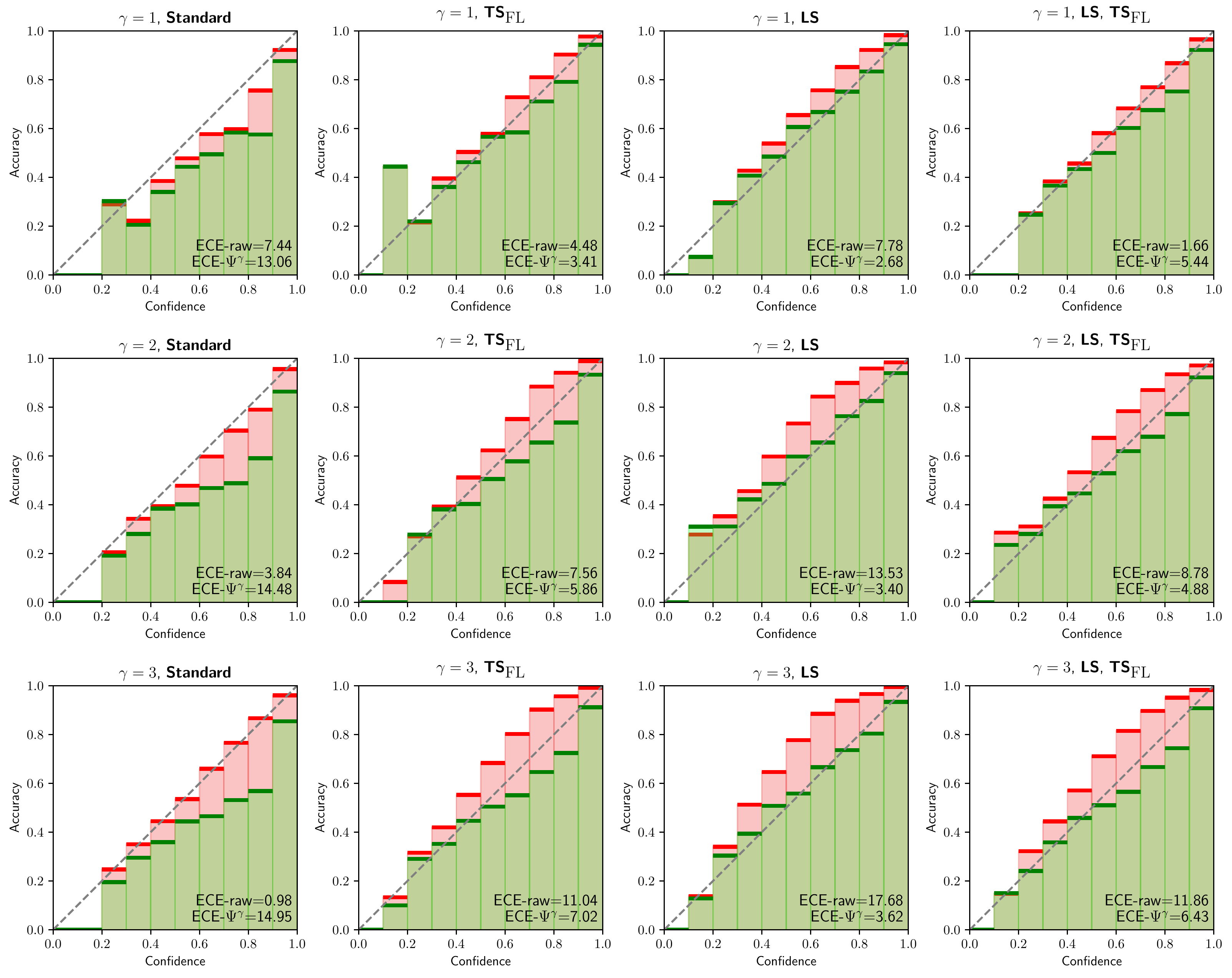}
    \caption{Reliability diagrams for CIFAR10-s and ResNet8. See each graph's title for the particular details.}
    \label{app:fig-rel-cifar10s-res8}
\end{figure*}         

\begin{figure*}
    \centering
    \includegraphics[width=0.7\textwidth]{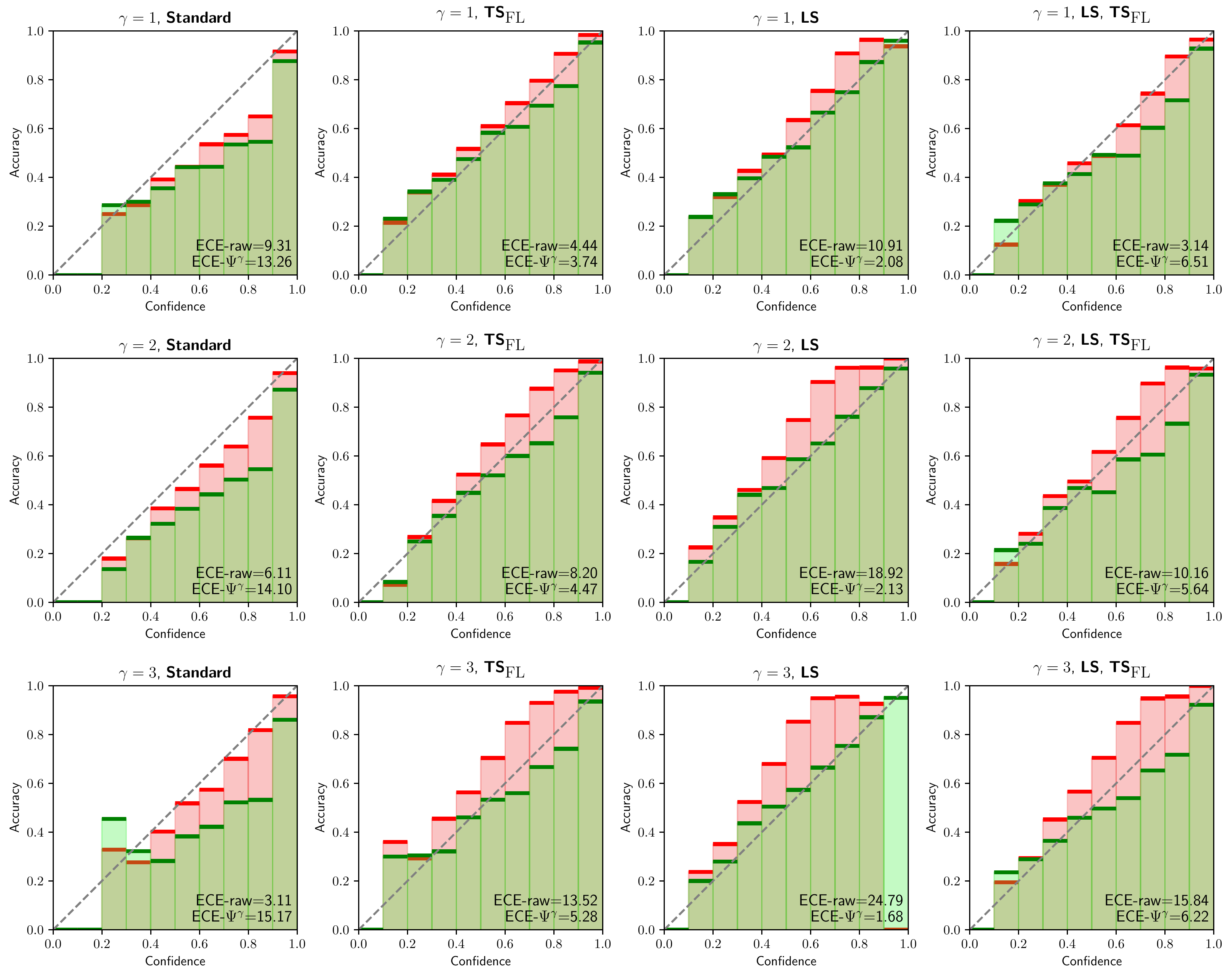}
    \caption{Reliability diagrams for CIFAR10-s and ResNet20. See each graph's title for the particular details.}
\end{figure*}    

\begin{figure*}
    \centering
    \includegraphics[width=0.7\textwidth]{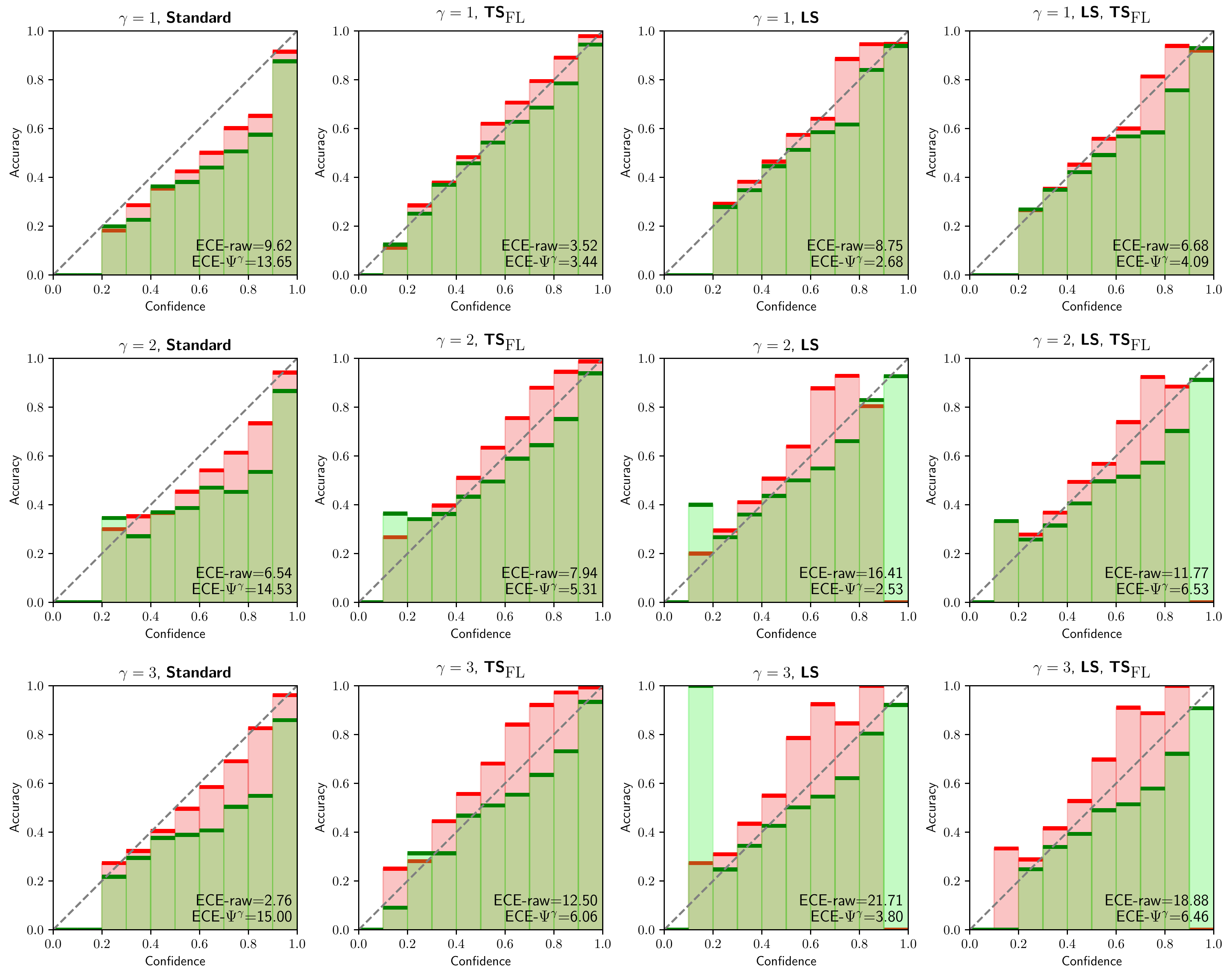}
    \caption{Reliability diagrams for CIFAR10-s and ResNet44. See each graph's title for the particular details.}
\end{figure*}    

\begin{figure*}
    \centering
    \includegraphics[width=0.7\textwidth]{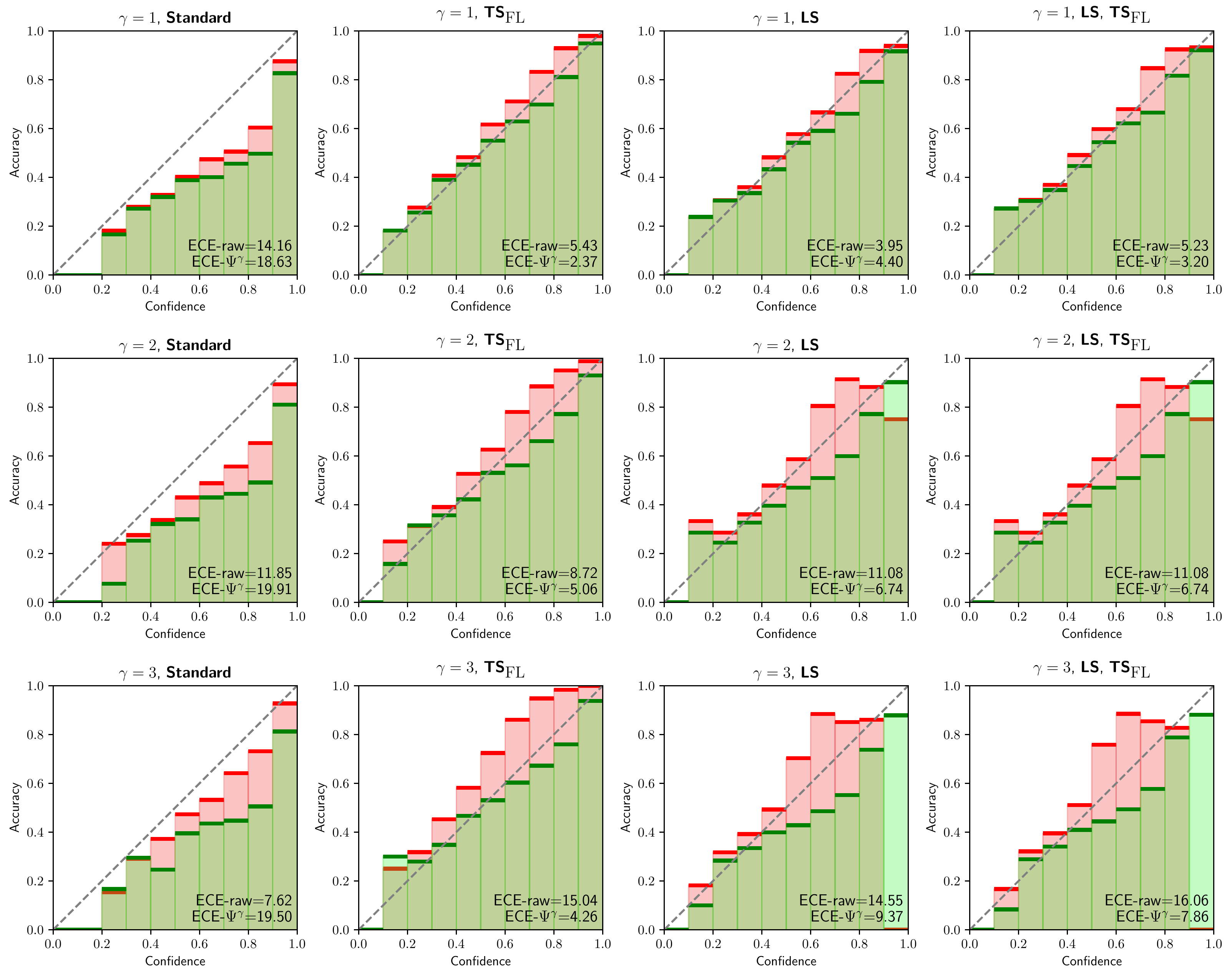}
    \caption{Reliability diagrams for CIFAR10-s and ResNet110. See each graph's title for the particular details.}
    \label{app:fig-rel-cifar10s-res110}
\end{figure*}

\begin{figure*}
    \centering
    \includegraphics[width=0.7\textwidth]{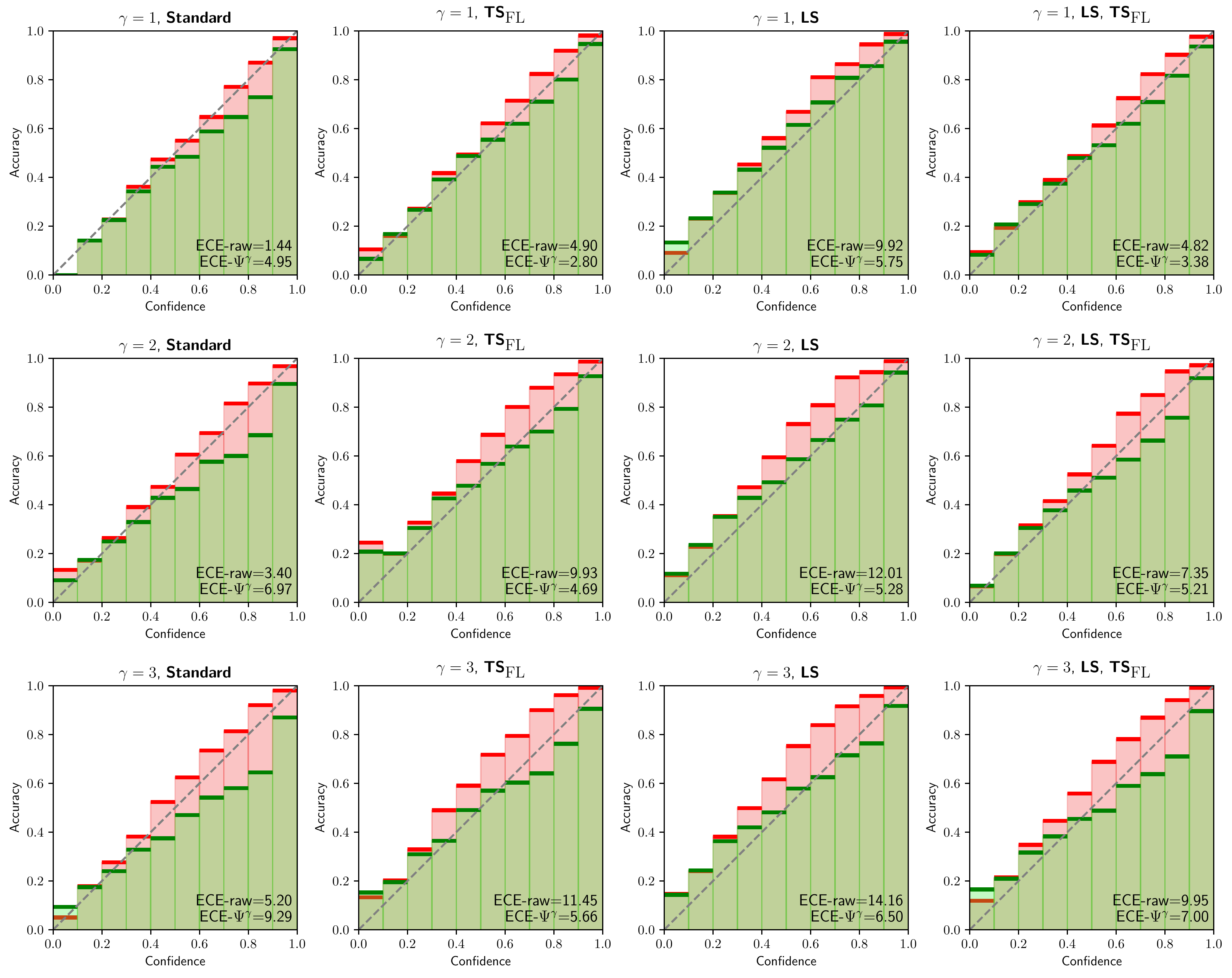}
    \caption{Reliability diagrams for CIFAR100 and ResNet8. See each graph's title for the particular details.}
    \label{app:fig-rel-cifar100-res8}
\end{figure*}         

\begin{figure*}
    \centering
    \includegraphics[width=0.7\textwidth]{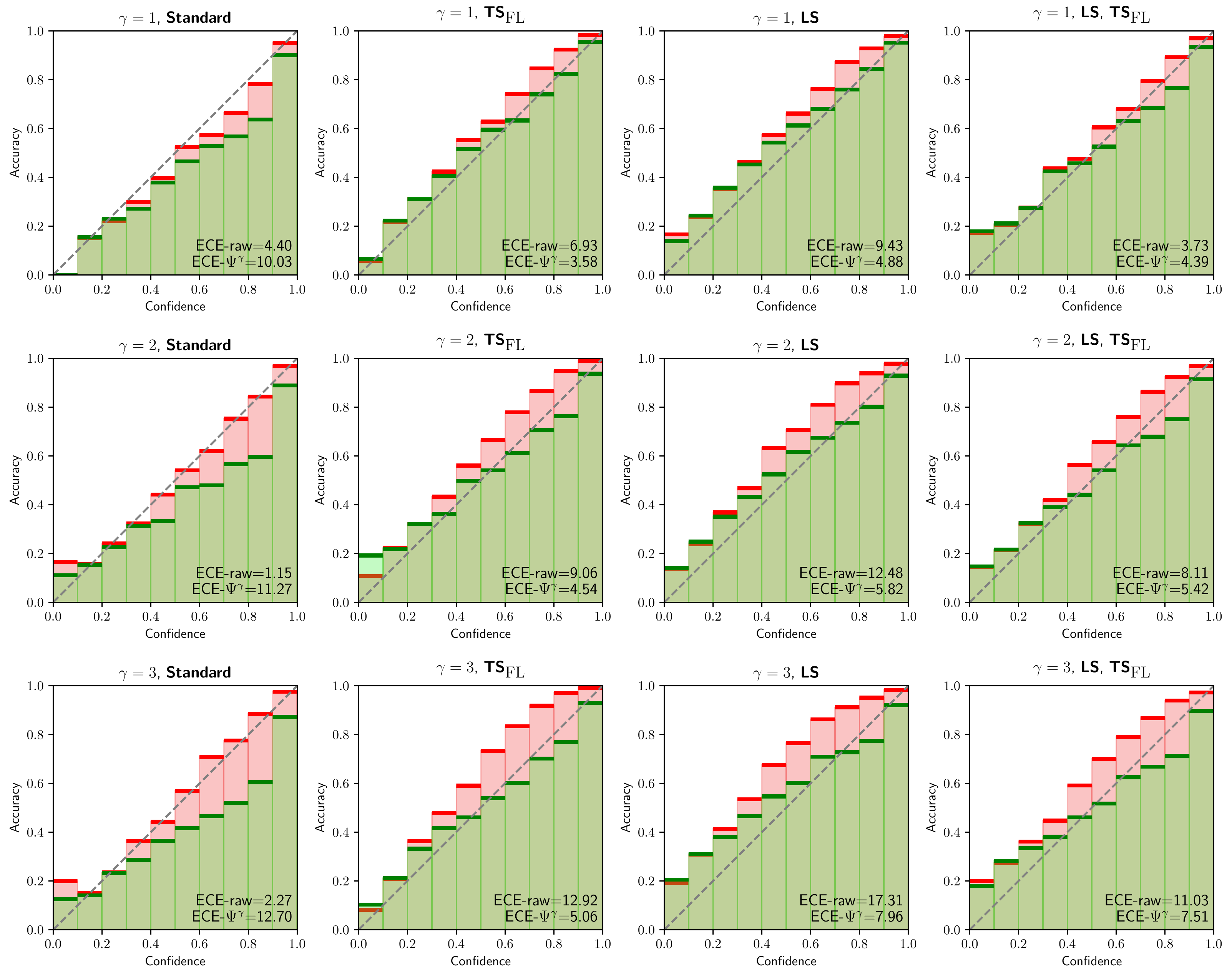}
    \caption{Reliability diagrams for CIFAR100 and ResNet20. See each graph's title for the particular details.}
\end{figure*}    

\begin{figure*}
    \centering
    \includegraphics[width=0.7\textwidth]{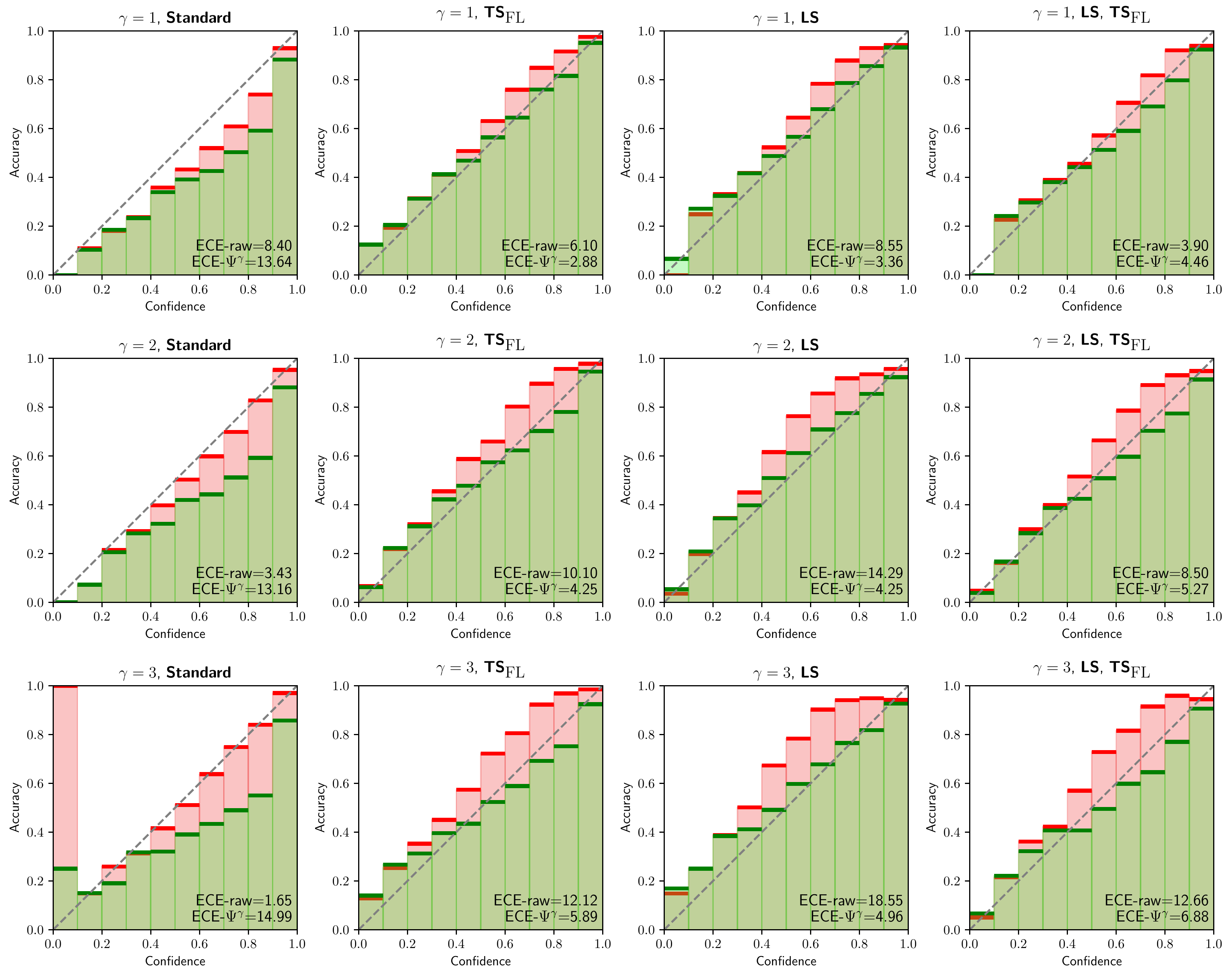}
    \caption{Reliability diagrams for CIFAR100 and ResNet44. See each graph's title for the particular details.}
\end{figure*}    

\begin{figure*}
    \centering
    \includegraphics[width=0.7\textwidth]{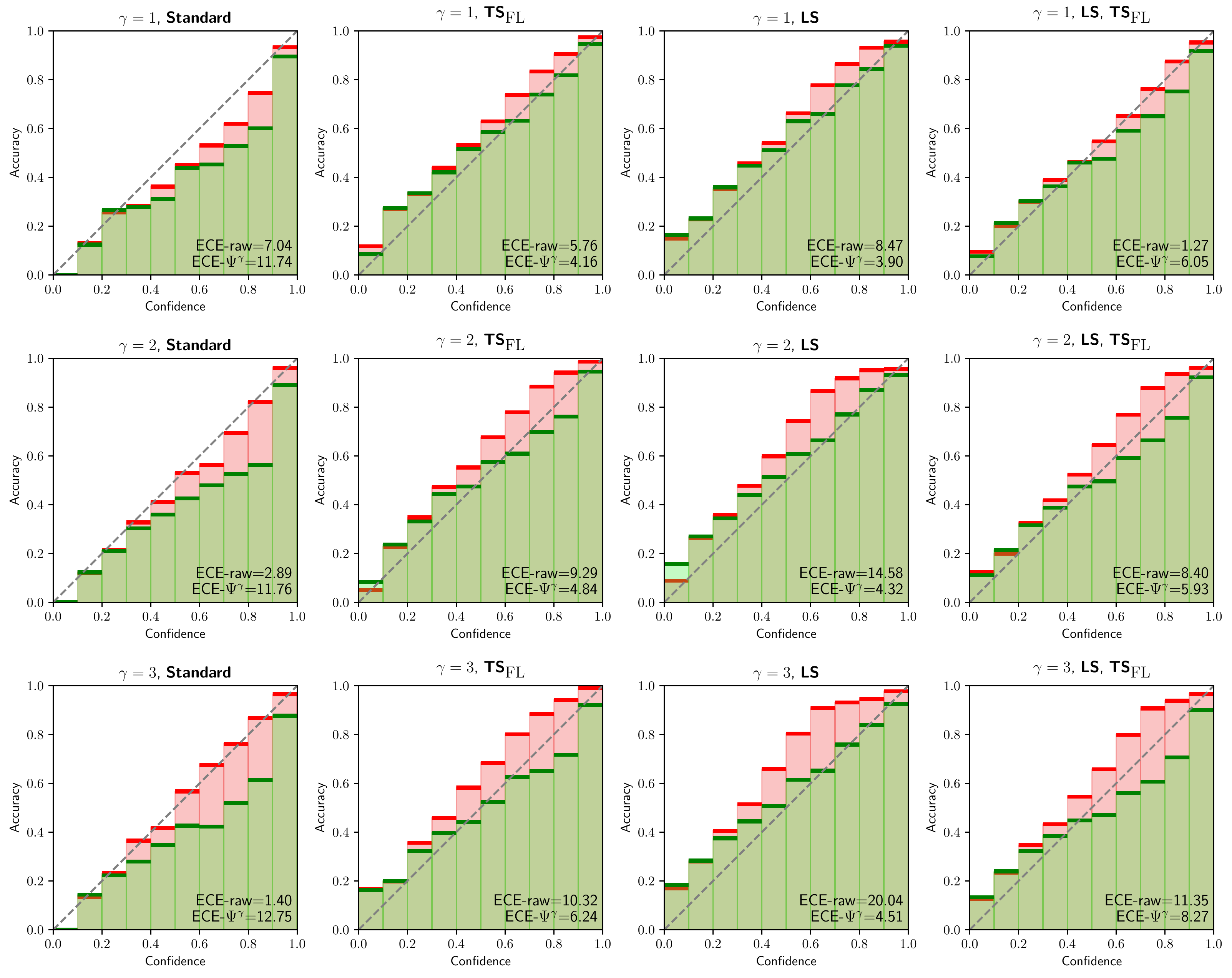}
    \caption{Reliability diagrams for CIFAR100 and ResNet110. See each graph's title for the particular details.}
    \label{app:fig-rel-cifar100-res110}
\end{figure*}

\newpage
\FloatBarrier
\subsection{Additional box plots of ECE, NLL, and CW-ECE for ResNet models}
Here, we present box plots of ECE, NLL, and CW-ECE for all ResNet models we used in this paper on the SVHN, CIFAR10, CIFAR10-s, and CIFAR100 datasets.
More details on evaluation metrics can be found in Appx.~\ref{app:metrics}.
\subsubsection*{Box plot index:}
\begin{itemize}
    \item Figures~\ref{app:fig-box-svhn-res8}-\ref{app:fig-box-svhn-res110}: Box plots for SVHN using ResNet8, ResNet20, ResNet44, and ResNet110.
    \item Figures~\ref{app:fig-box-cifar10-res8}-\ref{app:fig-box-cifar10-res110}: Box plots for CIFAR10 using ResNet8, ResNet20, ResNet44, and ResNet110.
    \item Figures~\ref{app:fig-box-cifar10s-res8}-\ref{app:fig-box-cifar10s-res110}: Box plots for CIFAR10-s using ResNet8, ResNet20, ResNet44, and ResNet110.
    \item Figures~\ref{app:fig-box-cifar100-res8}-\ref{app:fig-box-cifar100-res110}: Box plots for CIFAR100 using ResNet8, ResNet20, ResNet44, and ResNet110.
\end{itemize}
\FloatBarrier
\newpage
\begin{figure*}
    \centering
    \includegraphics[width=0.65\textwidth]{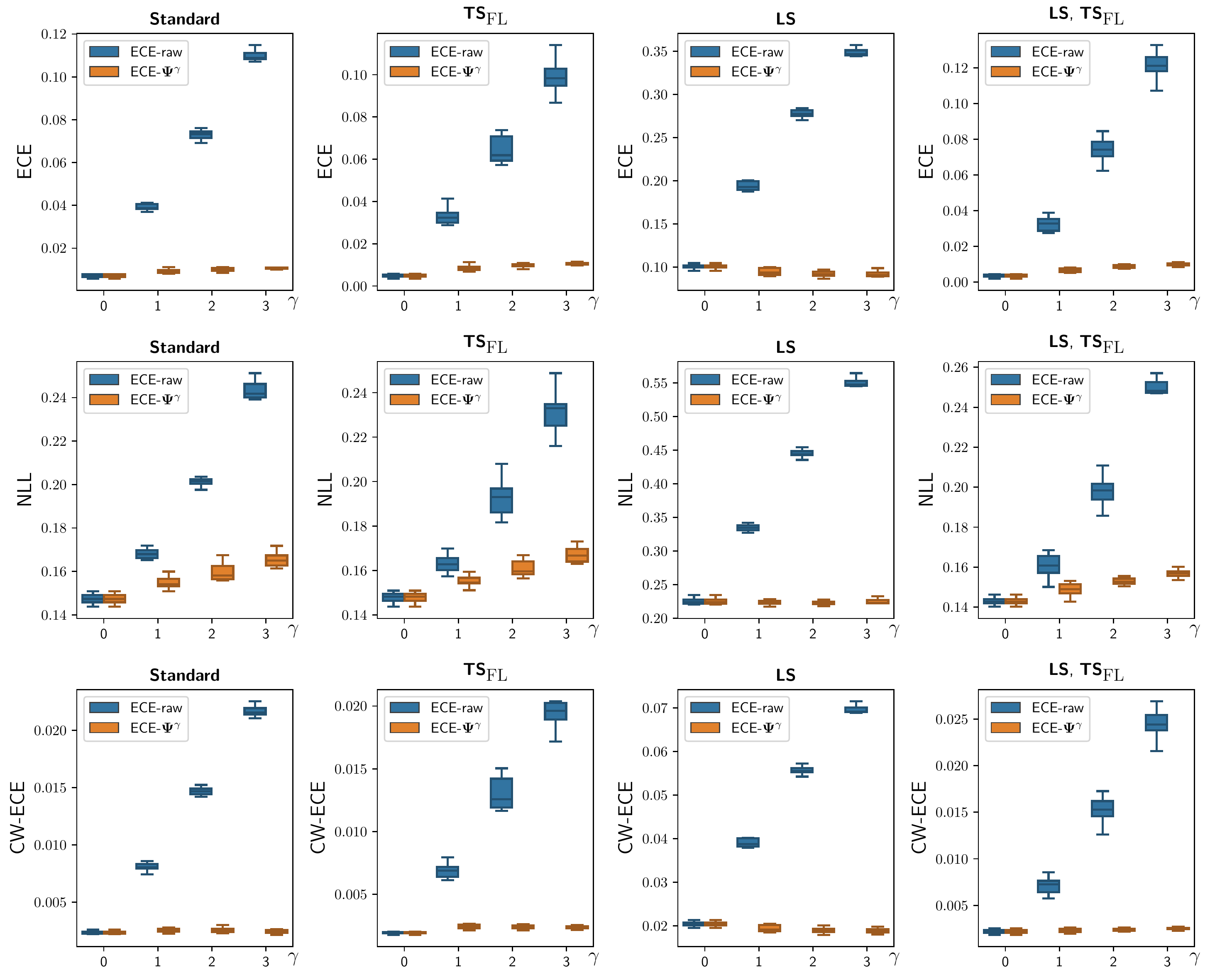}
    \caption{Box plots of ECE, NLL, and CW-ECE for SVHN and ResNet8. See each graph's title for the particular details.}
    \label{app:fig-box-svhn-res8}
\end{figure*}         

\begin{figure*}
    \centering
    \includegraphics[width=0.65\textwidth]{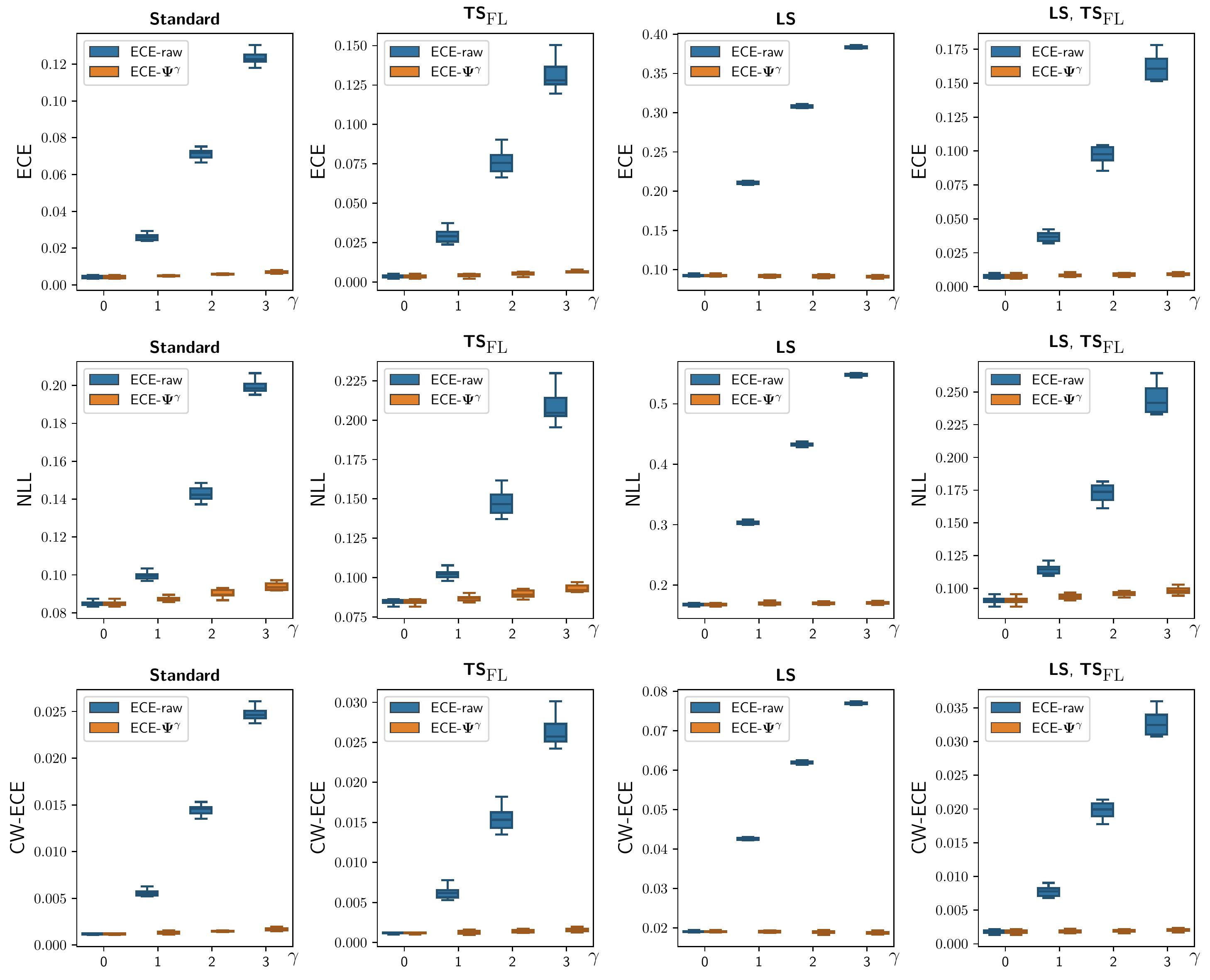}
    \caption{Box plots of ECE, NLL, and CW-ECE for SVHN and ResNet20. See each graph's title for the particular details.}
\end{figure*}    

\begin{figure*}
    \centering
    \includegraphics[width=0.65\textwidth]{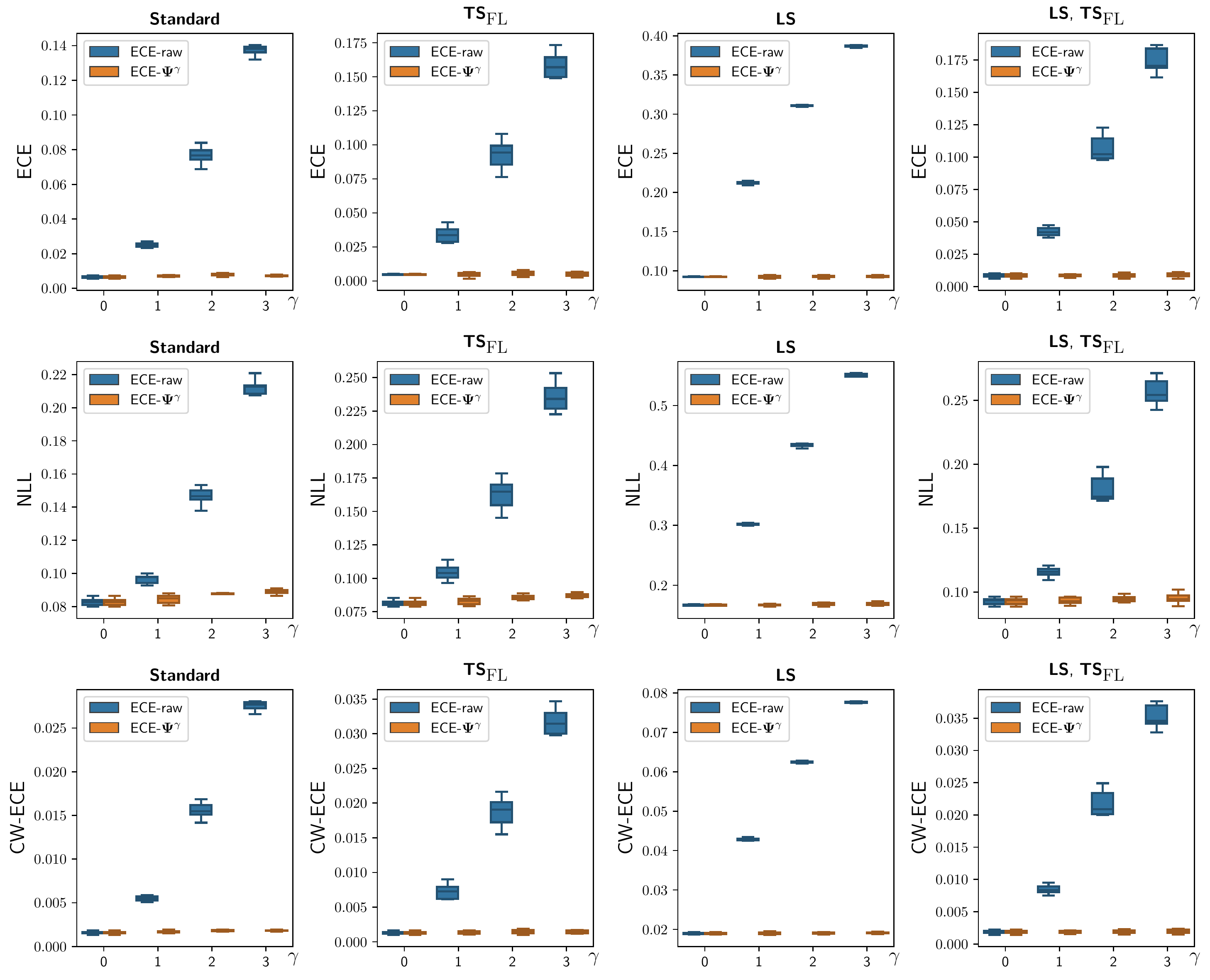}
    \caption{Box plots of ECE, NLL, and CW-ECE for SVHN and ResNet44. See each graph's title for the particular details.}
\end{figure*}    

\begin{figure*}
    \centering
    \includegraphics[width=0.65\textwidth]{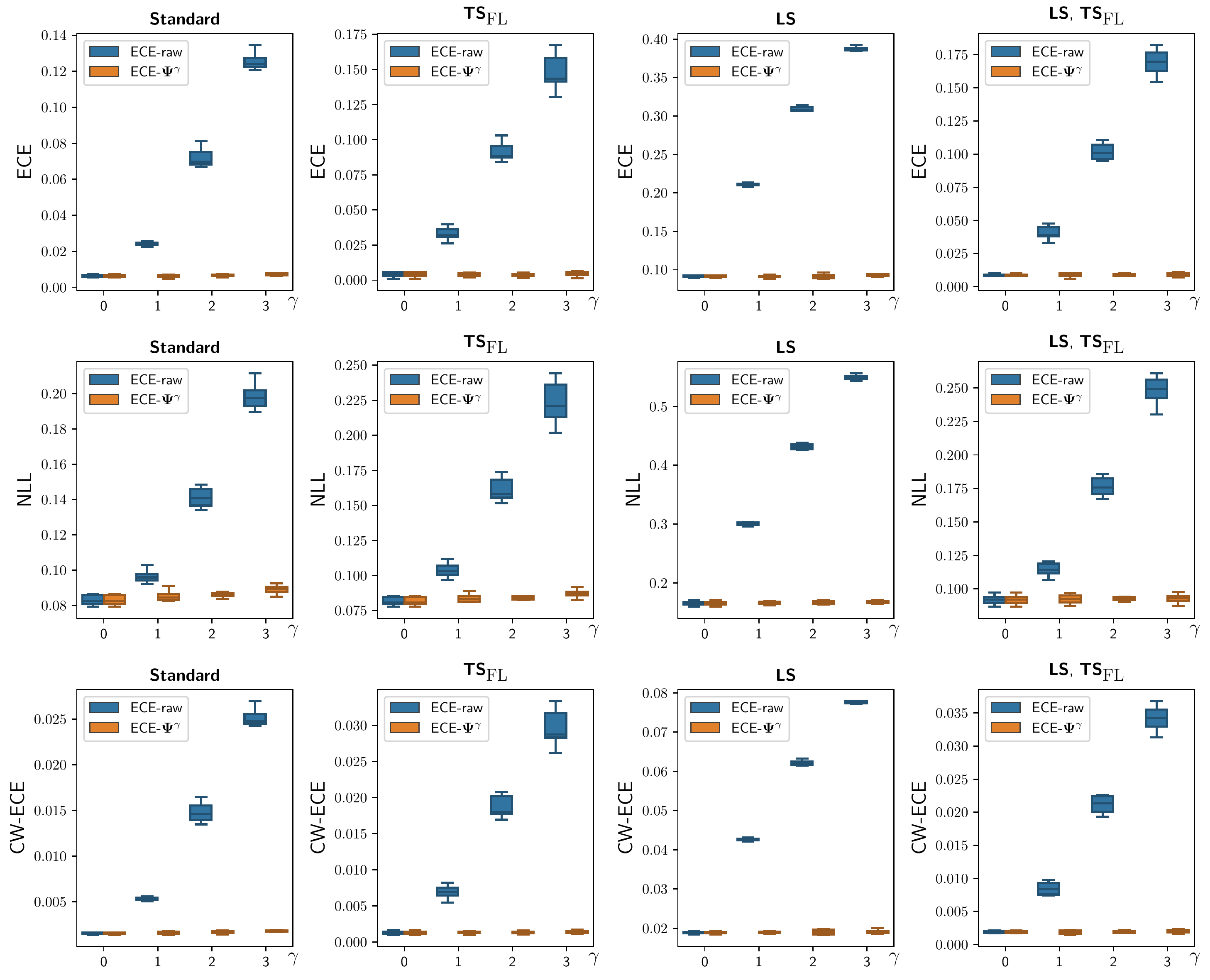}
    \caption{Box plots of ECE, NLL, and CW-ECE for SVHN and ResNet110. See each graph's title for the particular details.}
    \label{app:fig-box-svhn-res110}
\end{figure*}

\begin{figure*}
    \centering
    \includegraphics[width=0.65\textwidth]{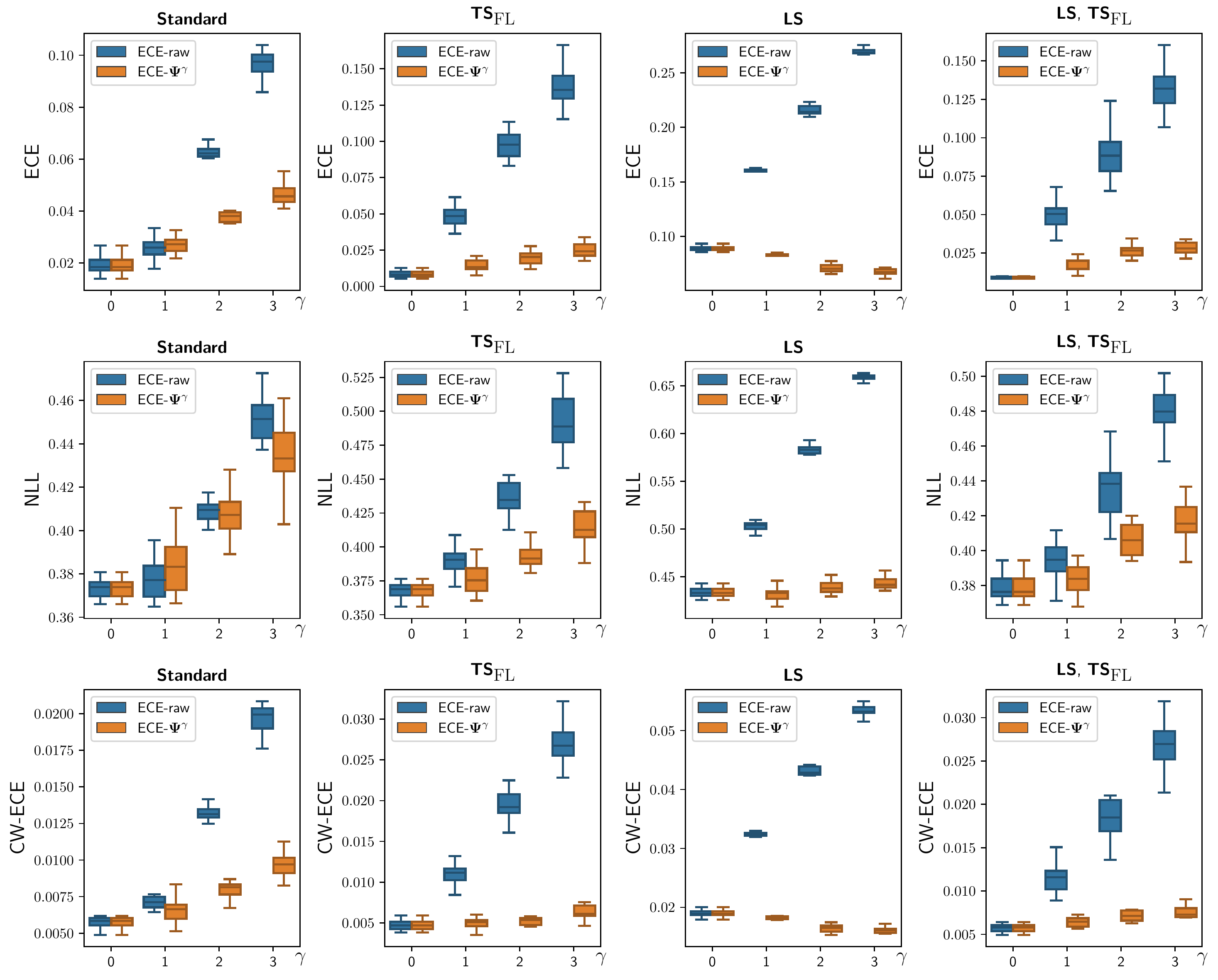}
    \caption{Box plots of ECE, NLL, and CW-ECE for CIFAR10 and ResNet8. See each graph's title for the particular details.}
    \label{app:fig-box-cifar10-res8}
\end{figure*}         

\begin{figure*}
    \centering
    \includegraphics[width=0.65\textwidth]{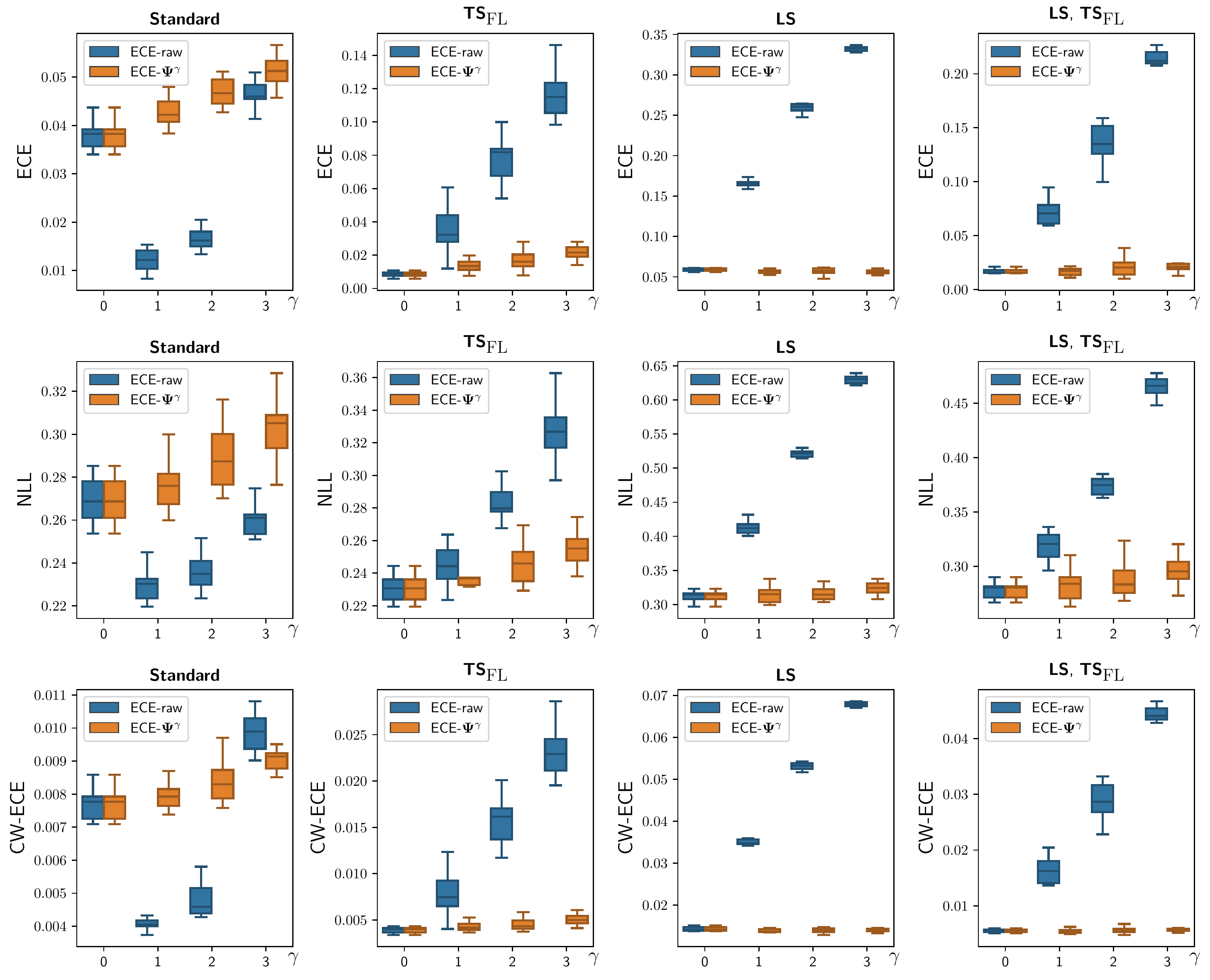}
    \caption{Box plots of ECE, NLL, and CW-ECE for CIFAR10 and ResNet20. See each graph's title for the particular details.}
\end{figure*}    

\begin{figure*}
    \centering
    \includegraphics[width=0.65\textwidth]{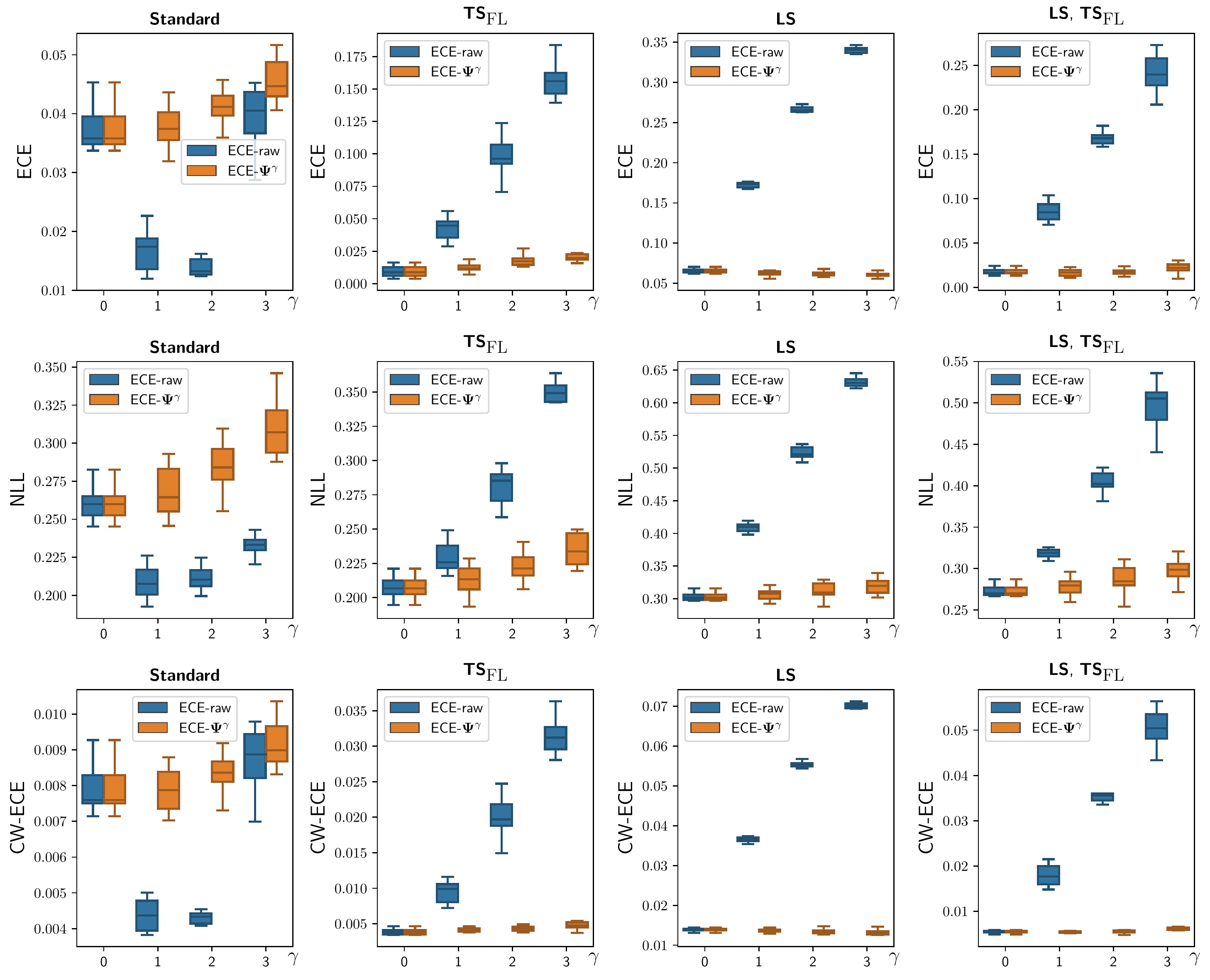}
    \caption{Box plots of ECE, NLL, and CW-ECE for CIFAR10 and ResNet44. See each graph's title for the particular details.}
\end{figure*}    

\begin{figure*}
    \centering
    \includegraphics[width=0.65\textwidth]{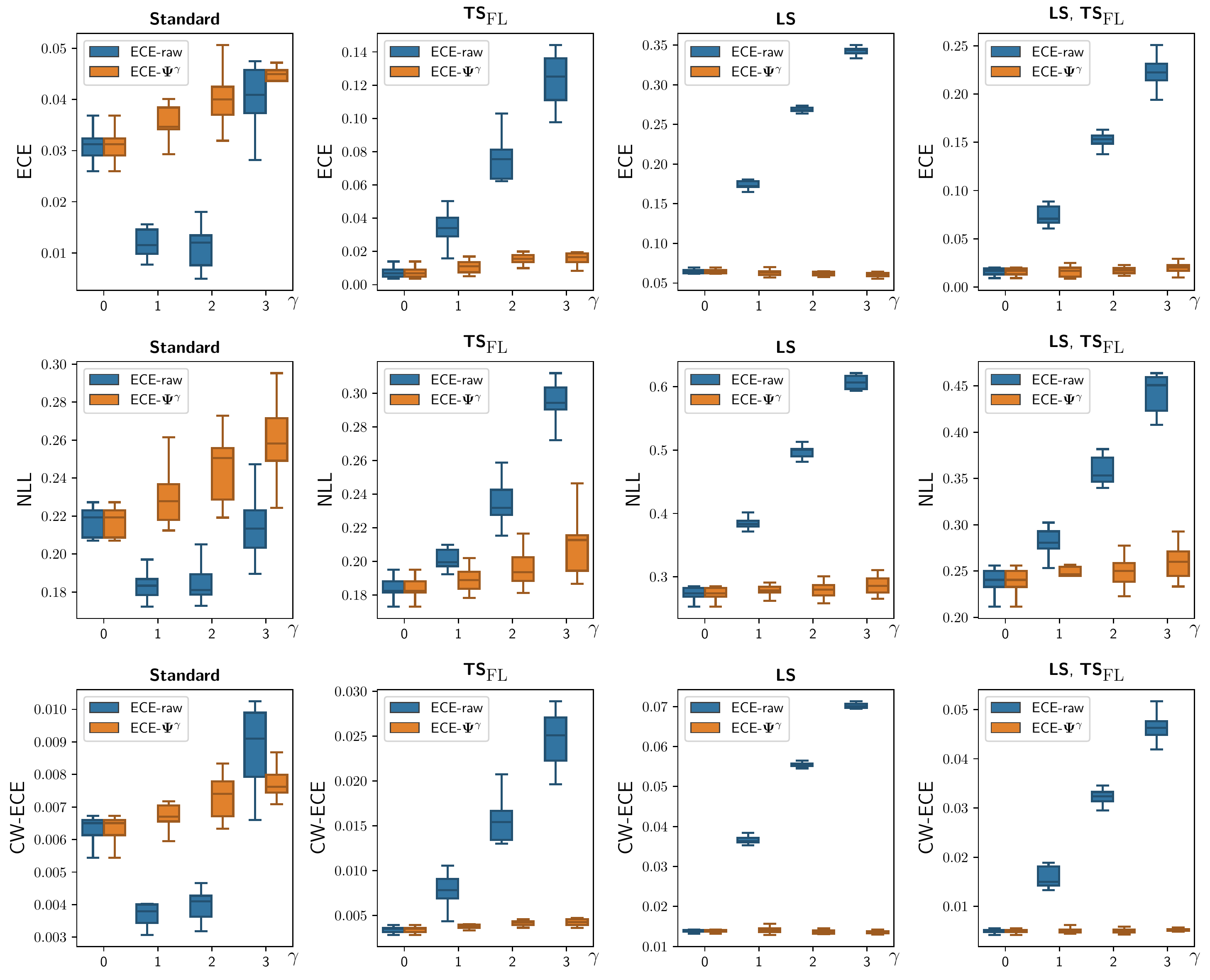}
    \caption{Box plots of ECE, NLL, and CW-ECE for CIFAR10 and ResNet110. See each graph's title for the particular details.}
    \label{app:fig-box-cifar10-res110}
\end{figure*}

\begin{figure*}
    \centering
    \includegraphics[width=0.65\textwidth]{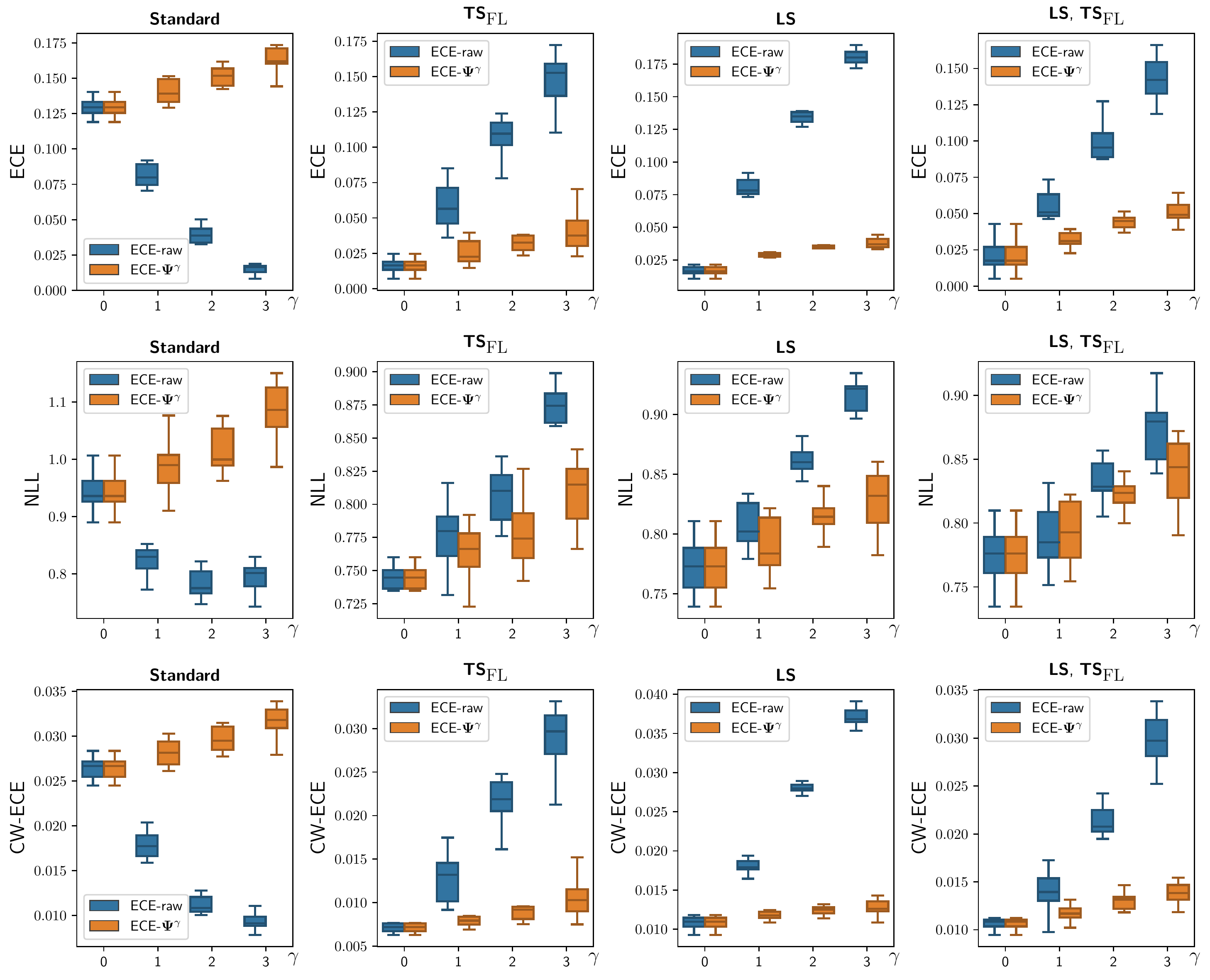}
    \caption{Box plots of ECE, NLL, and CW-ECE for CIFAR10-s and ResNet8. See each graph's title for the particular details.}
    \label{app:fig-box-cifar10s-res8}
\end{figure*}         

\begin{figure*}
    \centering
    \includegraphics[width=0.65\textwidth]{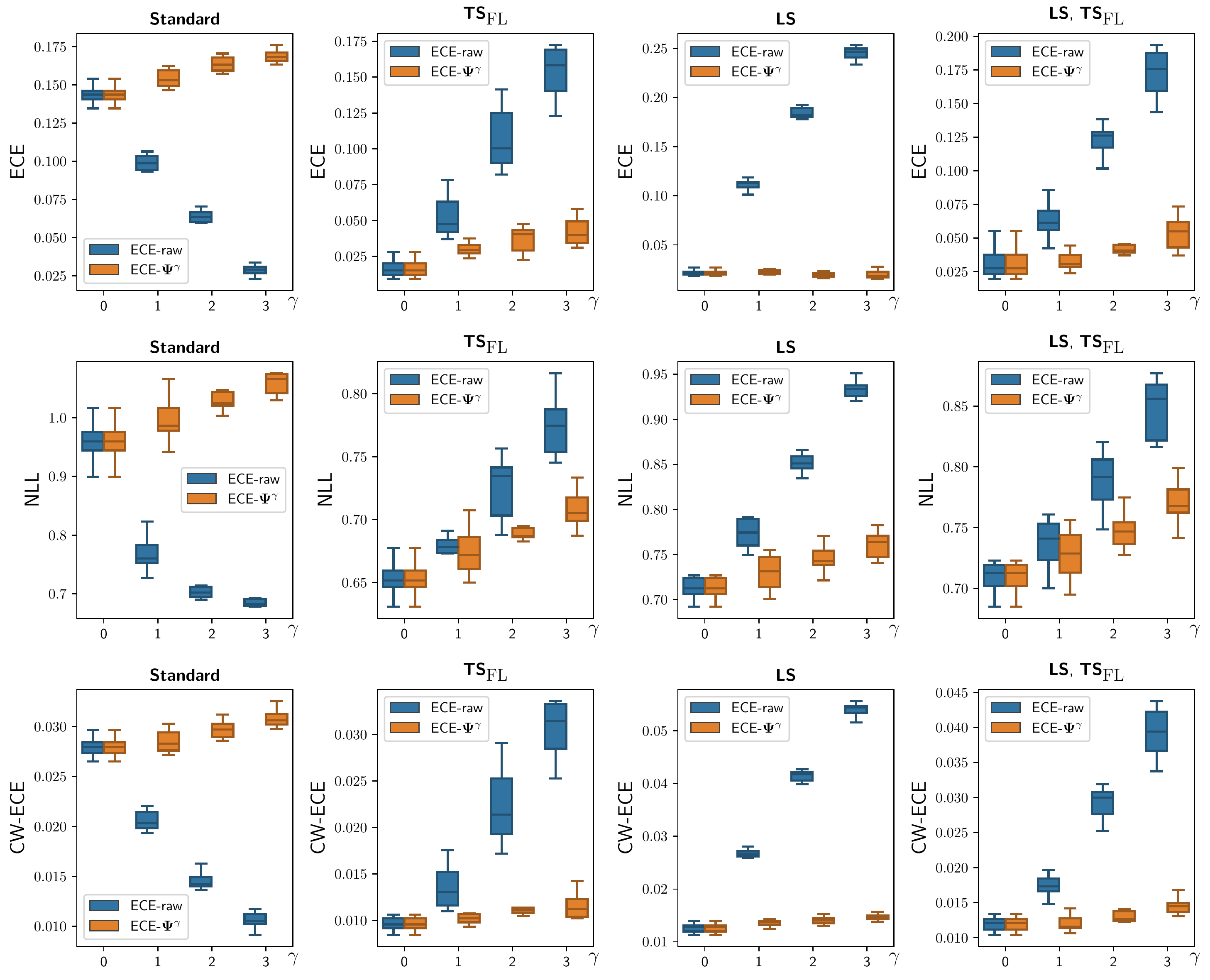}
    \caption{Box plots of ECE, NLL, and CW-ECE for CIFAR10-s and ResNet20. See each graph's title for the particular details.}
\end{figure*}    

\begin{figure*}
    \centering
    \includegraphics[width=0.65\textwidth]{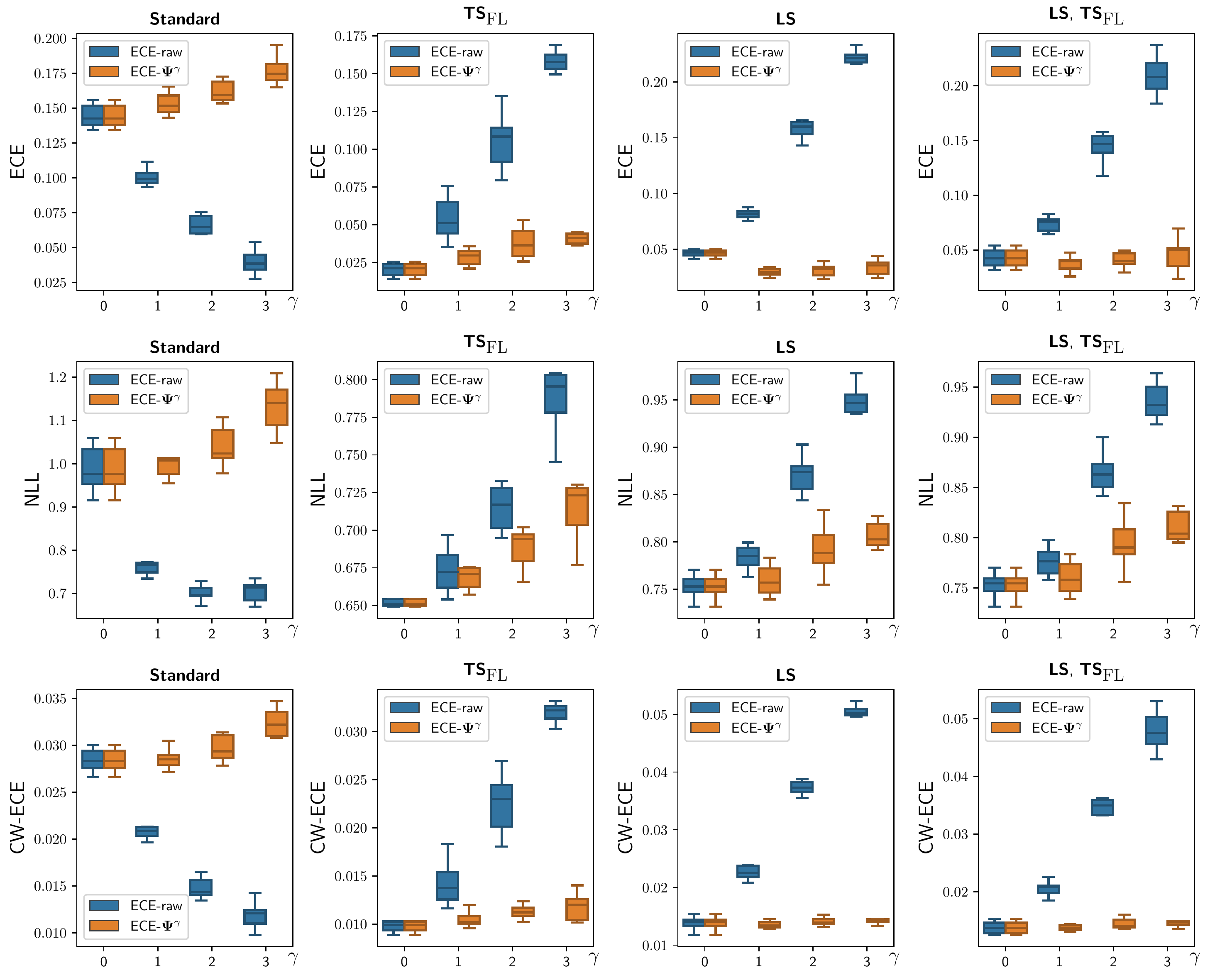}
    \caption{Box plots of ECE, NLL, and CW-ECE for CIFAR10-s and ResNet44. See each graph's title for the particular details.}
\end{figure*}    

\begin{figure*}
    \centering
    \includegraphics[width=0.65\textwidth]{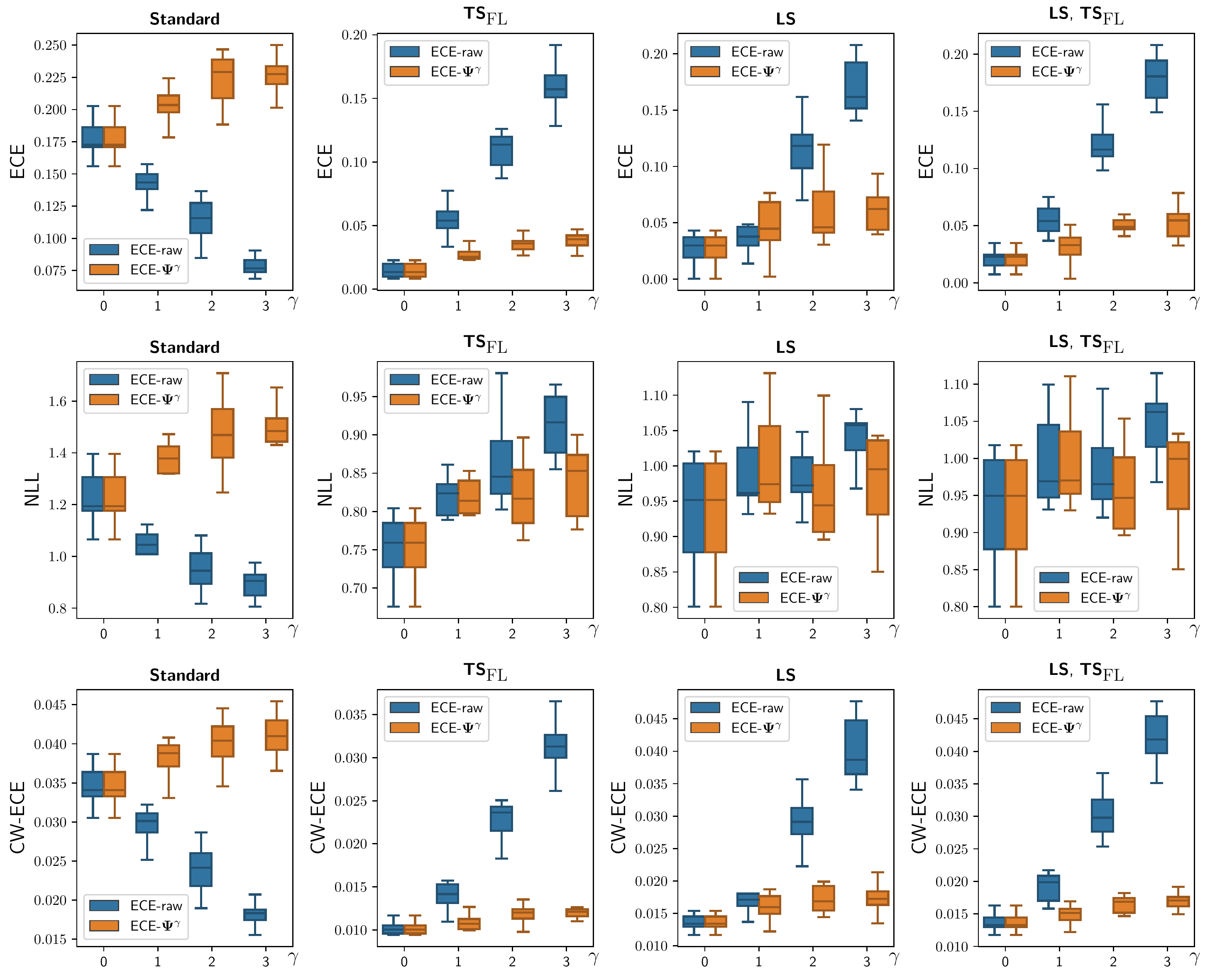}
    \caption{Box plots of ECE, NLL, and CW-ECE for CIFAR10-s and ResNet110. See each graph's title for the particular details.}
    \label{app:fig-box-cifar10s-res110}
\end{figure*}

\begin{figure*}
    \centering
    \includegraphics[width=0.65\textwidth]{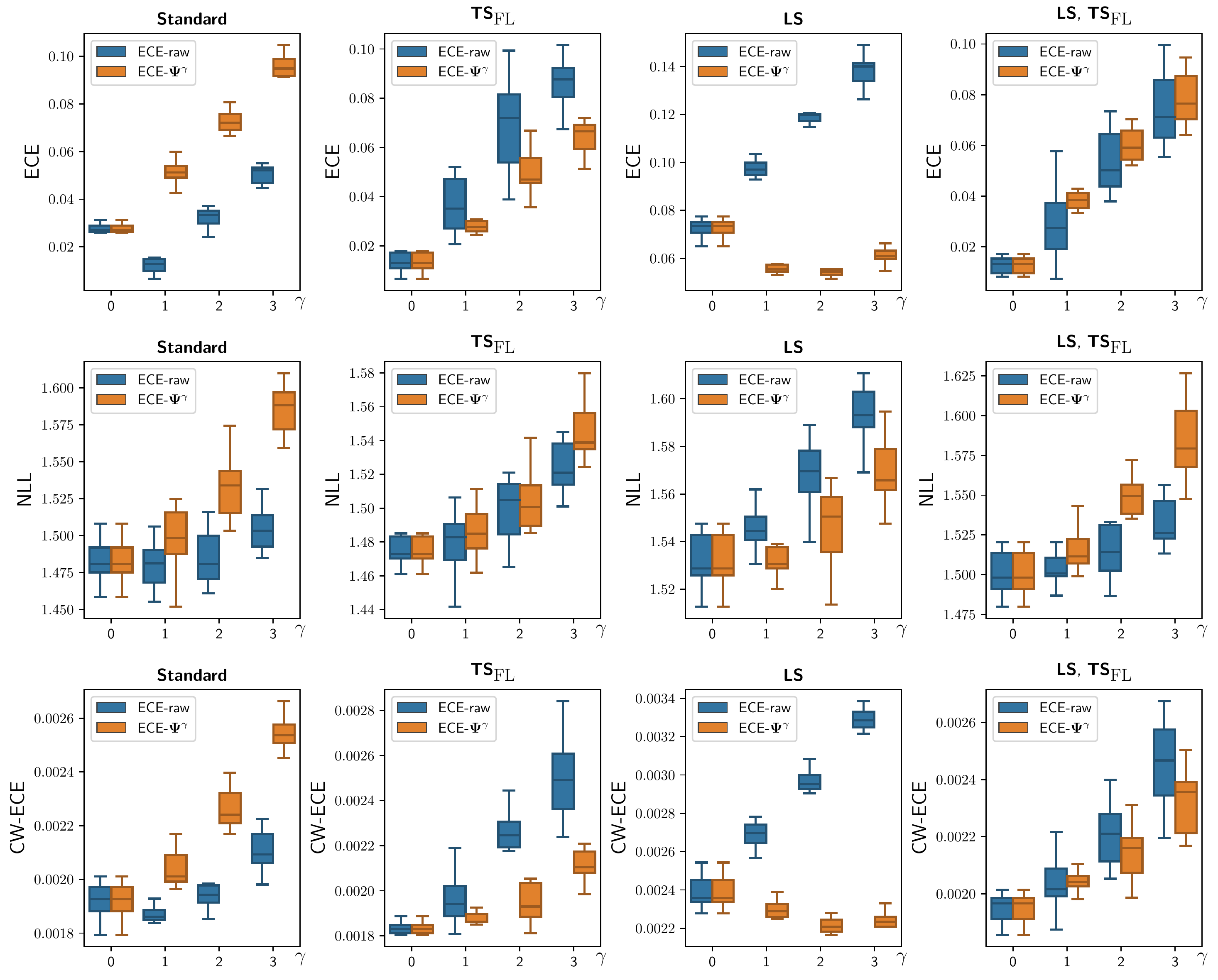}
    \caption{Box plots of ECE, NLL, and CW-ECE for CIFAR100 and ResNet8. See each graph's title for the particular details.}
    \label{app:fig-box-cifar100-res8}
\end{figure*}         

\begin{figure*}
    \centering
    \includegraphics[width=0.65\textwidth]{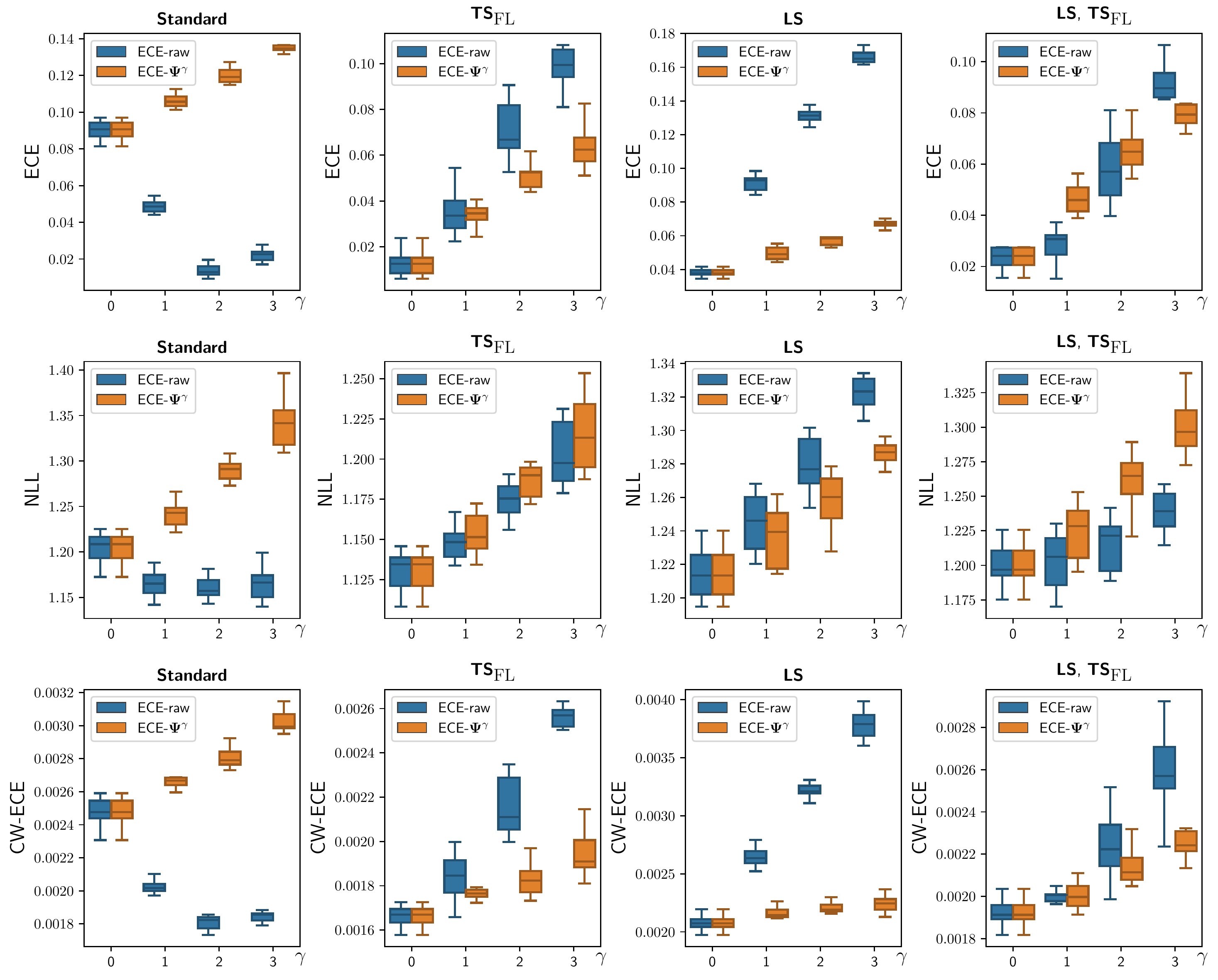}
    \caption{Box plots of ECE, NLL, and CW-ECE for CIFAR100 and ResNet20. See each graph's title for the particular details.}
\end{figure*}    

\begin{figure*}
    \centering
    \includegraphics[width=0.65\textwidth]{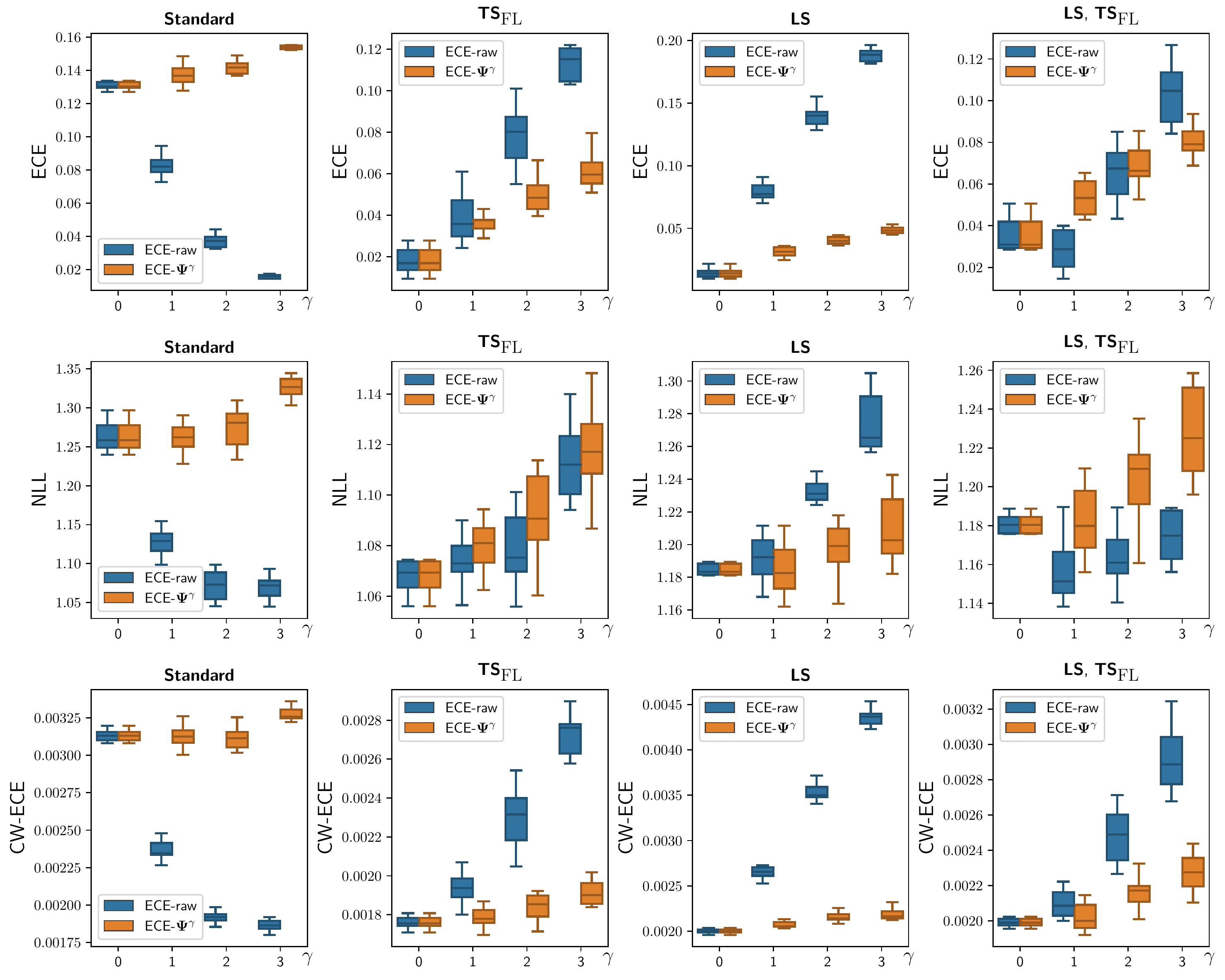}
    \caption{Box plots of ECE, NLL, and CW-ECE for CIFAR100 and ResNet44. See each graph's title for the particular details.}
\end{figure*}    

\begin{figure*}
    \centering
    \includegraphics[width=0.65\textwidth]{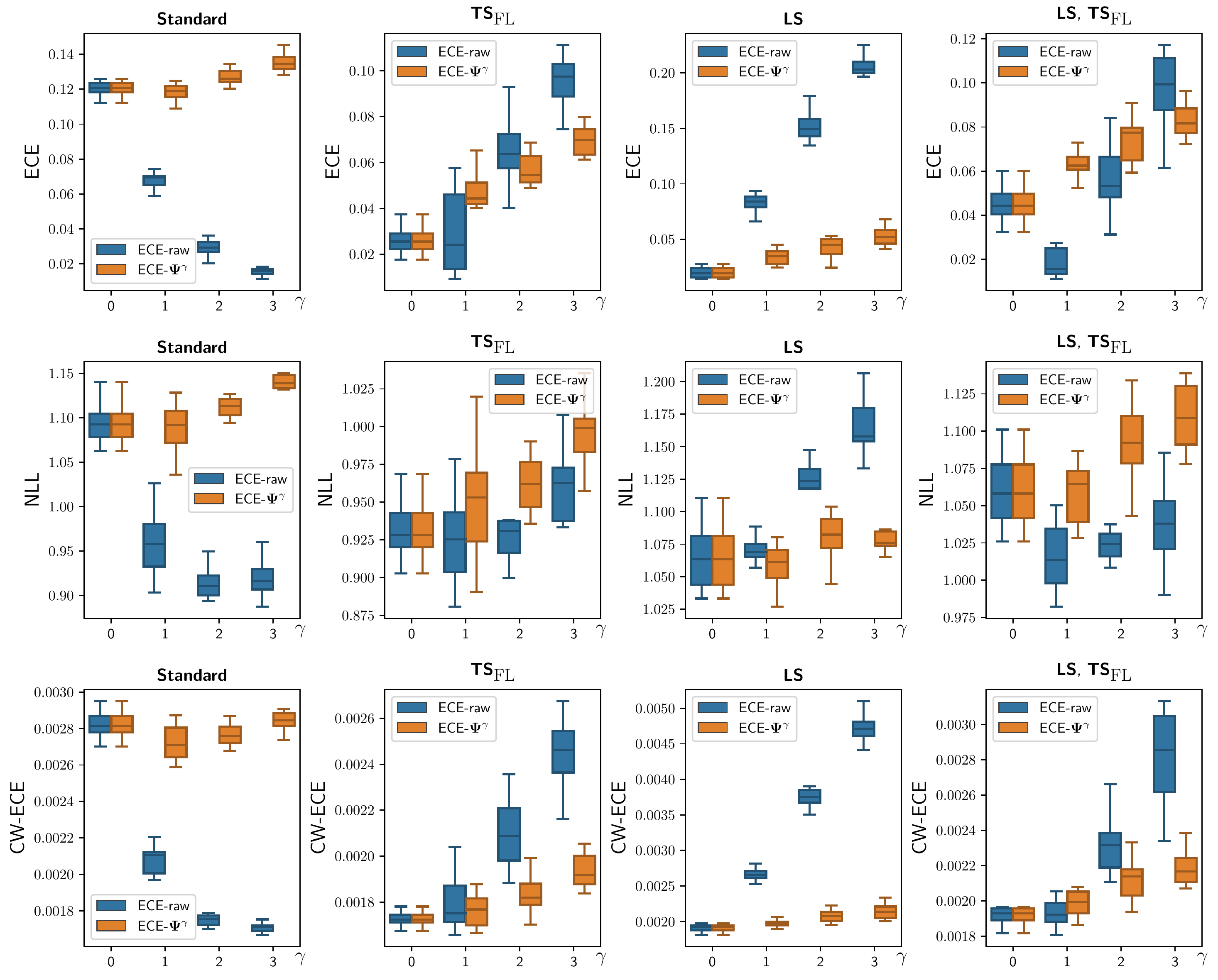}
    \caption{Box plots of ECE, NLL, and CW-ECE for CIFAR100 and ResNet110. See each graph's title for the particular details.}
    \label{app:fig-box-cifar100-res110}
\end{figure*} 
\end{document}